%% file: main.tex
\documentclass{article}
\usepackage{ourmacros}

\title{Pursuit of a Discriminative Representation for Multiple Subspaces via Sequential Games}

\author{Druv Pai, Michael Psenka, Chih-Yuan Chiu, Manxi Wu,
Edgar Dobriban,
Yi Ma\thanks{
D. Pai (\texttt{druvpai@berkeley.edu}), M. Psenka (\texttt{psenka@berkeley.edu}), C.Y. Chiu (\texttt{chihyuan\_chiu@berkeley.edu}), and Y. Ma (\texttt{yima@eecs.berkeley.edu}) are with the Department of Electrical Engineering and Computer Sciences at the University of California, Berkeley. M. Wu (\texttt{manxiwu@cornell.edu}) is with the School of Operations Research and Information Engineering at Cornell University, and the Simons Institute for Theory of Computing at the University of California, Berkeley.
E. Dobriban (\texttt{dobriban@upenn.edu}) is with the Department of Statistics and Data Science at the University of Pennsylvania.
}}

\begin{document}

\maketitle

\begin{abstract}
    We consider the problem of learning discriminative representations for data in a high-dimensional space with distribution supported on or around multiple low-dimensional linear subspaces. That is, we wish to compute a linear injective map of the data such that the features lie on multiple \textit{orthogonal} subspaces. Instead of treating this learning problem using multiple PCAs, we cast it as a sequential game using the closed-loop transcription (CTRL) framework recently proposed for learning discriminative and generative representations for general low-dimensional submanifolds. We prove that the equilibrium solutions to the game indeed give correct representations. Our approach unifies classical methods of learning subspaces with modern deep learning practice, by showing that subspace learning problems may be provably solved using the modern toolkit of representation learning. In addition, our work provides the first theoretical justification for the CTRL framework, in the important case of linear subspaces. We support our theoretical findings with compelling empirical evidence. We also generalize the sequential game formulation to more general representation learning problems. Our code, including methods for easy reproduction of experimental results, is \href{https://github.com/DruvPai/MultipleSubspaceRepresentationPursuit}{publically available on GitHub}.
\end{abstract}

\input{intro.tex}
\input{preliminaries.tex}
\input{ctrl_msp_theory.tex}
\input{ctrl_sg_theory.tex}
\input{experiments.tex}

\input{conclusion}
\input{acknowledgements_funding_etc}

\bibliographystyle{plainnat} 
\bibliography{references}


\appendix 

\section*{\center Appendix}

The appendix is organized as follows. In \Cref{sec:pca_gan_games}, we discuss popular representation learning algorithms, such as PCA, GAN, and CTRL in terms of simultaneous game theory and the representation learning framework we developed. In \Cref{sec:proofs} we give proofs of all theorems in the main body. In \Cref{sec:ctrl_ssp} we discuss a specialization of CTRL-MSP, which we call CTRL-SSP, to the case where the data lies on a single linear subspace; in this section we give mathematical justification and empirical support for CTRL-SSP. In \Cref{sec:more_experiments} we provide more experimental details, a more thorough empirical evaluation of CTRL-MSP, and more detailed comparisons with other representation learning algorithms.

\input{appendix_simultaneous_games}
\newpage
\input{appendix_ctrl_msp_proofs}
\newpage
\input{appendix_ctrl_ssp}

\newpage
\input{appendix_ctrl_msp_experiments}

\end{document}

%% file: intro.tex
\section{Motivation and Context} \label{sec:intro_motivation}

Learning representations of complex high-dimensional data with low underlying complexity is a central goal in machine learning, with applications to compression, sampling, out-of-distribution detection, classification, etc. For example, in the context of image data, one may perform clustering \citep{prasad2020variational}, and generate or detect fake images \citep{huang2018introduction}. There are a number of recently popular methods for representation learning, several proposed in the context of image generation; one such example is generative adversarial networks (GANs) \citep{goodfellow2014generative}, giving promising results \citep{karras2018style,mino2018logan}. Despite empirical successes, theoretical understanding of representation learning of high-dimensional data with underlying low complexity is still rather primitive. Classical methods with theoretical guarantees \citep{Jolliffe2002}, such as principal component analysis (PCA), are divorced from modern methods such as GANs whose justifications are mostly empirical and whose theoretical properties remain poorly understood \citep{tse2017understanding,ozdaglar2020nash}.

A challenge for our theoretical understanding is that \textit{high-dimensional data often has low-dimensional structure}, such as belonging to multiple subspaces and even nonlinear manifolds \citep{Wright-Ma-2022,li2018global,zhang2019structured,shen2020complete,zhai2020complete,zhai2019understanding,qu2019geometric,Lau2020Short,fefferman2013testing}. 
This hypothesis can be difficult to account for theoretically.\footnote{One assumption which violates this hypothesis implicitly is the existence of a probability density for the data. For instance, the analysis in several prominent works on representation learning, such as \cite{kingma2013autoencoding} and \cite{tse2017understanding} critically requires this assumption to hold. Probability densities with respect to the Lebesgue measure on \(\R^{n}\) do not exist if the underlying probability measure has a Lebesgue measure zero support, e.g., for lower-dimensional structures such as subspaces \citep{kallenberg2021foundations}. Thus, this assumption excludes a lower dimensionality of the data.} In fact, our understanding of this setting, and knowledge of principled and generalizable solutions, is still incomplete, even in the case when the data lies on multiple linear subspaces \citep{Vidal:Springer16}, and the representation map is linear.

In this work, we aim to bridge this gap. More specifically, we propose a new theoretically principled formulation, based on sequential game theory, for learning discriminative representations for multiple low-dimensional linear subspaces in high-dimensional spaces. We explicitly characterize the learned representations in this framework. Our results show that classical but complex subspace learning problems can be solved using modern deep learning tools, thus unifying the classical and modern perspectives on this class of problems. Our analysis is tailored to fit the assumption of high-dimensional data with low-dimensional structure.

\subsection{Related Works} \label{sub:related_works}

\paragraph{PCA, Subspace Clustering, and Autoencoding.} Principal component analysis (PCA) and its probabilistic versions \citep{hotelling1933analysis,tipping1999probabilistic} are a classical tool for learning low-dimensional representations. One finds the best \(\ell^{2}\)-approximating subspace of a given dimension for the data. Thus, PCA can be viewed as \textit{seeking to learn the linear subspace structure of the data}. Several generalizations of PCA exist. Generalized PCA (GPCA) \citep{vidal2012generalized} \textit{seeks to learn multiple linear subspace structure} by clustering. Unlike PCA and this work, GPCA does not learn transformed representations of the data. PCA has also been adapted to recover nonlinear structures in many ways \citep{van2009dimensionality}, e.g., via principal curves \citep{hastie1989principal} or autoencoders \citep{kramer1991nonlinear}.

\paragraph{GAN.} Generative Adversarial Networks (GANs) are a recently popular representation learning method \citep{goodfellow2014generative,arjovsky2017wasserstein}. GANs simultaneously learn a generator function, which maps low-dimensional noise to the data distribution, and a discriminator function, which maps the data to discriminative representations from which one can classify the data as authentic or synthetic with a simple predictor. The generator and discriminator are trained adversarially; the generator is trained to generate data which is distributionally close to real data, in order to fool the discriminator, while the discriminator is simultaneously trained to identify discrepancies between the generator output and empirical data. 

While GANs enjoy certain empirical success (see e.g.~\cite{karras2018style,mino2018logan}), their theoretical properties are less well developed, especially in the context of high-dimensional data with intrinsic structure. 
More specifically, the most prominent works of GAN analysis use the simplifying assumption of full-rank data \citep{tse2017understanding}, require explicit computation of objective functions which are intractable to even estimate using a finite sample \citep{arjovsky2017wasserstein,zhu2020deconstructing}, or show that GANs have poor theoretical behavior, such as their training game not having Nash equilibria \citep{ozdaglar2020nash}. In this work, we adopt the more realistic assumption of low-dimensional data in a high-dimensional space, use explicit, closed-form objective functions which are more convenient to optimize (at least in the linear case), and demonstrate the existence of global equilibria of the training game corresponding to our method.

%% file: preliminaries.tex
\section{Preliminaries} \label{sec:preliminaries}

\subsection{Representation Learning} \label{sub:representation_learning}

Let \(\x\) be a random variable taking values in \(\R^{\dx}\). Let \(\X \in \R^{\dx \times n}\) be a data matrix that contains \(n \geq 1\) data points \(\x^{1}, \dots, \x^{n}\) in \(\R^{\dx}\) which are independent realizations of \(\x\). To model that the dimension of the support of \(\x\) is lower than the ambient dimension \(\dx\), suppose that \(\x\) is supported on \textit{a union of \(k\) linear subspaces} \(\bigcup_{j = 1}^{k}\S_{j} \subseteq \R^{\dx}\), each of dimension \(\dSj \doteq \dim{\S_{j}}\). For each \(j \in \{1, \dots, k\}\), let \(\X_{j} \in \S_{j}^{n_{j}} \subseteq \R^{\dx \times n_{j}}\) be the matrix of the \(n_{j} \geq 1\) columns of \(\X\) contained in \(\S_{j}\), and let the \textit{class information}, containing the assignment of each data point \(\x^{1}, \dots, \x^{n}\) to its respective subspace index \(j\in\{1,\ldots, k\}\) be denoted by the one-hot encoding matrix \(\bPi \in \{0, 1\}^{n \times k}\). 

The goal is to learn an encoder mapping \(f \colon \R^{\dx} \to \R^{\dz}\), in some function class \(\F\), given \(\X\) and \(\bPi\). 
Normally we want \(f\) to be such that \(\dz \le \dx\) and \(f(\x)\) has better geometric properties, such as being supported on orthogonal subspaces. Moreover, we want to learn an approximate \textit{inverse} or decoder mapping \(g \colon \R^{\dz} \to \R^{\dx}\) in some function class \(\G\), such that the distributions of \(\x\) and \((g \circ f)(\x)\) are close. As we will see, with the simplifying assumption that the data are on linear subspaces, the encoder function \(f \in \F\) and decoder function \(g \in \G\) can both be linear maps. One can view these \(f\) and \(g\) as special ``one-layer'' cases of multi-layer (deep) models, which may be required to handle \textit{nonlinear} low-dimensional structures.\footnote{Analysis where the data has low-dimensional nonlinear structure is beyond the scope of this paper, since even the case of low-dimensional linear structure was open.} 

\subsection{Closed-Loop Transcription} \label{sub:ctrl}

To learn the encoder/decoder mappings \(f\) and \(g\), we use the Closed-Loop Transcription (CTRL) framework, a recent method which was proposed for representation learning of low-dimensional submanifolds in high-dimensional space and has had good empirical results \citep{LDR}. This framework generalizes \textit{both} autoencoders and GANs; \(f\) has dual roles as an encoder and discriminator, and \(g\) has dual roles as a decoder and a generator.

For the data matrix \(\X\), we define \(f(\X) = \mat{f(\x^{1}), \ldots, f(\x^{n})}\), etc. The training process follows a \textit{closed loop}: starting with the data \(\X\) and the autoencoded data \((g \circ f)(\X)\), the data representations \(f(\X)\) and the autoencoded data representations \((f \circ g \circ f)(\X)\) are used to train \(f\) and \(g\). This approach has a crucial advantage over the GAN formulation: contrary to GANs \citep{arjovsky2017wasserstein,zhu2020deconstructing}, since \(f(\X)\) and \((f \circ g \circ f)(\X)\) both live in the structured representation space \(\R^{\dz}\), interpretable quantifications of representation quality and of the difference between \(f(\X)\) and \((f \circ g \circ f)(\X)\) exist and may be computed \textit{efficiently in closed form}.

\subsection{Rate Reduction} \label{sub:rate_reduction}

These tractable quantities are based on the information-theoretic and statistical paradigm of \textit{rate reduction} discussed in the CTRL literature \citep{LDR, OriginalMCR2} as well as previous works \citep{OriginalRateReduction}. Here we review the main principles, as they are central to an information-theoretic interpretation of our objective functions. 

Let \(\z\) be a random variable taking values in \(\R^{\dz}\). Let \(\RD{\cdot \mid \z}\) be the rate distortion function of \(\z\) with respect to the Euclidean squared distance distortion \citep{cover1999elements}. Information-theoretically, this is the \textit{coding rate} of the data; that is, the average number of bits required to encode \(\z\), such that the expected Euclidean squared distance between \(\z\) and its encoding is at most the first argument of the function.

For a symmetric matrix \(\bm{A}\), let \(\lambda_{\min}(\bm{A})\) be the minimum eigenvalue of \(\bm{A}\). If \(\bm{u} \sim \normal{\bm{0}_{\dz}, \bm{\Gamma}}\) is a multivariate Gaussian random vector with mean \(\bm{0}_{\dz}\) and covariance \(\bm{\Gamma}\), then
\begin{equation*}
    \RD{\eps \mid \bm{u}} = \frac{1}{2}\logtdet{\frac{\dz}{\eps^{2}}\bm{\Gamma}} \qquad \forall \eps \in \Big[0, \sqrt{\dz\cdot \lambda_{\min}(\bm{\Gamma})}\Big].
\end{equation*}
For larger \(\eps\), the rate distortion function becomes more complicated and can be found by the water-filling algorithm on the eigenvalues of \(\bm{\Gamma}\). 
However, \cite{OriginalRateReduction} proposes the following approximation of the rate distortion. 
For \(\bm{w}_{\eps} \sim \normal{\bm{0}_{\dz}, \frac{\eps^{2}}{\dz}\bm{I}_{\dz}}\) independent of $\z$, let
\begin{equation*}
    R_{\eps}(\z) \doteq \RD{\eps \mid \z + \bm{w}_{\eps}}.
\end{equation*}
If \(\z \sim \normal{\bm{0}_{\dz}, \bSigma}\), then we may derive a closed form expression for \(R_{\eps}(\z)\) for \textit{all} \(\eps > 0\). Since \(\z\) and \(\bm{w}_{\eps}\) are normally distributed, so is \(\z + \bm{w}_{\eps}\), and 
\begin{equation*}
    \bm{z} + \bm{w}_{\eps} \sim \normal{\bm{0}_{\dz}, \frac{\eps^{2}}{\dz}\bm{I}_{\dz} + \bSigma}.
\end{equation*}
Thus,
\begin{equation*}
    \sqrt{\dz \cdot \lambda_{\min}\mathopen{}\left(\frac{\eps^{2}}{\dz}\bm{I}_{\dz} + \bSigma\right)} = \sqrt{\dz \cdot \left(\frac{\eps^{2}}{\dz} + \lambda_{\min}(\bSigma)\right)} = \sqrt{\eps^{2} + \dz\lambda_{\min}(\bSigma)} \geq \eps.
\end{equation*}
Therefore, we have the following closed form expression for \(R_{\eps}(\z)\) for \textit{all} \(\eps > 0\).
\begin{align*}
	R_{\eps}(\z) 
	&= \RD{\eps \mid \z + \bm{w}_{\eps}}
	= \frac{1}{2}\logtdet{\frac{\dz}{\eps^{2}}\left(\bSigma + \frac{\eps^{2}}{\dz}\I_{\dz}\right)}
	\\
	&= \frac{1}{2}\logtdet{\I_{\dz} + \frac{\dz}{\eps^{2}}\bSigma}.
\end{align*}
In information-theoretic terms, \(R_{\eps}(\z)\) is a \textit{regularized rate distortion} function. Heuristically, it counts the average number of bits required to encode \(\z\) up to \(\eps\) precision, and thus it quantifies the expansiveness of the distribution of \(\z\), or in other words how ``spread out'' the distribution is.

From this quantity we can also define a difference function\footnote{Unfortunately, it is not a true distance function; for starters, it can be zero for random variables with non-identical distributions.} between distributions of two possibly-correlated random vectors \(\z_{1}, \z_{2} \in \R^{\dz}\). This function approximately computes the average number of bits saved by encoding \(\z_{1}\) and \(\z_{2}\) separately and independently compared to encoding them together, say by encoding a mixture random variable \(\z\) which is \(\z_{1}\) with probability \(\frac{1}{2}\) and \(\z_{2}\) with probability \(\frac{1}{2}\), up to precision \(\eps\). In this notation, we have
\begin{equation*}
	\Delta R_{\eps}(\z_{1}, \z_{2}) \doteq R_{\eps}(\z) - \frac{1}{2}R_{\eps}(\z_{1}) - \frac{1}{2}R_{\eps}(\z_{2}).
\end{equation*}
This difference function has several advantages over Wasserstein or Jensen-Shannon distances. It is a principled quantification of difference which is computable in closed-form for the widely representative class of Gaussian distributions. In particular, due to the existence of the closed-form representation, it is much simpler to do analysis on the solutions of optimization problems involving this function.

We may generalize the difference function to several random vectors. Specifically, define probabilities \(\pi_{1}, \dots, \pi_{k} \in [0, 1]\) such that \(\sum_{j = 1}^{k}\pi_{j} = 1\), arranged in a vector \(\bm{\pi} \in [0, 1]^{k}\), and let \(\z_{1}, \dots, \z_{k}\) be random variables taking values in \(\R^{\dz}\). Define \(\z\) to be the mixture random vector which equals \(\z_{j}\) with probability \(\pi_{j}\). Then the \textit{coding rate reduction} of \(\z\) given \(\bm{\pi}\) is given by
\begin{equation*}
	\Delta R_{\eps}(\z \mid \bm{\pi}) \doteq R_{\eps}(\z) - \sum_{j = 1}^{k}\pi_{j}R_{\eps}(\z_{j}).
\end{equation*}
Heuristically, this again approximates the average number of bits saved by encoding each \(\z_{j}\) separately as opposed to encoding \(\z\) as a whole,
and thus 
it quantifies how compact the distribution of each \(\z_{j}\) is and how expansive, or ``spread out'' the distribution of \(\z\) as a whole is. More precisely, it was shown by \cite{OriginalMCR2} that, subject to rank and Frobenius norm constraints on the \(\z_{j}\), this expression is maximized when the \(\z_{j}\) are distributed on pairwise orthogonal subspaces, and also each \(\z_{j}\) has isotropic (or nearly isotropic) covariance on its subspace. 

In practice, we do not know the distribution of the data, and the features are not perfectly a mixture of Gaussians. Still, the mixture of Gaussians is often a reasonable model for lower-dimensional feature distributions \citep{OriginalRateReduction,OriginalMCR2,LDR}, so we use the Gaussian form for the approximate coding rate. 

Also, in practice we may not have access to the full distribution of data, and so we need to estimate all relevant quantities via a finite sample. For Gaussians, \(R_{\eps}\) is only a function of \(\z\) through its covariance \(\bSigma\); in practice, this covariance is estimated via a finite sample \(\Z \in \R^{\dz \times n}\), assumed to be centered, as \(\Z\Z^{\top}/n\). This also allows us to estimate \(\Delta R_{\eps}(\cdot, \cdot)\) from a finite sample. To estimate \(\Delta R_{\eps}(\cdot \mid \cdot)\), we also need to estimate \(\bm{\pi}\). For this, we require finite sample information \(\bPi\) telling us which samples correspond to which random vector \(\z_{j}\). Denote by \(n_{j} \geq 1\) the number of samples in \(\Z\) which correspond to \(\z_{j}\). Then \(\bm{\pi}\) may be estimated via plug-in as \(\hat\pi_{j} = \frac{n_{j}}{n}\). 

This set of approximations yields estimates \(R_{\eps}(\Z)\), \(\Delta R_{\eps}(\Z_{1}, \Z_{2})\), and \(\Delta R_{\eps}(\Z \mid \bPi)\). Henceforth, we drop the \(\eps\) subscript and use the natural logarithm instead of the base-\(2\) logarithm. In this notation, the expressions for Gaussian \(\z\) and \(\z_{j}\), which we use in practice, are:
\begin{align*}
    R(\Z)
    &= \frac{1}{2}\logdet{\I_{\dz} + \frac{\dz}{n\eps^{2}}\Z\Z^{\top}}, \\
    \Delta R(\Z_{1}, \Z_{2})
    &= R(\mat{\Z_{1}, \Z_{2}}) - \frac{1}{2}R(\Z_{1}) - \frac{1}{2}R(\Z_{2}), \\
    \Delta R(\Z \mid \bPi)
    &= R(\mat{\Z_{1}, \dots, \Z_{k}}) - \sum_{j = 1}^{k}\frac{n_{j}}{n}R(\Z_{j}).
\end{align*}
We note here that although the assumption that \(\z\) and \(\z_{j}\) are Gaussian provides an information-theoretic interpretation of the coding rate and rate reduction, our results in this work \textit{do not} rely on anything being exactly distributed according to a mixture of Gaussians. This is because the proofs use purely the \textit{algebraic} properties of the coding rate approximations.

\subsection{Game Theoretic Formulation} \label{sub:game_theory}

Now that we have quantities in the representation space for the properties we want to encourage in the encoder and decoder, we now discuss how to train the encoder function \(f\) and decoder function \(g\).

Several methods, e.g., PCA, GANs \citep{goodfellow2014generative}, and the original CTRL formulation \citep{LDR}, can be viewed as learning the encoder (or discriminator) function \(f\) and decoder (or generator) function \(g\) via finding the Nash equilibria of an appropriate two-player \textit{simultaneous} game between the encoder and decoder. More discussion on this general framework can be found in \Cref{sec:pca_gan_games}. In this work, we approach this problem from a different angle; we learn the encoder function \(f\) and decoder function \(g\) via an appropriate two-player \textit{sequential} game between the encoder and the decoder; finding the so-called \textit{Stackelberg equilibria}. We now cover the basics of sequential game theory; a more complete treatment is found in \cite{basar1998dynamic}.

In a sequential game between the encoder --- whose move corresponds to picking \(f \in \F\) --- and decoder --- whose move corresponds to picking \(g \in \G\) --- both the encoder and the decoder attempt to maximize their own objectives, the so-called \textit{utility functions} \(u_{\enc} \colon \F \times \G \to \R\) and \(u_{\dec} \colon \F \times \G \to \R\) respectively, by making their moves one at a time. In our formulation, the \textit{encoder moves first}, since it is the role of the decoder to invert the encoder.

The solution concept for sequential games --- that is, the encoder and decoder that may be learned by an algorithm -- is the \textit{Stackelberg equilibrium} \citep{basar1998dynamic,fiez2019convergence, jin2019local}. In our context, \((f_{\star}, g_{\star})\) is a Stackelberg equilibrium if and only if
\begin{align*}
    f_{\star} &\in \argmax_{f \in \F}\inf\bigg\{u_{\enc}(f, g)\ \bigg|\ g \in \argmax_{\widetilde{g} \in \G}u_{\dec}(f, \widetilde{g})\bigg\},\\
    g_{\star} &\in \argmax_{g \in \G}u_{\dec}(f_{\star}, g).
\end{align*}
The sequential notion of the game is reflected in the definition of the equilibrium; the decoder, going second, may play \(g\) to maximize \(u_{\dec}\) with full knowledge of the encoder's play \(f\) (assuming the encoder plays rationally), while the encoder plays \(f\) to maximize \(u_{\enc}\) with only the knowledge that the decoder will play optimally.

%% file: ctrl_msp_theory.tex
\section{Multiple-Subspace Pursuit via the CTRL Framework} \label{sec:ctrl_msp}

With this background in place, we shortly introduce the closed-loop multi-subspace pursuit (CTRL-MSP) method. Recall that we seek to learn an encoder function \(f_{\star} \in \F\) and a decoder \(g_{\star} \in \G\) with the following desiderata:
\begin{itemize}
    \item The encoder function \(f_{\star}\) is \textit{injective} on the union of data subspaces \(\bigcup_{j = 1}^{k}\S_{j}\);
    \item The encoder function \(f_{\star}\) is \textit{discriminative} between the data subspaces \(\S_{j}\); 
    \item The encoder and decoder functions \(f_{\star}\) and \(g_{\star}\) form a \textit{self-consistent} closed-loop autoencoding.
\end{itemize}
Moreover, as discussed in \Cref{sub:game_theory}, we seek to learn \(f_{\star}\) and \(g_{\star}\) as equilibria for a two-player sequential game.

Before introducing the game whose equilibria are encoder and decoder functions with these properties, we first discuss how to quantify our desiderata. As a notation, for two sets \(U, V\), a map \(h \colon U \to V\), and a subset \(W \subseteq U\), we denote by \(h(W) = \{h(u) \mid u \in W\}\) the image of \(W\) under \(h\).
\begin{itemize}
	\item To enforce the \textit{injectivity} of the encoder, we aim to ensure that each \(f_{\star}(\S_{j})\) is a linear subspace of dimension equal to that of \(\S_{j}\), and furthermore, we aim to enforce that the covariance matrix of each \(f_{\star}(\X_{j})\) should have no small nonzero singular values. The first property means that the encoder is mathematically injective, i.e., \(f_{\star}\) does not map two points to the same representation. The second property means that the representations \(f_{\star}(\X_{j})\) are spread out across all directions of the subspace, thus ensuring that \(f_{\star}\) does not map two \textit{distant} points in the same subspace to \textit{close} representations, ensuring \textit{well-behaved} (i.e., not pathological) injectivity.
	
	\item To enforce the \textit{discriminativeness} of the encoder, we aim to ensure that the \(f_{\star}(\S_{j})\) are pairwise orthogonal subspaces. This property means that the \(f_{\star}(\S_{j})\) are statistically incoherent, ensuring that a given sample \(\x^{i}\) can be cleanly assigned to one of the subspaces \(\S_{j}\) based on the statistical correlations between its representation \(f_{\star}(\x^{i})\) and vectors from each representation subspace \(f_{\star}(\S_{j})\).
	
	\item To enforce internal \textit{self-consistency of the autoencoding}, we aim to have the subspace representations \(f_{\star}(\S_{j})\) and the autoencoded subspace representations \((f_{\star} \circ g_{\star} \circ f_{\star})(\S_{j})\) be equal. This property means that the decoder has accurately learned the linear structure of the representation space induced by the encoder.
\end{itemize}

Recall that \cite{OriginalMCR2} shows that maximizing \(\Delta R(\Z \mid \bPi)\) over \(\Z\) subject to normalization on the \(\Z_{j}\) provides the first two desiderata: in particular we have that, at optimum, \(\rank{\Z_{j}} = \dSj\), each \(\Z_{j}\) has \(\dSj\) large singular values where at least \(\dSj - 1\) of them are equal, and the spans of the columns of the \(\Z_{j}\) form orthogonal subspaces. This motivates that our encoder should maximize \(\Delta R(f(\X_{j}) \mid \bPi)\) over \(f \in \F\), where \(\F\) is an appropriate set of functions with normalization constraints.

Regarding the third desideratum, the following lemma motivates minimizing \(\Delta R(\Z_{1}, \Z_{2})\) over \(\Z_{1}\) and \(\Z_{2}\). As a notation, for a matrix \(\bm{A}\), we denote by \(\Col{\bm{A}}\) the linear span of the columns of \(\bm{A}\).
\begin{lemma}\label{lem:min_delta_R_distance}
    Suppose \(\Z_{1}, \Z_{2} \in \R^{\dz \times n}\). Then \(\Delta R(\Z_{1}, \Z_{2}) \geq 0\). Furthermore, if \(\Delta R(\Z_{1}, \Z_{2}) = 0\), then \(\Col{\Z_{1}} = \Col{\Z_{2}}\).
\end{lemma}
\begin{proof}
    By Lemma A.4 of \cite{OriginalMCR2}, we have \(\Delta R(\Z_{1}, \Z_{2}) \geq 0\), with equality if and only if \(\Z_{1}\Z_{1}^{\top} = \Z_{2}\Z_{2}^{\top}\), implying that \(\Col{\Z_{1}} = \Col{\Z_{2}}\).
\end{proof}

This lemma motivates that our decoder should minimize \(\Delta R(f(\X_{j}), (f \circ g \circ f)(\X_{j}))\), for all \(j \in \{1, \dots, k\}\), over \(g \in \G\). In line with the CTRL framework, the encoder should attempt to \textit{maximize} this quantity over \(f \in \F\); this is because the encoder conceptually plays a dual role as a GAN-type discriminator and thus seeks to distinguish the representations of the true data \(f(\X)\) from the representations of the autoencoded data \((f \circ g \circ f)(\X)\), especially in a distributional sense.

This discussion motivates the following game, which we call the closed-loop multi-subspace pursuit (CTRL-MSP) game. As notation, let \(\L(\R^{a}, \R^{b})\) be the set of linear maps from \(\R^{a}\) to \(\R^{b}\). Let \(\norm{\cdot}_{F} \colon \bm{A} \mapsto \sqrt{\sum_{i,j}\bm{A}_{ij}^{2}}\) be the usual Frobenius norm on matrices.

\begin{definition}[CTRL-MSP Game]\label{def:ctrl_msp}
    The CTRL-MSP game is a two-player sequential game between: 
    \begin{enumerate}
        \item The encoder, moving first, choosing functions \(f\) in the function class
        \begin{equation*}
            \F \doteq \{f \in \L(\R^{\dx}, \R^{\dz}) \mid \norm{f(\X_{j})}_{F}^{2} \leq n_{j},\ \forall j \in \{1, \dots, k\}\},
        \end{equation*}
        and having utility function
        \begin{equation*}
            u_{\enc}(f, g) \doteq \Delta R(f(\X) \mid \bPi) + \sum_{j = 1}^{k}\Delta R(f(\X_{j}), (f \circ g \circ f)(\X_{j})).
        \end{equation*}
        \item The decoder, moving second, choosing functions \(g\) in the function class \(\G \doteq \L(\R^{\dz}, \R^{\dx})\), and having utility function
        \begin{equation*}
            u_{\dec}(f, g) \doteq -\sum_{j = 1}^{k}\Delta R(f(\X_{j}), (f \circ g \circ f)(\X_{j})).
        \end{equation*}
    \end{enumerate}
\end{definition}

We now explicitly characterize the Stackelberg equilibria in the CTRL-MSP game, and show how they connect to each of the desiderata. As notation, let \(\sigma_{i}(\bm{A})\), \(1 \leq i \leq \min\{m, n\}\) be the singular values of \(\bm{A}\) sorted in non-increasing order.  Finally, for subspaces \(S_{1}, S_{2}\), denote by \(S_{1} + S_{2}\) the sum vector space \(\{\bm{s}_{1} + \bm{s}_{2} \mid \bm{s}_{1} \in S_{1}, \bm{s}_{2} \in S_{2}\}\). With these notations, our key assumptions are summarized below:

\begin{assumption}[Assumptions in CTRL-MSP Games]\label{a:ctrl_msp}
    \phantom{}
	\begin{enumerate}
	    \item (Multiple classes.) \(k \geq 2\).
		\item (Informative data.) For each \(j \in \{1, \dots, k\}\), \(\Col{\X_{j}} = \S_{j}\).
		\item (Large enough representation space.) \(\sum_{j = 1}^{k}\dSj \leq \min\{\dx, \dz\}\).
		\item (Incoherent class data.) \(\sum_{j = 1}^{k}\dSj = \dim{\sum_{j = 1}^{k}\S_{j}}\).\footnote{An intuitive understanding of this condition is that if we take a linearly independent set from each \(\S_{j}\), the union of all these sets is still linearly independent.}
		\item (High coding precision.) \(\eps^{4} \leq \min_{j = 1}^{k}(n_{j}/n\cdot \dz^{2}/\dSj^{2})\).
	\end{enumerate}
\end{assumption}
Our main result is:

\begin{theorem}[Stackelberg Equilibria of CTRL-MSP Games]\label{thm:ctrl_msp}
If Assumption \ref{a:ctrl_msp} holds, then the CTRL-MSP game has the following properties:
	\begin{enumerate}
		\item A Stackelberg equilibrium \((f_{\star}, g_{\star})\) exists.
		\item Any Stackelberg equilibrium \((f_{\star}, g_{\star})\) enjoys the following properties:
		\begin{enumerate}
			\item (Injective encoder.) For each \(j \in \{1, \dots, k\}\), we have that \(f_{\star}(\S_{j})\) is a linear subspace of dimension \(\dSj\). Further, for each \(j \in \{1, \dots, k\}\), one of the following holds:
			\begin{itemize}
			    \item \(\sigma_{1}(f_{\star}(\X_{j})) = \cdots = \sigma_{\dSj}(f_{\star}(\X_{j})) = \frac{n_{j}}{\dSj}\); or
			    \item \(\sigma_{1}(f_{\star}(\X_{j})) = \cdots = \sigma_{\dSj - 1}(f_{\star}(\X_{j})) \in (\frac{n_{j}}{\dSj}, \frac{n_{j}}{\dSj - 1})\) and \(\sigma_{\dSj}(f_{\star}(\X_{j})) > 0\), where if \(\dSj = 1\) then \(\frac{n}{\dSj - 1}\) is interpreted as \(+\infty\). 
			\end{itemize}
			\item (Discriminative encoder.) The subspaces \(\{f_{\star}(\S_{j})\}_{j = 1}^{k}\) are orthogonal.
			\item (Consistent encoding and decoding.) For each \(j \in \{1, \dots, k\}\), we have that \(f_{\star}(\S_{j}) = (f_{\star} \circ g_{\star} \circ f_{\star})(\S_{j})\).
		\end{enumerate}
	\end{enumerate}
\end{theorem}

The proof of this theorem requires \Cref{thm:ctrl_sg}, and is deferred to \Cref{sec:proofs} for brevity. Once in the framework of \Cref{thm:ctrl_sg}, the main difficulty is the characterization of the maximizers of \(f \mapsto \Delta R(f(\X) \mid \bPi)\). This function is non-convex and challenging to analyze; we proceed by carefully applying inequalities on the singular values of the representation matrices.

As the theorem indicates, the earlier check-list of desired quantitative properties can be achieved by CTRL-MSP. That is, CTRL-MSP provably learns injective and discriminative representations of multiple-subspace structure.

For the special case \(k = 1\), where we are learning a single-subspace structure, it is possible to change the utility functions and function classes of CTRL-MSP, to learn a \textit{different} set of properties that more closely mirrors PCA. In particular, a Stackelberg equilibrium encoder \(f_{\star}\) of this modified game does not render the covariance of \(f(\X_{1})\) nearly-isotropic; it instead is an explicit \(\ell^{2}\)-isometry on \(\S_{1}\), which ensures well-behaved injectivity. The details are left to \Cref{sec:ctrl_ssp}.

We now discuss an implication of the CTRL-MSP method. The original problem statement of learning discriminative representations for multiple subspace structure may be solved directly via orthogonalizing the representations produced by using PCA on each data subspace. This solution is a discrete and ad-hoc procedure that is far divorced from modern representation learning. However, CTRL-MSP provides an alternative approach: simultaneously learning and representing the subspaces via the modern representation learning toolkit within a continuous optimization framework. This gives a unifying perspective on classical and modern representation learning, by showing that classical methods can be viewed as \textit{special cases} of modern methods, and that they may be formulated to learn the same types of representations. A major benefit (discussed further in \Cref{sec:generalization,sec:experiments}) is that the new formulation computationally can be generalized to much broader families of structures, beyond subspaces to submanifolds, as compelling empirical evidence from \cite{LDR} demonstrates. 

%% file: ctrl_sg_theory.tex
\section{Generalization via CTRL-SG} \label{sec:generalization}

We now generalize CTRL-MSP to representation learning scenarios more diverse than our task of learning multiple linear subspace structure. This generalized method, which we call \textit{closed-loop sequential games} (CTRL-SG), builds on sequential game theory and the CTRL framework \citep{LDR}. 

To introduce the CTRL-SG game framework, we make the following changes from the CTRL-MSP game:
\begin{itemize}
    \item Change \(\F\) and \(\G\) from restricted classes of linear operators to general function classes.
    
    \item Change the functional \(f \mapsto \Delta R(f(\X) \mid \bPi)\), which quantifies the property of the encoder function \(f \in \F\) to be injective and discriminative, to a general functional \(\Q \colon \F \to \R\) which conceptually quantifies the quality of the representations induced by the encoder.
    
    \item Change the functional \((f, g) \mapsto -\sum_{j = 1}^{k}\Delta R(f(\X_{j}), (f \circ g \circ f)(\X_{j}))\), which quantifies the the similarity of the representation subspaces with the autoencoded representation subspaces, to a general functional \(\C \colon \F \times \G \to \R\) which conceptually quantifies the consistency of the closed-loop autoencoding.
\end{itemize}

This gives the following, vastly more general, game framework.

\begin{definition}[CTRL-SG Game]\label{def:ctrl_sg}
    The CTRL-SG game is a two-player sequential game between: 
    \begin{enumerate}
        \item The encoder, moving first, choosing functions \(f\) in the function class \(\F\), and having utility function \(u_{\enc}(f, g) \doteq \Q(f) - \C(f, g)\).
        \item The decoder, moving second, choosing functions \(g\) in the function class \(\G\), and having utility function \(u_{\dec}(f, g) \doteq \C(f, g)\).
    \end{enumerate}
\end{definition}

We may generically characterize the Stackelberg equilibria of CTRL-SG games, given mild regularity conditions on the choices of \(\F\), \(\G\), \(\Q\), and \(\C\).

\begin{assumption}[Assumptions in CTRL-SG Games]\label{a:ctrl_sg}
    \phantom{}
	\begin{enumerate}
		\item (Quality can be maximized.) \(\argmax_{f \in \F}\Q(f)\) is nonempty.
		\item (Consistency can be maximized.) \(\argmax_{g \in \G}\C(f, g)\) is nonempty for every \(f \in \F\).
		\item (The decoder can obtain equally good outcomes regardless of the encoder's play.) The function \(f \mapsto \max_{g \in \G}\C(f, g)\) is constant. 
	\end{enumerate}
\end{assumption}

With these assumptions, we now characterize the Stackelberg equilibria.

\begin{theorem}[Stackelberg Equilibria of CTRL-SG]\label{thm:ctrl_sg}
    If Assumption \ref{a:ctrl_sg} holds, then the CTRL-SG game has the following properties:
	\begin{enumerate}
		\item A Stackelberg equilibrium \((f_{\star}, g_{\star})\) exists.
		\item Any Stackelberg equilibrium \((f_{\star}, g_{\star})\) enjoys:
		\begin{equation*}
			f_{\star} \in \argmax_{f \in \F}\Q(f), \qquad g_{\star} \in \argmax_{g \in \G}\C(f_{\star}, g).
		\end{equation*}
	\end{enumerate}
\end{theorem}
\begin{proof}
    We show both consequences of the theorem at the same time, by first computing the equilibria \(f_{\star}\), then computing the corresponding \(g_{\star}\). 
    By Assumption \ref{a:ctrl_sg}.3, the function \(f \mapsto \max_{g \in \G}\C(f, g)\) is constant; say equal to \(c \in \R\).
    Then we have 
    \begin{align*}
        &\argmax_{f \in \F}\inf\bigg\{u_{\enc}(f, g)\bigg| g \in \argmax_{\widetilde{g} \in \G}u_{\dec}(f, \widetilde{g})\bigg\} \\
        &= \argmax_{f \in \F}\inf\bigg\{\Q(f) - \C(f, g) \bigg| g \in \argmax_{\widetilde{g} \in \G}\C(f, \widetilde{g})\bigg\} \\
        &= \argmax_{f \in \F}\left[\Q(f) - \sup\bigg\{\C(f, g) \bigg| g \in \argmax_{\widetilde{g} \in \G}\C(f, \widetilde{g})\bigg\}\right] \\
        &= \argmax_{f \in \F}\left[\Q(f) - \max_{g \in \G}\C(f, g)\right]
        = \argmax_{f \in \F}\left[\Q(f) - c\right] 
        = \argmax_{f \in \F}\Q(f).
    \end{align*}
    By \Cref{a:ctrl_sg}.1, this set is nonempty. Suppose \(f_{\star}\) is in this set. Then
    \begin{equation*}
        \argmax_{g \in \G}u_{\dec}(f_{\star}, g) = \argmax_{g \in \G}\C(f_{\star}, g)
    \end{equation*}
    and by Assumption \ref{a:ctrl_sg}.2, this set is also nonempty. Thus, a Stackelberg equilibrium exists. If \((f_{\star}, g_{\star})\) is a Stackelberg equilibrium, then \(f_{\star} \in \argmax_{f \in \F}\Q(f)\) by the first calculation, and \(g_{\star} \in \argmax_{g \in \G}\C(f_{\star}, g)\) by the second calculation, confirming the remaining part of the theorem.
\end{proof}

The generalized CTRL-SG system allows us to use the CTRL framework for representation learning, choose principled objective functions to encourage the desired representation, and then explicitly characterize the optimal learned encoder and decoder for that algorithm. It also suggests principled optimization strategies and algorithms, such as GDMax \citep{jin2019local}, for obtaining these optimal functions.

This system generalizes the original setting of learning from a fixed finite labelled dataset. Since it is a purely game-theoretic formulation, in principle, one may adapt it to other learning contexts than the ones developed here, e.g., semi-supervised learning and online/incremental learning.

%% file: experiments.tex
\section{Empirical Evaluation for CTRL-MSP and CTRL-SG} \label{sec:experiments}

In this section, we demonstrate empirical convergence, via appropriate algorithms, of learned \(f\) and \(g\) to Stackelberg equilibria of the CTRL-MSP game which satisfy the conclusions of \Cref{thm:ctrl_msp}, in the case of data lying on coherent subspaces \(\S_{j}\) and possessing rich correlation structure. We then demonstrate CTRL-MSP's robustness to noise. Next, we apply CTRL-MSP to the popular image dataset MNIST, and show partial success in achieving the desired equilibria even far outside the scope of the CTRL-MSP assumptions. Finally, we replace the function classes \(\F\) and \(\G\), which in CTRL-MSP are sets of linear operators, with function classes corresponding to simple, but nonlinear, neural networks. We empirically demonstrate convergence to appropriate equilibria which satisfy the conclusions of \Cref{thm:ctrl_sg}.

\subsection{Optimization Algorithm}

Both CTRL-MSP and CTRL-SG operate in the framework of sequential games. Thus, our implementation uses an optimization strategy which is informed by previous works on convergence to Stackelberg equilibria. Here, we use the GDMax algorithm proposed in \cite{jin2019local}, which, in our context, alternates between taking one optimization step on \(f\) to maximize the encoder utility \(u_{\enc}(\cdot, g)\) and taking optimization steps on \(g\) to maximize the decoder utility \(u_{\dec}(f, \cdot)\) until the decoder utility converges. For the sake of reducing the runtime, we instead take a fixed number \(L \gg 1\) of decoder optimization steps; empirically the decoder nearly converges in this number of iterations.

\subsection{Ways to Evaluate Results}

In both CTRL-MSP and the CTRL-SG instance discussed above, we have the same claims for our Stackelberg equilibria \((f_{\star}, g_{\star})\), furnished by \Cref{thm:ctrl_msp}, that we may test empirically. As a notation, let \(\norm{\cdot}_{\ell^{2}}\) denote the standard \(\ell^{2}\) norm.
\begin{itemize}
    \item We wish to show that the encoder is injective, i.e., each \(f_{\star}(\S_{j})\) is a linear subspace of dimension \(\dSj\), and that each \(f_{\star}(\X_{j})\) has nearly isotropic covariance. To show this, we plot the singular value distribution of \(f_{\star}(\X_{j})\) and show that the \(\sigma_{p}(f_{\star}(\X_{j}))\) are large and clustered around one or two values, for \(1 \leq p \leq \dSj\).
    
    \item We wish to show that the encoder orthogonalizes the data subspaces, i.e., the \(f_{\star}(\S_{j})\) are orthogonal. To show this, we plot the absolute cosine similarity (or, magnitude of the correlation coefficient), defined by 
    \begin{equation*}
        \texttt{acs}(\z_{1}, \z_{2}) \doteq \frac{\abs{\z_{1}^{\top}\z_{2}}}{\norm{\z_{1}}_{\ell^{2}}\norm{\z_{2}}_{\ell^{2}}}
    \end{equation*}
    across all representations of data points \(f_{\star}(\x^{p})\) and \(f_{\star}(\x^{q})\). We order our dataset so that each class \(\X_{j}\) is a continuous block of data points, i.e., \(\X = \mat{\X_{1} & \X_{2} & \cdots & \X_{k}}\). We show that the plot of the absolute cosine similarities forms a block diagonal, showing that \(\texttt{acs}(f_{\star}(\x^{p}), f_{\star}(\x^{q})) \approx 0\) when \(\x^{p}, \x^{q}\) are drawn from different subspaces, and is a large value if they are drawn from the same subspace.
    
    \item We wish to show that the encoder and decoder form a self-consistent autoencoding, i.e., that \(f_{\star}(\S_{j}) = (f_{\star} \circ g_{\star} \circ f_{\star})(\S_{j})\) for all \(j\). Since \(f_{\star}(\S_{j}) = \Col{f_{\star}(\X_{j})}\) and \((f_{\star} \circ g_{\star} \circ f_{\star})(\S_{j}) = \Col{(f_{\star} \circ g_{\star} \circ f_{\star})(\X_{j})}\), we may compare the images of the subspace \(\S_{j}\) by comparing the images of the data \(\X_{j}\). In particular, we show that the distribution of residuals
    \begin{equation*}
        \texttt{resid}_{j}(f_{\star}(\x^{i})) \doteq \norm{f_{\star}(\x^{i}) - \operatorname{proj}_{\Col{(f_{\star} \circ g_{\star} \circ f_{\star})(\X_{j})}}(f_{\star}(\x))}_{\ell^{2}}
    \end{equation*}
    for all \(\x^{i}\) which are columns of \(\X_{j}\), is concentrated near zero, for all \(j\).
\end{itemize}

\subsection{CTRL-MSP with Data on Subspaces}

To generate data which lies on subspaces with rich correlation structure, we fix \(\nu \geq 0\) and \(\alpha \in [0, 1]\), then generate a single random matrix with orthonormal columns \(\U \in \R^{\dx \times (\max_{j}\dSj)}\) using the QR decomposition. Then, for each subspace, we select \(\dSj\) random columns of \(\U\) uniformly without replacement to form a matrix \(\U_{j} \in \R^{\dx \times \dSj}\). We then generate random matrices \(\bm{\Theta}_{j} \in \R^{\dx \times \dSj}\), \(\bm{\Xi}_{j} \in \R^{\dSj \times n_{j}}\), \(\bm{\Phi}_{j} \in \R^{\dx \times n_{j}}\) whose entries are i.i.d.~standard normal random variables, and set \(\widetilde{\U}_{j}\) to be the matrix whose columns are the normalized columns of \((1 - \alpha)\U_{j} + \alpha \bm{\Theta}_{j}\). We finally obtain \(\X_{j}\) using the formula \(\X_{j} \doteq \widetilde{\U}_{j}\bm{\Xi}_{j} + \sqrt{\frac{\nu}{\dx}}\bm{\Phi}_{j}\).

If \(\alpha = 0\) then the \(\X_{j}\) are highly correlated; in particular, each class subspace \(\S_{j}\) is spanned by vectors from the same set \(\U\). Some subspaces will have basis vectors in common, making them impossible to orthogonalize by a linear encoder. If \(\alpha = 1\) then the \(\X_{j}\) are highly incoherent, since each class subspace \(\S_{j}\) is spanned by a random basis \(\bm{\Theta}_{j}\), and random vectors in high dimensions are highly incoherent with high probability \citep{Wright-Ma-2022}. In order to make the problem harder, and thus demonstrate the capacity of CTRL-MSP, we set \(\alpha\) to be low, close to \(0\).

For experiments, we set baseline values of \(n = 1500\), \(k = 3\), \(n_{1} = n_{2} = n_{3} = n/k = 500\), \(\dx = 50\),  \(d_{1} = 3\), \(d_{2} = 4\), \(d_{3} = 5\), \(\dz = 40\), \(\eps^{2} = 1\), \(\alpha = \frac{1}{10}\), and \(\nu = 0\). 

Regarding optimization strategy, we set the number of decoder steps per encoder step as \(L = 10^{3}\). We set the learning rate of \(f\) to be \(10^{-2}\) and the learning rate of \(g\) to be \(10^{-3}\). We use the Adam optimizer \citep{kingma2014adam} to optimize both \(f\) and \(g\). The data are randomly partitioned into mini-batches of size \(b = 50\) during optimization. We train for \(2\) epochs.

We observe success in the baseline regime: the \(\texttt{acs}(f_{\star}(\x^{i}), f_{\star}(\x^{j}))\) plot presents a significant block diagonal structure, the representation matrix \(f_{\star}(\X_{j})\) corresponding to each subspace \(\S_{j}\) has \(\dSj\) non-zero singular values, which are large, and the residuals \(\texttt{resid}_{j}(f_{\star}(\x^{i}))\) are concentrated near \(0\) (\Cref{fig:ctrl_msp_baseline}).

\begin{figure}
    \centering
    \begin{subfigure}[b]{0.24\textwidth}
        \centering 
        \includegraphics[width=\textwidth]{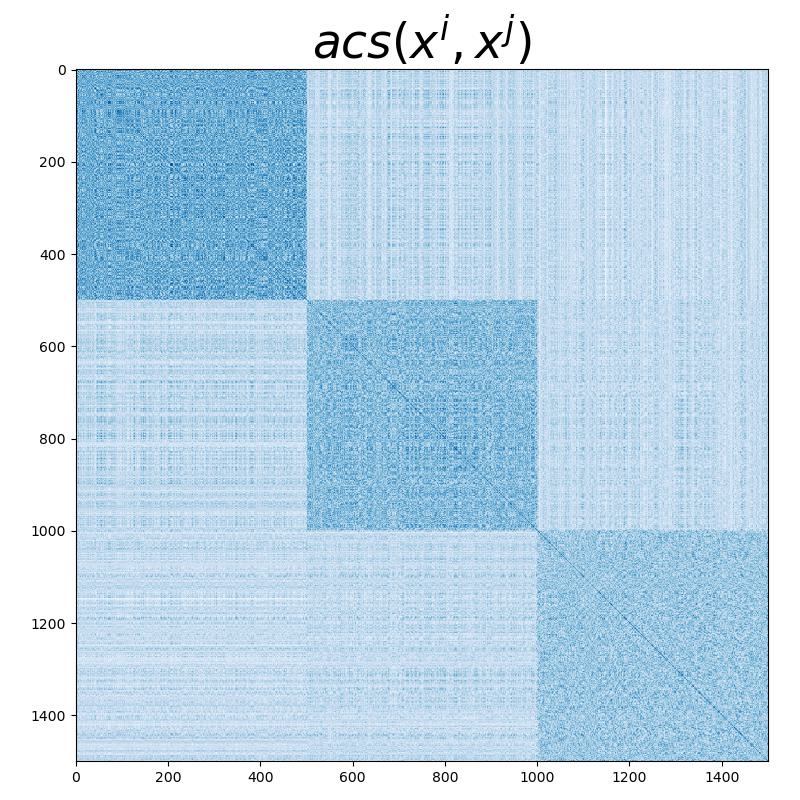}
        \caption{}
    \end{subfigure}
    \begin{subfigure}[b]{0.24\textwidth}
        \centering 
        \includegraphics[width=\textwidth]{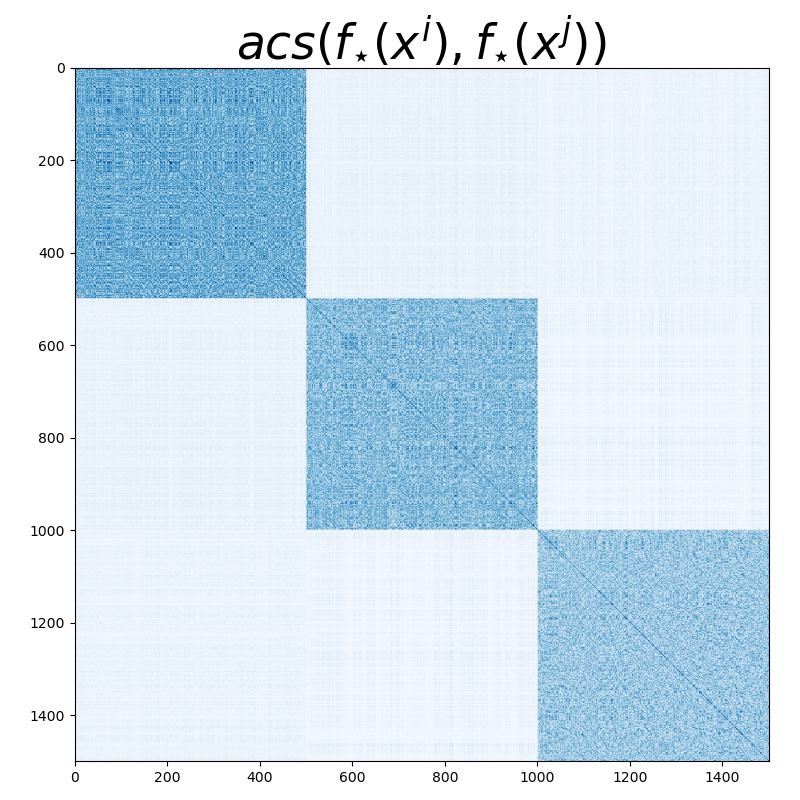}
        \caption{}
    \end{subfigure}
    \begin{subfigure}[b]{0.24\textwidth}
        \centering 
        \includegraphics[width=\textwidth]{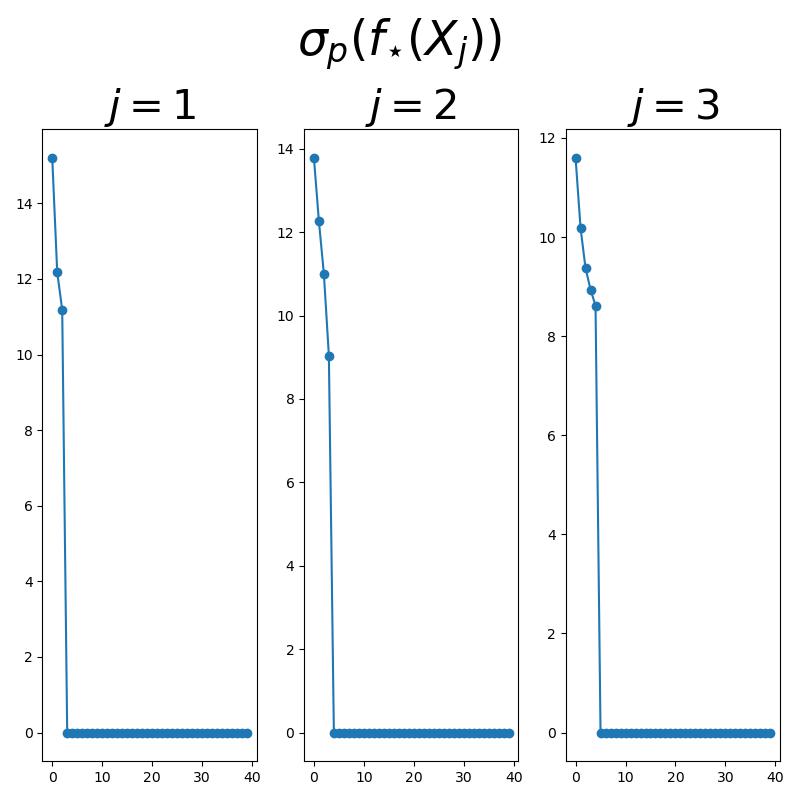}
        \caption{}
    \end{subfigure}
    \begin{subfigure}[b]{0.24\textwidth}
        \centering 
        \includegraphics[width=\textwidth]{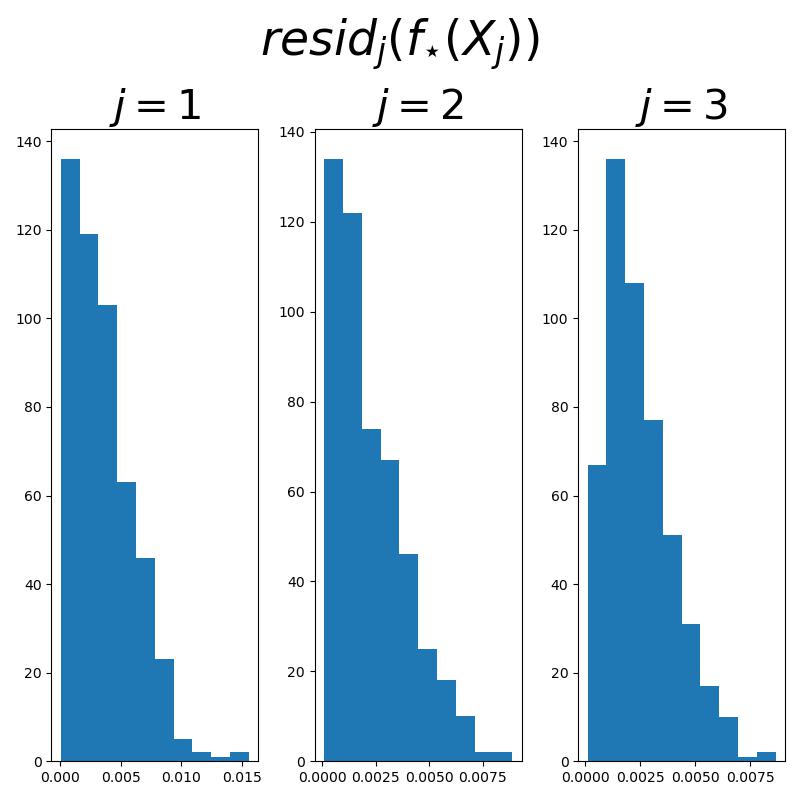}
        \caption{}
    \end{subfigure}
    \caption{Baseline performance of CTRL-MSP. (a) Heatmap of absolute cosine similarities of the original data \(\texttt{acs}(\x^{p}, \x^{q})\). (b) Heatmap of absolute cosine similarities of the learned representations \(\texttt{acs}(f_{\star}(\x^{p}), f_{\star}(\x^{q}))\). (c) Spectra of the representation matrices for each data subspace, i.e., distribution of \(\sigma_{p}(f_{\star}(\X_{j}))\) for each \(j\). (d) Histogram of the residuals \(\texttt{resid}_{j}(f_{\star}(\x^{i}))\) for all \(\x^{i}\) in \(\X_{j}\), for every \(j\).}
    \label{fig:ctrl_msp_baseline}
\end{figure}

We now test the impact of noise on CTRL-MSP. Specifically, we set \(\nu = 10^{-2}\). This adds off-subspace noise to the data, which then no longer satisfies the assumption that the data are distributed on linear subspaces. However, CTRL-MSP still clearly succeeds: the \(\texttt{acs}(f_{\star}(\x^{i}), f_{\star}(\x^{j}))\) plot presents a clearly visible block diagonal, the representation matrix \(f_{\star}(\X_{j})\) corresponding to each subspace \(\S_{j}\) has \(\dSj\) large non-zero singular values, and the residuals \(\texttt{resid}_{j}(f_{\star}(\x^{i}))\) are concentrated around \(0\) (\Cref{fig:ctrl_msp_nu_1e-2}).

\begin{figure}
    \centering
    \includegraphics[width=0.24\textwidth]{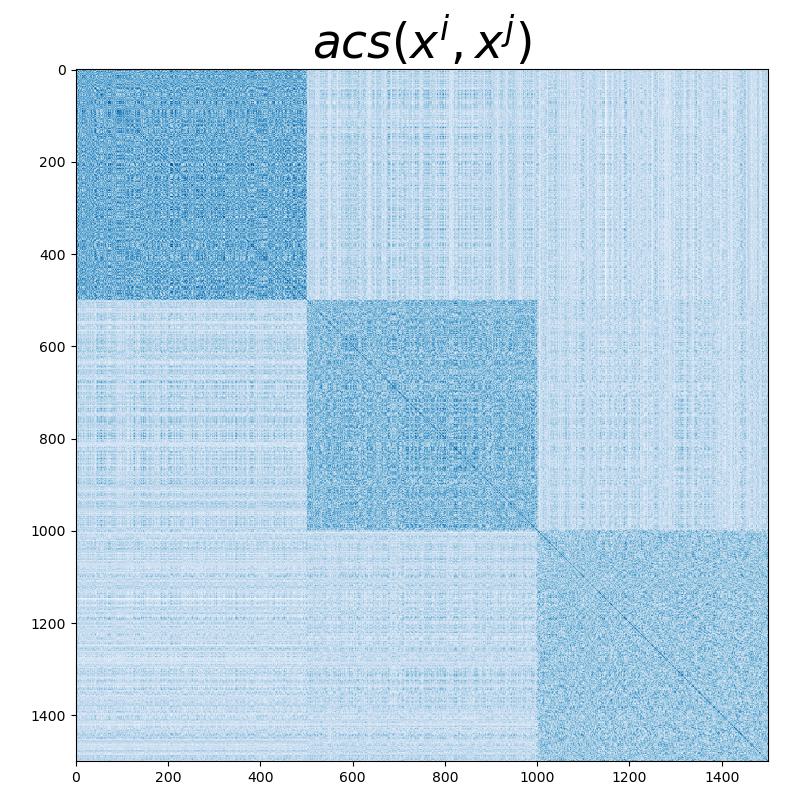}
    \includegraphics[width=0.24\textwidth]{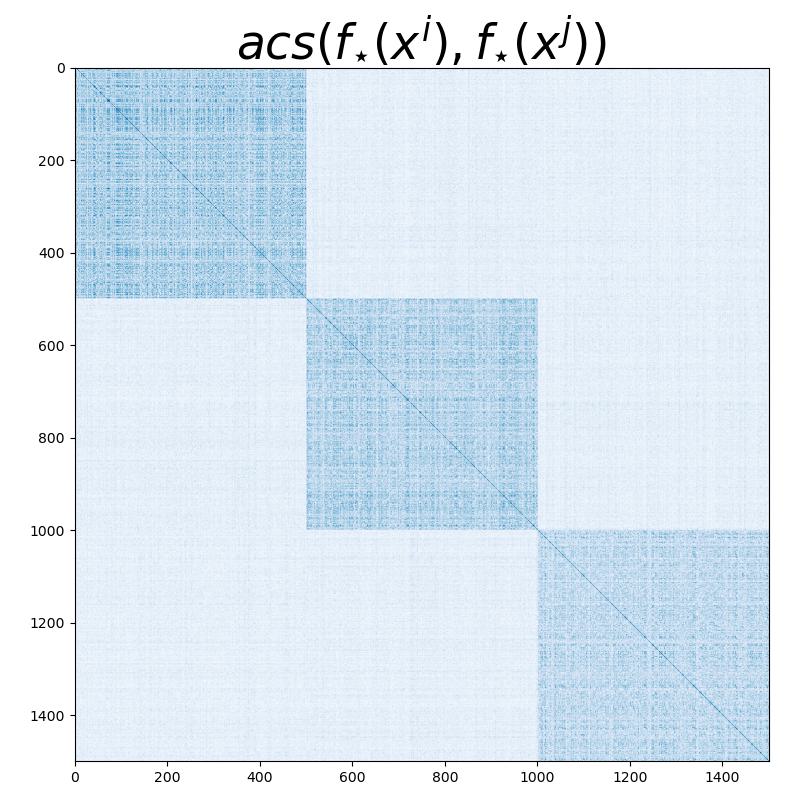}
    \includegraphics[width=0.24\textwidth]{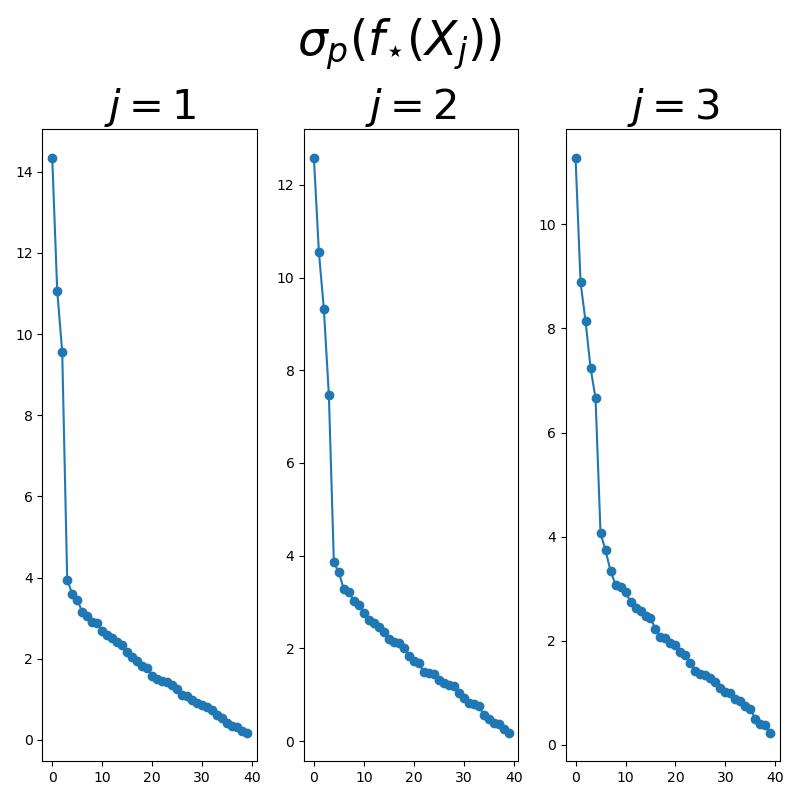}
    \includegraphics[width=0.24\textwidth]{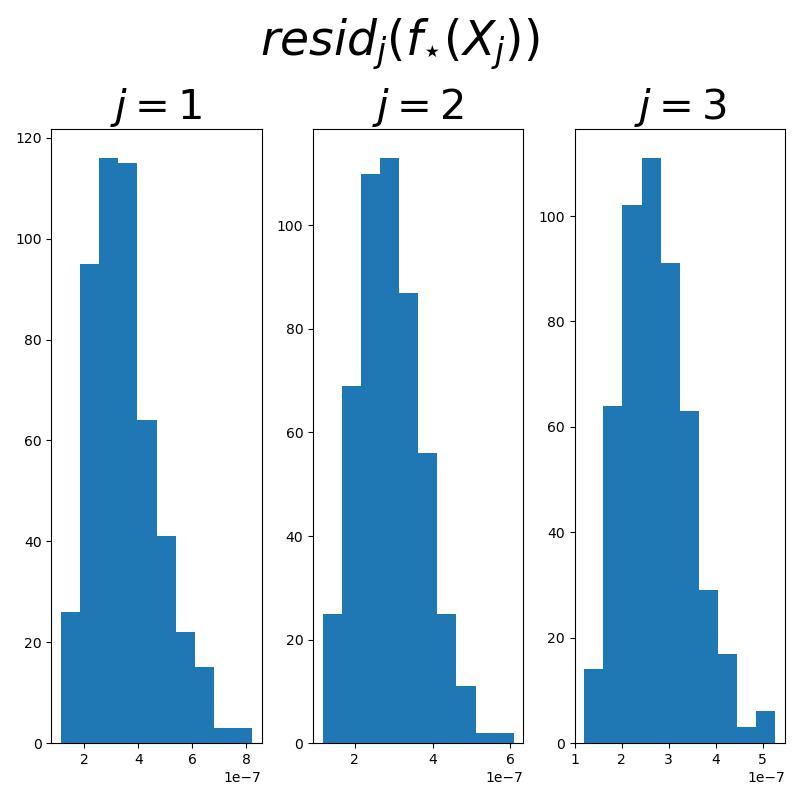}
    \caption{Performance of CTRL-MSP, \(\nu = 10^{-2}\).}
    \label{fig:ctrl_msp_nu_1e-2}
\end{figure}

\subsection{CTRL-MSP on MNIST Data}

Another setting we test CTRL-MSP on is the popular MNIST digit recognition dataset, where each \(28 \times 28\) black-and-white image can be viewed as a vector in \(\R^{784}\). This makes \(\dx = 784\). 

In this setting, we make the following changes to the optimization scheme for the sake of improving runtime:
\begin{itemize}
    \item Lower \(L\), the number of decoder optimization steps for each encoder optimization step, from \(10^{3}\) to \(2 \cdot 10^{2}\).
    \item Raise the batch size \(b\) from \(50\) to \(200\).
    \item Train for \(1\) epoch instead of \(2\) epochs.
\end{itemize}

In the MNIST dataset, the data lie not on linear subspaces \(\S_{j} \subseteq \R^{784}\) but \textit{nonlinear submanifolds} \citep{fefferman2013testing,OriginalMCR2,LDR}. Thus, our theoretical assumption that our data are distributed on linear subspaces is completely shattered. Still, CTRL-MSP with \(\dz = 150\) shows partial success in this regime; the block diagonal structure of the \(\texttt{acs}(f_{\star}(\x^{i}), f_{\star}(\x^{j}))\) plot is faint but visible, the representation matrices \(f_{\star}(\X_{j})\) each have a few large singular values, and the residuals \(\texttt{resid}_{j}(f_{\star}(\x^{i}))\) are concentrated around \(0\) (\Cref{fig:ctrl_msp_mnist}).

\begin{figure}
    \centering
    \includegraphics[width=0.3\textwidth]{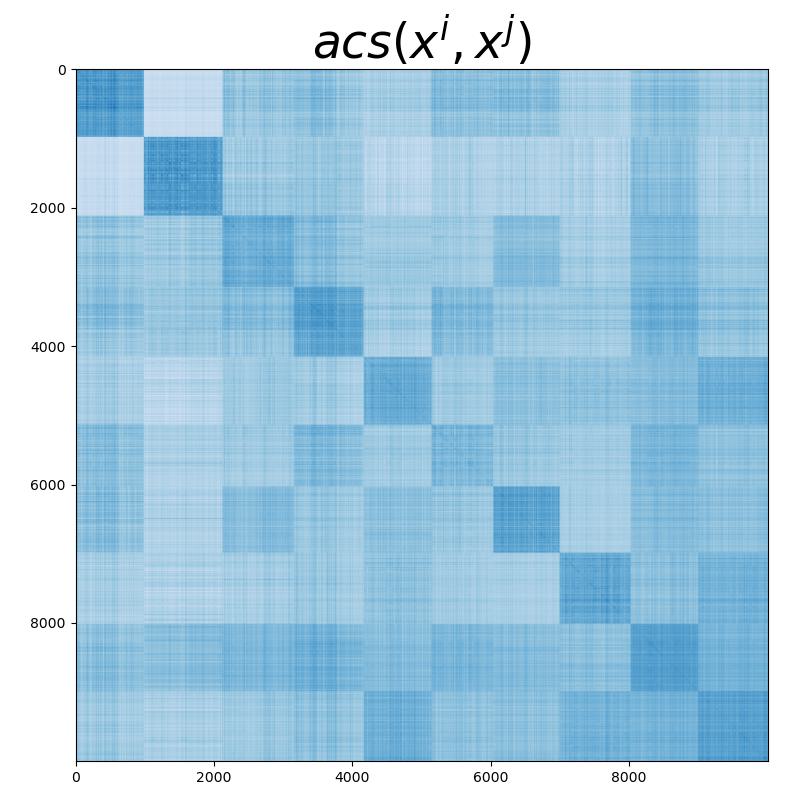}
    \includegraphics[width=0.3\textwidth]{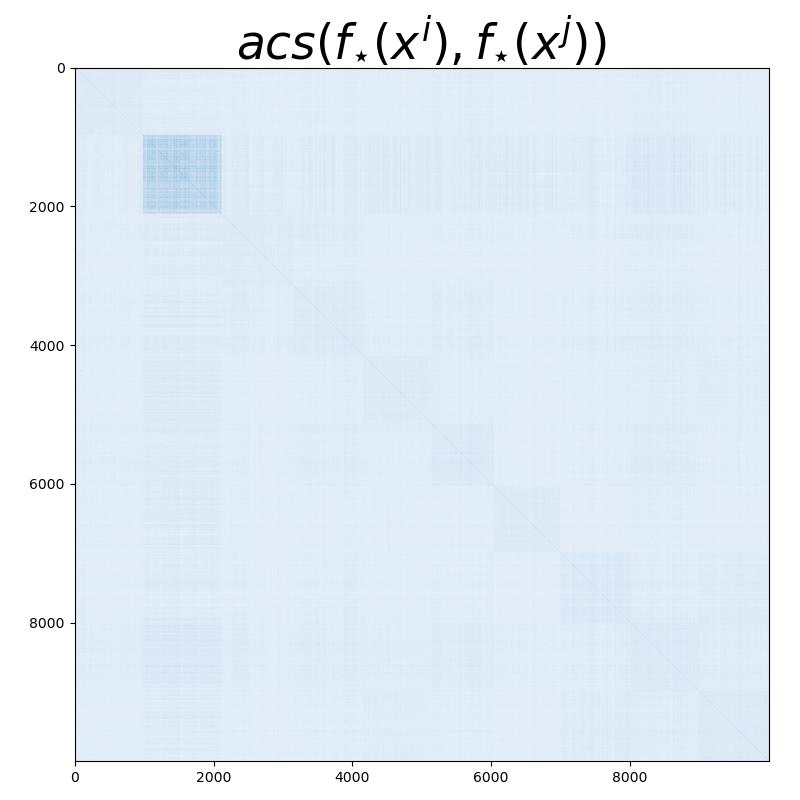}
    \includegraphics[width=\textwidth]{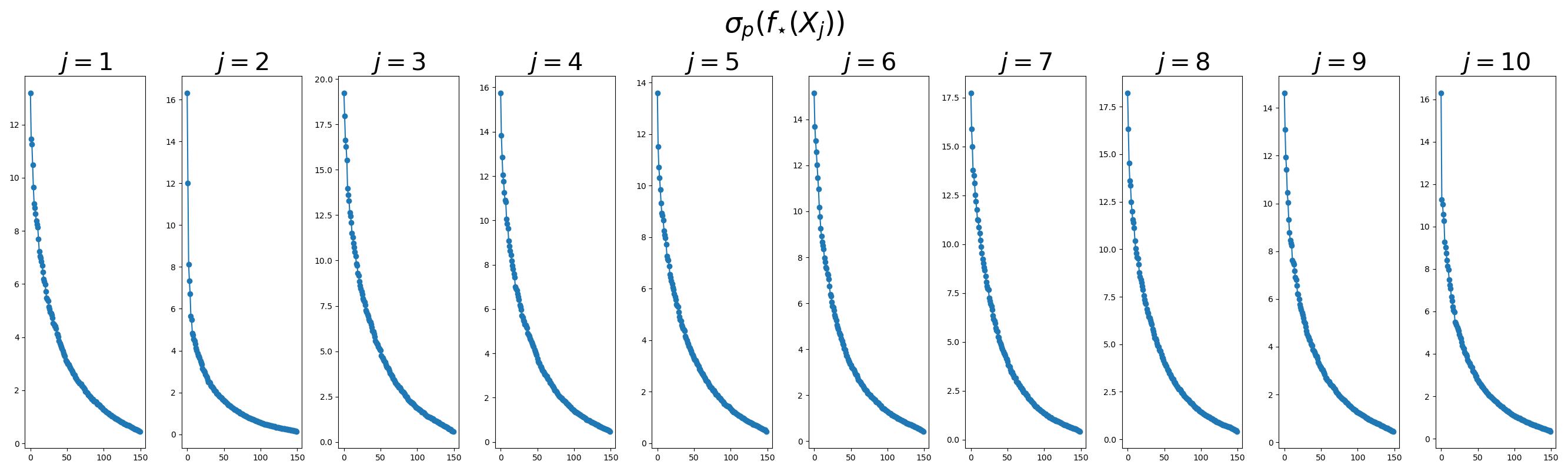}
    \includegraphics[width=\textwidth]{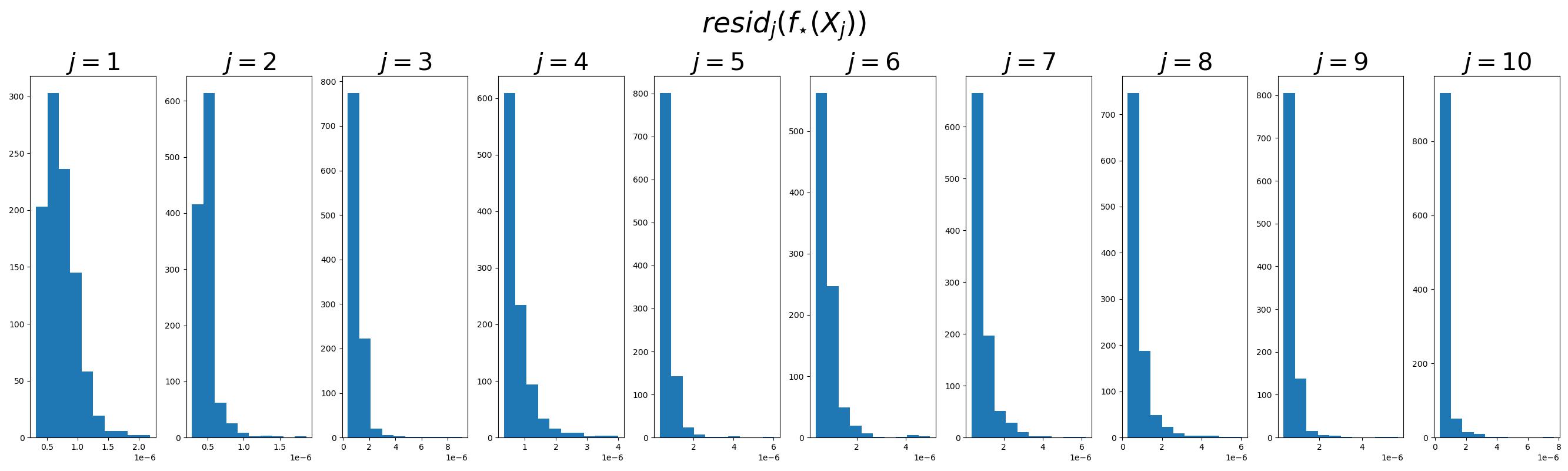}
    \caption{Behavior of CTRL-MSP on MNIST.}
    \label{fig:ctrl_msp_mnist}
\end{figure}

\subsection{CTRL-SG on MNIST Data}

To improve performance of our methods on real-world data, we turn to the more general CTRL-SG framework. Seeing as how the limiting factor of CTRL-MSP is the fact that the function classes \(\F\) and \(\G\) are linear maps, we replace \(\F\) and \(\G\) by function classes corresponding to neural networks; that is, we replace the linear encoder and decoder with simple two-layer feed-forward networks, where the representation dimension is still \(\dz = 150\), and the neural networks' latent dimension is \(d_{\mathrm{latent}} = 150\).

In order to make this choice of \(\F\) and \(\G\) fully principled, we need to verify that the choice of \(\F\), \(\G\), and the objective functions form an instance of CTRL-SG that satisfy \Cref{a:ctrl_sg} and thus enjoy the guarantees of \Cref{thm:ctrl_sg}. The difficult assumption to show is \Cref{a:ctrl_sg}.3, which claims that the function \(f \mapsto \max_{g \in \G}\C(f, g)\) is constant, where in this context \(\C(f, g) \doteq -\sum_{j = 1}^{k}\Delta R(f(\X_{j}), (f \circ g \circ f)(\X_{j}))\). We can check this empirically by verifying that (after a start-up period), the function \(\C(f, g)\) is approximately constant after optimizing over \(g\) (\Cref{fig:ctrl_sg_mnist_losses}).

\begin{figure}
    \centering
    \includegraphics[width=0.6\textwidth]{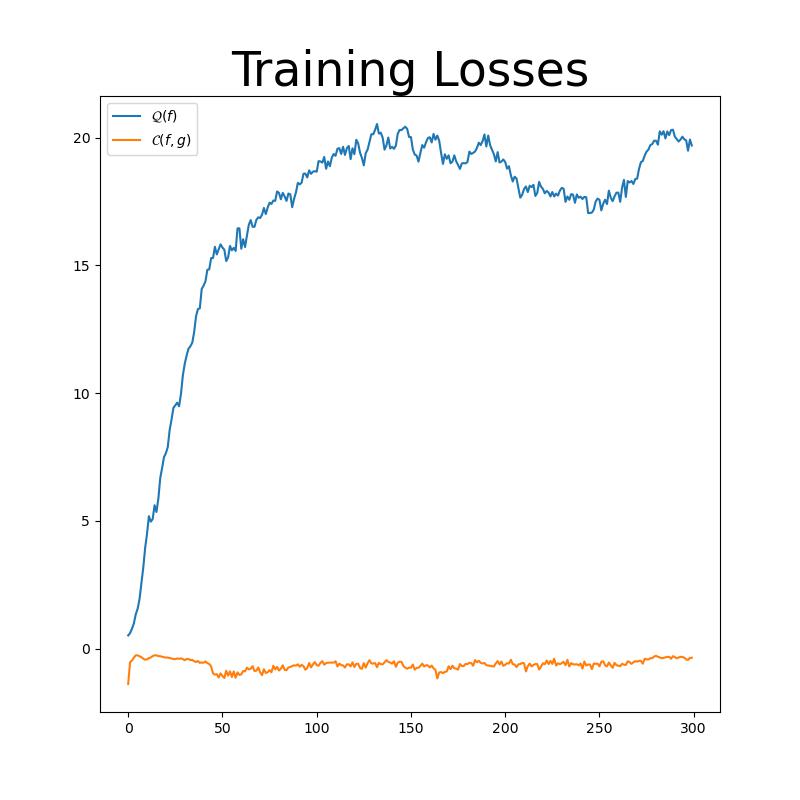}
    \caption{Loss functions of CTRL-SG while training on MNIST, where losses are computed after optimization over the decoder function \(g\). Notice that \(\C(f, g)\) does not depend on the \(f\) that is chosen (past some start-up period).}
    \label{fig:ctrl_sg_mnist_losses}
\end{figure}

Thus, \Cref{a:ctrl_sg} holds in practice, and thus the guarantees of \Cref{thm:ctrl_sg} also hold; that is, \(f_{\star}\) maximizes \(f \mapsto \Delta R(f(\X) \mid \bPi)\), and \(g_{\star}\) minimizes \(g \mapsto \sum_{j = 1}^{k}\Delta R(f_{\star}(\X_{j}), (f_{\star} \circ g \circ f_{\star})(\X_{j}))\). By Theorem A.6 of \cite{OriginalMCR2} and \Cref{lem:min_delta_R_distance}, we see that \(f_{\star}\) and \(g_{\star}\) obtain analogous guarantees to \Cref{thm:ctrl_msp}. Now, we see that our instance of CTRL-SG obtains much better performance than CTRL-MSP (\Cref{fig:ctrl_sg_mnist}).

\begin{figure}
    \centering
    \includegraphics[width=0.3\textwidth]{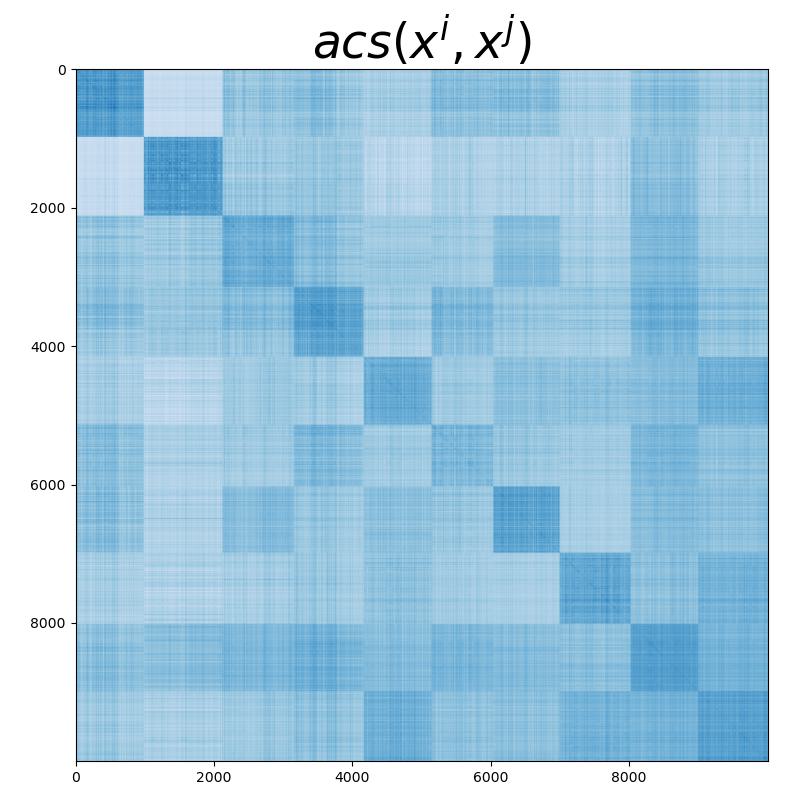}
    \includegraphics[width=0.3\textwidth]{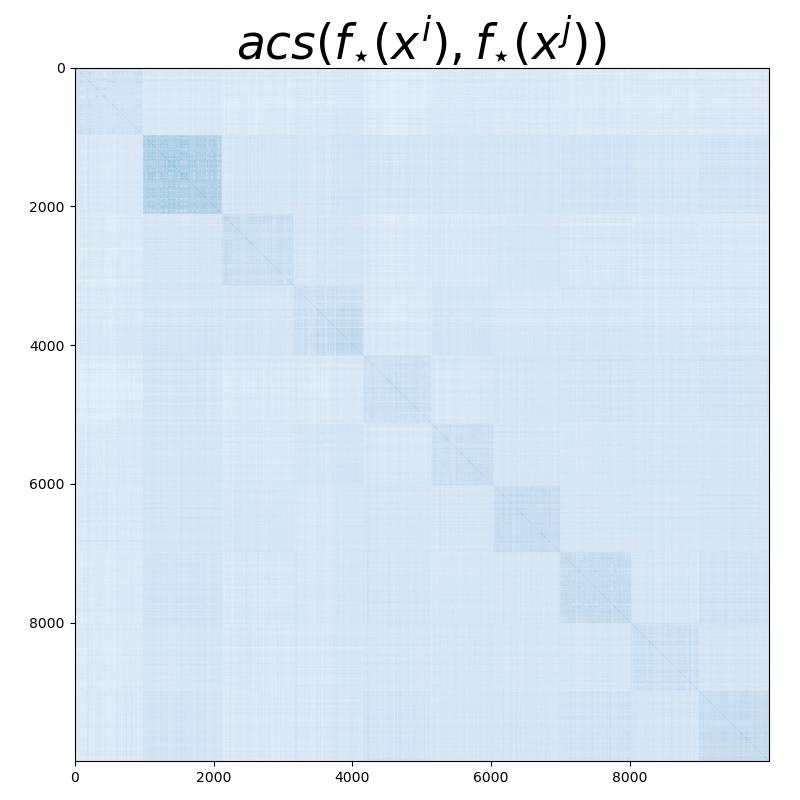}
    \includegraphics[width=\textwidth]{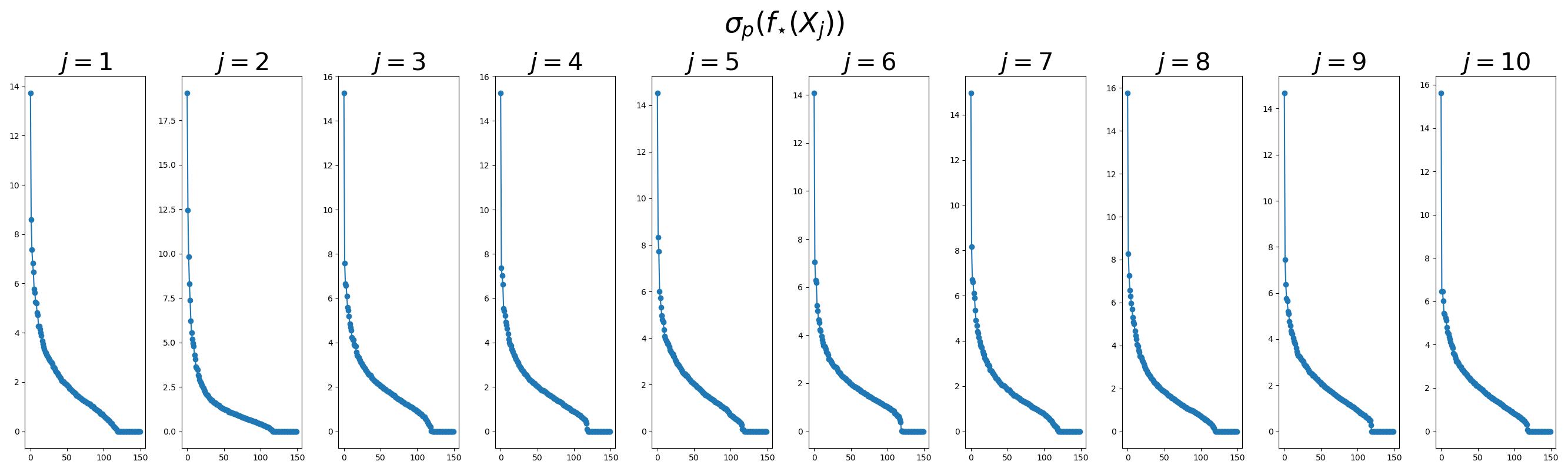}
    \includegraphics[width=\textwidth]{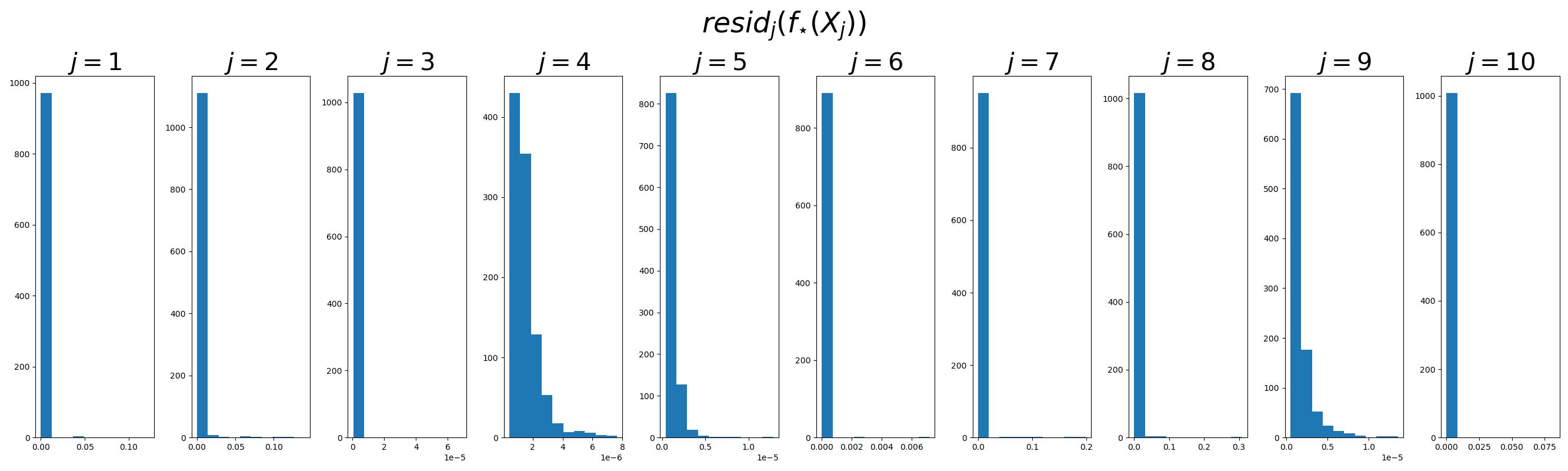}
    \caption{Behavior of our CTRL-SG instance on MNIST.}
    \label{fig:ctrl_sg_mnist}
\end{figure}

As an implementation detail, since the encoder is now a nonlinear neural network, it is intractable to directly optimize over only norm-constrained \(f\), i.e., only the \(f\) such that \(\norm{f(\X_{j})}_{F}^{2} \leq n_{j}\) for all \(j\). As such, we optimize over a \textit{smaller} class of functions, namely those \(f\) such that \(\norm{f(\x)}_{\ell^{2}} = 1\) for all \(\x\). This constraint is \textit{stricter} than the previous normalization constraint, and is much easier to implement; in particular, adding a normalization step just before the output of the encoder network enforces that \(\norm{f(\x)}_{\ell^{2}} = 1\) for all \(\x\). This choice of normalization is used in practice with large-scale networks and is shown to realize the same optima as those realized by the original normalization \citep{LDR}, as we observe in \Cref{fig:ctrl_sg_mnist}.

%% file: conclusion.tex
\section{Conclusion} \label{sec:conclusion}

In this work, we introduced the closed-loop multi-subspace pursuit (CTRL-MSP) framework for learning representations of the multiple linear subspace structure. We explicitly characterized the Stackelberg equilibria of the associated CTRL-MSP game, and provided empirical support for the characterization. Finally, we introduced a generalization, CTRL-SG, for more general representation learning, and characterized the Stackelberg equilibria of the associated game.

There are several directions for future work, both for CTRL-MSP and for CTRL-SG. First, the current analysis of CTRL-MSP holds when the data lie perfectly on linear subspaces; it may be fruitful to consider the conditions under which the addition of noise causes the learned encoder and decoder for CTRL-MSP to diverge. In addition, one may consider modifications to CTRL-MSP which are able to better handle noise. For instance, one might modify the \(\Q\) function of CTRL-MSP to also minimize the coding rate of the autoencoded representations, i.e., \(\Delta R((f \circ g \circ f)(\X) | \bPi)\), in order to avoid the phenomenon where the encoder increases the coding rate by magnifying off-subspace noise (say by adding dimensions to the representation subspaces \(f_{\star}(\S_{j})\)). Analysis of this type of scheme would simultaneously reveal properties of the original CTRL game (discussed in \Cref{sec:pca_gan_games}), which is used in practice to achieve state-of-the-art empirical results on large scale datasets \citep{LDR}. Also, it may be interesting to analyze cases where the data lies on more general non-linear manifolds. We believe that this analysis would require us to consider function classes \(\F\) and \(\G\) containing nonlinear mappings, each of which is a concatenation of simple (possibly linear) layers that iteratively and gradually deform the nonlinear manifolds to orthogonal subspaces (for \(\F\)) or reverse this process (for \(\G\)). Regarding CTRL-SG, it is possible to use the framework for other kinds of representation learning problems in different contexts, and characterize the learned encoder and decoder similarly to this work.

In conclusion, CTRL-MSP exemplifies how classical subspace learning problems can be formulated as special cases of modern representation learning problems. In general, unifying classical and modern perspectives greatly contributes towards better understanding the behavior of modern algorithms on both classical and modern problems.

%% file: acknowledgements_funding_etc.tex
\section{Acknowledgements}

We thank Peter Tong and Xili Dai of UC Berkeley for insightful discussion regarding practical optimization strategies, as well as fair comparisons to other methods.

\paragraph{Funding.} Edgar would like to acknowledge support by the NSF under grants 2046874 and 2031895. Yi would like to acknowledge the support of ONR grants N00014-20-1-2002 and N00014-22-1-2102, and the~joint Simons Foundation-NSF DMS grant \#2031899.

%% file: appendix_simultaneous_games.tex
\section{PCA, GAN, and CTRL as Games} \label{sec:pca_gan_games}

In the main body of the paper, we use the framework of learning an encoder function \(f \in \F\) and decoder function \(g \in \G\) via a two-player \textit{sequential} game between the encoder and decoder. This is different than conventional formulations of other representation learning algorithms, including PCA, nonlinear PCA (autoencoding) \citep{kramer1991nonlinear}, GANs \citep{goodfellow2014generative,arjovsky2017wasserstein}, and CTRL \citep{LDR} as \textit{simultaneous} games between the encoder (or discriminator) and decoder (or generator). For completeness, we briefly introduce aspects of simultaneous game theory (a more detailed introduction is again found in \cite{basar1998dynamic}), and then discuss each of these frameworks in terms of our general representation learning formulation as well as simultaneous game theory.

\paragraph{Simultaneous Game Theory.} In a simultaneous game between the encoder --- playing \(f \in \F\) --- and the decoder --- playing \(g \in \G\) --- both players make their move at the same time with no information about the other player's move. As in the sequential game framework, both players rationally attempt to maximize their utility functions \(u_{\enc} \colon \F \times \G \to \R\) and \(u_{\dec} \colon \F \times \G \to \R\) respectively. The solution concept for a simultaneous game is a so-called \textit{Nash equilibrium}. Formally, \((f_{\star}, g_{\star}) \in \F \times \G\) is a Nash equilibrium if and only if
\begin{equation*}
    f_{\star} \in \argmax_{f \in \F}u_{\enc}(f, g_{\star}), \qquad g_{\star} \in \argmax_{g \in \G}u_{\dec}(f, g_{\star}).
\end{equation*}

\paragraph{PCA and Autoencoding.} PCA finds the best approximating subspace, which we denote \(\S_{\mathrm{PCA}}\), to the data in the following sense. Let \(\mathrm{Gr}(\dz, \R^{\dx})\) be the set of \(\dz\)-dimensional linear subspaces of \(\R^{\dx}\). Then, PCA solves the problem 
\begin{equation*}
    \S_{\mathrm{PCA}} \in \argmin_{S \in \mathrm{Gr}(\dz, \R^{\dx})}\sum_{i = 1}^{n}\norm{\x^{i} - \operatorname{proj}_{S}(\x^{i})}_{\ell^{2}}^{2}.
\end{equation*}
The solution is well-known and exists in closed form in terms of the SVD of the data matrix \(\X\) \citep{Wright-Ma-2022}. To formulate this problem in terms of a game between an encoder and decoder, we learn the orthogonal projection operator \(\operatorname{proj}_{S}(\cdot)\) as a rank-\(\dz\) composition of so-called \textit{semi-orthogonal} linear maps. Let \(p, q \geq 1\) be integers. For a linear map \(h \colon \R^{p} \to \R^{q}\) denote its transpose \(h^{\top} \colon \R^{q} \to \R^{p}\) as the linear map whose matrix representation is the transpose of the matrix representation of \(h\). For a set \(A\) let \(\mathrm{id}_{A} \colon A \to A\) be the identity map on \(A\). Finally, define the semi-orthogonal linear maps \(\O{\R^{p}, \R^{q}}\) by
\begin{equation*}
    \O{\R^{p}, \R^{q}} \doteq \left\{h \in \L(\R^{p}, \R^{q}) \middle|\, \begin{array}{cc}h^{\top} \circ h = \mathrm{id}_{\R^{p}} & \text{if}\ p \leq q \\ h \circ h^{\top} = \mathrm{id}_{\R^{q}} & \text{if}\ q \leq p\end{array}\right\}.
\end{equation*}
We may then define a simultaneous PCA game to learn the projection operator.
\begin{definition}[PCA Game]
    The PCA game is a two-player simultaneous game between:
    \begin{enumerate}
        \item The encoder, choosing functions \(f\) in the function class \(\F \doteq \O{\R^{\dx}, \R^{\dz}}\), and having utility function \(u_{\enc}(f, g) \doteq -\sum_{i = 1}^{n}\norm{\x^{i} - (g \circ f)(\x^{i})}_{\ell^{2}}\).
        \item The decoder, choosing functions \(g\) in the function class \(\G \doteq \O{\R^{\dz}, \R^{\dx}}\), and having utility function \(u_{\dec}(f, g) \doteq -\sum_{i = 1}^{n}\norm{\x^{i} - (g \circ f)(\x^{i})}_{\ell^{2}}\).
    \end{enumerate}
\end{definition}

There are alternate formulations of PCA as a game \citep{gemp2020eigengame}.

Autoencoder games \citep{kramer1991nonlinear} may be formulated similar to PCA games, perhaps with the same utility functions, but the function classes \(\F\) and \(\G\) are less constrained. In particular, they may include functions modelled by neural networks. This makes the Nash equilibrium analysis much more difficult.

\paragraph{GAN.} In the GAN framework \citep{goodfellow2014generative,arjovsky2017wasserstein}, \(f\) is interpreted as a discriminator and \(g\) is interpreted as a generator. Unlike the cooperative games of PCA and autoencoding, GANs train the generator and discriminator functions adversarially, in a zero-sum fashion (meaning \(u_{\enc} = -u_{\dec}\)). The generator attempts to fool the discriminator to treat generated data similarly to real data (by mapping to similar representations), while the discriminator seeks to discriminate between real data and generated data. Let \(\Z \in \R^{\dz \times n}\) be a random matrix whose entries are i.i.d.~(probably Gaussian) noise. Also, let \(\mathcal{D} \colon \R^{\dx \times n} \times \R^{\dx \times n} \to \R\) be any finite-sample estimate of the distance between the distributions which generate the columns of its first argument and second argument (for the purpose of concreteness, one may take \(f\) to be an estimator of the Jensen-Shannon divergence or Wasserstein distance). Then the GAN training game may be posited as follows.
\begin{definition}[GAN Game]
    The GAN game is a two-player simultaneous game between:
    \begin{enumerate}
        \item The discriminator, choosing functions \(f\) in the function class \(\F\), and having utility function \(u_{\enc}(f, g) \doteq \mathcal{D}(f(\X), (g \circ f)(\Z))\).
        \item The generator, choosing functions \(g\) in the function class \(\G\), and having utility function \(u_{\dec}(f, g) \doteq -\mathcal{D}(f(\X), (g \circ f)(\Z))\).
    \end{enumerate}
\end{definition}

Despite having conceptually simple foundations, GANs have technical problems. The above two-player game may not have a Nash equilibrium \citep{ozdaglar2020nash}, and even when \(\F\) and \(\G\) are very simple, e.g., linear and quadratic functions, and one assumes that equilibria exists, GANs may not converge to the equilibria \citep{tse2017understanding}. A further, very important, complexity is that the distances most commonly used for GANs, such as Wasserstein distance and Jensen-Shannon divergence \citep{arjovsky2017wasserstein}, are defined variationally and do not have a closed form for any non-trivial distributions (even mixtures of Gaussians). Thus, one has to approximate the distance via another distance function which is more tractable \citep{arjovsky2017wasserstein}; this estimate becomes worse as the data dimension increases \citep{Wright-Ma-2022}, and is thus less suitable for high-dimensional data.

\paragraph{CTRL.} In CTRL \citep{LDR}, the function \(f\) is both an encoder and a discriminator, while \(g\) is both a decoder and generator. Closed-loop training compares the representations \(f(\X)\) and the representations of the autoencoded data \((f \circ g \circ f)(\X)\). More specifically, using the rate reduction difference measure \(\Delta R(\cdot, \cdot)\) discussed in \Cref{sub:rate_reduction}, the difference measure between the distributions generating the representations of the data and of the autoencoded data is estimated by the class-wise difference measure \(\sum_{j = 1}^{k}\Delta R(f(\X_{j}), (f \circ g \circ f)(\X_{j}))\). Similar to the GAN game, the encoder attempts to maximize this quantity, and the decoder attempts to minimize it. Unlike the GAN game, CTRL also seeks to achieve structured representations; in order to do this, the encoder jointly attempts to maximize the representation quality of both the data and the autoencoded data, which is represented by an additive term \(\Delta R(f(\X) \mid \bPi) + \Delta R((f \circ g \circ f)(\X) \mid \bPi)\). Formalizing this gives the CTRL game.

\begin{definition}[CTRL Game]
    The CTRL game is a two-player simultaneous game between:
    \begin{enumerate}
        \item The encoder, choosing functions \(f\) in the function class \(\F\), and having utility function 
        \begin{align*}
            u_{\enc}(f, g) 
            &\doteq \Delta R(f(\X) | \bPi) + \Delta R((f \circ g \circ f)(\X) | \bPi) \\
            &\quad + \sum_{j = 1}^{k}\Delta R(f(\X_{j}), (f \circ g \circ f)(\X_{j})).
        \end{align*}
        \item The decoder, choosing functions \(f\) in the function class \(\F\), and having utility function \(u_{\dec}(f, g) \doteq -u_{\enc}(f, g)\).
    \end{enumerate}
\end{definition}

Theoretical analysis of this particular simultaneous game is still an open problem, though this formulation \citep{LDR} and other closely related formulations \citep{tong2022incremental} have achieved good empirical results. 

%% file: appendix_ctrl_msp_proofs.tex
\section{Proof of \Cref{thm:ctrl_msp}} \label{sec:proofs}

The proof of \Cref{thm:ctrl_msp} is split up into several steps, which we capture into a few lemmas.
\begin{enumerate}
    \item We first show that the CTRL-MSP game is a particular instance of the CTRL-SG game. We show that if the CTRL-MSP assumptions hold then the CTRL-SG assumptions hold.
    \item We then characterize the optima of the quality and consistency functions of the CTRL-MSP game.
\end{enumerate}

\begin{lemma}\label{lem:ctrl_msp_instance_ctrl_sg}
    The CTRL-MSP game is an instance of the CTRL-SG game. Moreover, if the CTRL-MSP assumptions hold then the CTRL-SG assumptions hold.
\end{lemma}
\begin{proof}
    The CTRL-MSP game is an instance of the CTRL-SG game, with the following correspondences:
    \begin{itemize}
        \item \(\F \doteq \{f \in \L(\R^{\dx}, \R^{\dz}) \mid \norm{f(\X_{j})}_{F}^{2} \leq n_{j}\ \forall j \in \{1, \dots, k\}\}\).
        \item \(\G \doteq \L(\R^{\dz}, \R^{\dx})\).
        \item \(\Q \colon f \mapsto \Delta R(f(\X) \mid \bPi)\).
        \item \(\C \colon (f, g) \mapsto -\sum_{j = 1}^{k}\Delta R(f(\X_{j}), (f \circ g \circ f)(\X_{j}))\).
    \end{itemize}

    We now show that the CTRL-SG assumptions hold.
    First, we claim that \(\argmax_{f \in \F}\Q(f) = \argmax_{f \in \F}\Delta R(f(\X) \mid \bPi)\) is nonempty. 
    Let \(\S \doteq \Span{\bigcup_{j = 1}^{k}\S_{j}}\). While \(\Q\) is continuous in \(f\), compactness (required for the usual argument showing the existence of maxima) is not immediate from the definition: 
    linear maps in \(\F\) are controlled only on \(\S\) and may have arbitrarily large operator norms on \(\S^{\perp}\), thus making \(\F\) an unbounded set and not compact. 
    To remedy this, consider the related problem of optimization over the set
    \begin{equation*}
        \F' \doteq \F \cap \left\{f \in \L(\R^{\dx}, \R^{\dz}) \middle|\, f(\S^{\perp}) = \{\bm{0}\}\right\}.
    \end{equation*}
    Now we have 
    \begin{equation*}
        \max_{f \in \F}\Q(f) = \max_{f \in \F'}\Q(f) \qquad \text{and} \qquad \argmax_{f \in \F'}\Q(f) \subseteq \argmax_{f \in \F}\Q(f).
    \end{equation*}
    Thus, it suffices to show that \(\argmax_{f \in \F'}\Q(f)\) is nonempty. Clearly \(\F'\) is compact. The extreme value theorem holds for optimizing \(\Q\) over \(\F'\), and the claim is proved.
    
    Now we claim that \(\argmax_{g \in \G}\C(f, g) = \argmin_{g \in \G}\sum_{j = 1}^{k}\Delta R(f(\X_{j}), (f \circ g \circ f)(\X_{j}))\) exists for every \(f\). Indeed, \Cref{lem:min_delta_R_distance} shows that \(\Delta R(\Z_{1}, \Z_{2}) \geq 0\) with equality if and only if \(\Z_{1}\Z_{1}^{\top} = \Z_{2}\Z_{2}^{\top}\), so \(\C(f, g) \leq 0\) for all \((f, g) \in \F \times \G\). If \(f^{+}\) is taken to be the Moore-Penrose pseudoinverse of \(f\), then \(f \circ f^{+} \circ f = f\) and \(\C(f, f^{+}) = 0\), which implies \(f^{+} \in \argmax_{g \in \G}\C(f, g)\). This implies that the set of maximizers is nonempty, proving the claim.
    
    Finally, we claim that the function \(f \mapsto \max_{g \in \G}\C(f, g)\) is constant. Indeed, by the choice of \(g = f^{+}\) which is well-defined for all linear maps \(f\), this function is constantly zero, as desired.
\end{proof}

Now, we can separate the utility and analyze its parts; this is the main theoretical benefit of casting the problem in the CTRL-SG paradigm.

\begin{lemma}\label{lem:max_delta_R}
    Suppose the CTRL-MSP assumptions hold, and let \(\F, \G\) be defined as in the CTRL-MSP game. Then any \(f_{\star} \in \argmax_{f \in \F}\Delta R(f(\X) \mid \bPi)\) enjoys properties 2a and 2b from \Cref{thm:ctrl_msp}.
\end{lemma}

\begin{proof}
    First, since \(f_{\star} \in \F\) is linear, \(f_{\star}(\S_{j})\) is a linear subspace; further, \(\dim{f_{\star}(\S_{j})} \leq \dSj\).
    We now claim that the subspaces \(\{f_{\star}(\S_{j})\}_{j = 1}^{k}\) are orthogonal. Since \(f_{\star}(\S_{j}) = \Col{f_{\star}(\X_{j})}\), this is equivalent to the columns of \(f_{\star}(\X_{j})\) being orthogonal to the columns of \(f_{\star}(\X_{\ell})\) for all \(\ell \neq j\), i.e., \(f_{\star}(\X_{j})^{\top}f_{\star}(\X_{\ell}) = \bm{0}\). 
    
    The essential tool we use to show that the \(f_{\star}(\X_{j})\) have orthogonal columns is Lemma A.5 of \cite{OriginalMCR2}, which states that, for matrices \(\Z_{j} \in \R^{\dz \times n_{j}}\) which are collected in a matrix \(\Z \in \R^{\dz \times n}\), 
    \begin{align}\label{drb}
        \Delta R(\Z \mid \bPi) 
        \leq \frac{1}{2n}\sum_{j = 1}^{k}\sum_{p = 1}^{\dSj}\log{\frac{\left(1 + \frac{\dz}{n\eps^{2}}\sigma_{p}(\Z_{j})^{2}\right)^{n}}{\left(1 + \frac{\dz}{n_{j}\eps^{2}}\sigma_{p}(\Z_{j})^{2}\right)^{n_{j}}}}
    \end{align}
    with equality if and only if \(\Z_{j}^{\top}\Z_{\ell}=0\) for all \(1 \leq j < \ell \leq k\).
    
    Suppose for the sake of contradiction that \(f_{\star}(\X_{j})^{\top}f_{\star}(\X_{\ell}) \neq \bm{0}\) for some \(1 \leq j < \ell \leq k\). Since \(\dz \geq \sum_{j = 1}^{k}\dSj\) and the subspaces \(\S_{j}\), $j=1,\ldots,k$ have linearly independent bases, one can construct via the SVD another linear map \(\widetilde{f} \in \L(\R^{\dx}, \R^{\dz})\) such that
    \begin{itemize}
        \item \(\sigma_{p}(f_{\star}(\X_{j})) = \sigma_{p}(\widetilde{f}(\X_{j}))\), for \(1 \leq p \leq \dSj\) and \(1 \leq j \leq k\).
        \item \(\widetilde{f}(\X_{j})^{\top}\widetilde{f}(\X_{\ell}) = 0\) for all for \(1 \leq j < \ell \leq k\).
    \end{itemize}
    Then for each \(j\) we have 
    \begin{equation*}
        \norm{\widetilde{f}(\X_{j})}_{F}^{2} = \sum_{p = 1}^{\dSj}\sigma_{p}^2(\widetilde{f}(\X_{j})) = \sum_{p = 1}^{\dSj}\sigma_{p}^2(f_{\star}(\X_{j})) = \norm{f_{\star}(\X_{j})}_{F}^{2} \leq n_{j}
    \end{equation*}
    so we have \(\widetilde{f} \in \F\). Further, since equality holds in the inequality \eqref{drb} of Lemma A.5 of \cite{OriginalMCR2} for \(\widetilde{f}\) but not for \(f_{\star}\), we have 
    \begin{align*}
        \Delta R(\widetilde{f}(\X) \mid \bPi)
        &= \frac{1}{2n}\sum_{j = 1}^{k}\sum_{p = 1}^{\dSj}\log{\frac{\left(1 + \frac{\dz}{n\eps^{2}}\sigma_{p}(\widetilde{f}(\X_{j}))^{2}\right)^{n}}{\left(1 + \frac{\dz}{n_{j}\eps^{2}}\sigma_{p}(\widetilde{f}(\X_{j}))^{2}\right)^{n_{j}}}} \\
        &= \frac{1}{2n}\sum_{j = 1}^{k}\sum_{p = 1}^{\dSj}\log{\frac{\left(1 + \frac{\dz}{n\eps^{2}}\sigma_{p}(f_{\star}(\X_{j}))^{2}\right)^{n}}{\left(1 + \frac{\dz}{n_{j}\eps^{2}}\sigma_{p}(f_{\star}(\X_{j}))^{2}\right)^{n_{j}}}} \\
        &> \Delta R(f_{\star}(\X) \mid \bPi).
    \end{align*}
    Thus \(f_{\star}\) does not maximize \(f \mapsto \Delta R(f(\X) \mid \bPi)\) over \(f \in \F\), a contradiction. Thus, we must have that \(f_{\star}(\X_{j})^{\top}f_{\star}(\X_{\ell}) = \bm{0}\) for all \(1 \leq j < \ell \leq k\), and so the \(\{f_{\star}(\S_{j})\}_{j = 1}^{k}\) are orthogonal subspaces.
    
    Now, we claim that either \(\sigma_{1}(f_{\star}(\X_{j})) = \cdots = \sigma_{\dSj}(f_{\star}(\X_{j})) = \frac{n_{j}}{d_{j}}\), or \(\sigma_{1}(f_{\star}(\X_{j})) = \cdots = \sigma_{\dSj - 1}(f_{\star}(\X_{j})) \in (\frac{n_{j}}{d_{j}}, \frac{n_{j}}{d_{j} - 1})\) and \(\sigma_{\dSj}(f_{\star}(\X_{j})) > 0\). 
    To show this, the general approach is to isolate the effect of \(f_{\star}\) on each \(\X_{j}\). In particular, fix \(t \in \{1, \dots, k\}\). We claim that
    \begin{equation}\label{fmax}
        f_{\star} \in \argmax_{f \in \F}\sum_{p = 1}^{\dSt}\log{\frac{\left(1 + \frac{\dz}{n\eps^{2}}\sigma_{p}(f(\X_{t}))^{2}\right)^{n}}{\left(1 + \frac{\dz}{n_{t}\eps^{2}}\sigma_{p}(f(\X_{t}))^{2}\right)^{n_{t}}}}.
    \end{equation}
    Indeed, suppose that this does not hold,  
    and there exists \(\widehat{f} \in \F\) such that
    \begin{equation*}
        \sum_{p = 1}^{\dSt}\log{\frac{\left(1 + \frac{\dz}{n\eps^{2}}\sigma_{p}(f_{\star}(\X_{t}))^{2}\right)^{n}}{\left(1 + \frac{\dz}{n_{t}\eps^{2}}\sigma_{p}(f_{\star}(\X_{t}))^{2}\right)^{n_{t}}}} < \sum_{p = 1}^{\dSt}\log{\frac{\left(1 + \frac{\dz}{n\eps^{2}}\sigma_{p}(\widehat{f}(\X_{t}))^{2}\right)^{n}}{\left(1 + \frac{\dz}{n_{t}\eps^{2}}\sigma_{p}(\widehat{f}(\X_{t}))^{2}\right)^{n_{t}}}}.
    \end{equation*}
    Then, again since \(\dz \geq \sum_{j = 1}^{k}\dSj\) and the subspaces \(\S_{j}\) have linearly independent bases, one can construct another linear map \(\widetilde{f} \in \L(\R^{\dx}, \R^{\dz})\) such that
    \begin{itemize}
        \item \(\sigma_{p}(\widetilde{f}(\X_{t})) = \sigma_{p}(\widehat{f}(\X_{t}))\), for \(1 \leq p \leq \dSt\).
        \item \(\sigma_{p}(\widetilde{f}(\X_{j})) = \sigma_{p}(f_{\star}(\X_{j}))\), for \(1 \leq p \leq \dSj\), \(1 \leq j \leq k\) with \(j \neq t\).
        \item \(\widetilde{f}(\X_{j})^{\top}\widetilde{f}(\X_{\ell}) = \bm{0}\) for \(1 \leq j < \ell \leq k\).
    \end{itemize}
    For the same reason as in the previous claim, \(\widetilde{f} \in \F\). 
    Moreover, \(\Delta R(\widetilde{f}(\X) \mid \bPi) > \Delta R(f_{\star}(\X) \mid \bPi)\), because
    \begin{align*}
        &2n\cdot \Delta R(\widetilde{f}(\X) \mid \bPi) \\
        &= \sum_{p = 1}^{\dSt}\log{\frac{\left(1 + \frac{\dz}{n\eps^{2}}\sigma_{p}(\widetilde{f}(\X_{t}))^{2}\right)^{n}}{\left(1 + \frac{\dz}{n_{t}\eps^{2}}\sigma_{p}(\widetilde{f}(\X_{t}))^{2}\right)^{n_{t}}}} + \sum_{\substack{j = 1 \\ j \neq t}}^{k}\sum_{p = 1}^{\dSj}\log{\frac{\left(1 + \frac{\dz}{n\eps^{2}}\sigma_{p}(\widetilde{f}(\X_{j}))^{2}\right)^{n}}{\left(1 + \frac{\dz}{n_{j}\eps^{2}}\sigma_{p}(\widetilde{f}(\X_{j}))^{2}\right)^{n_{j}}}} \\
        &= \sum_{p = 1}^{\dSt}\log{\frac{\left(1 + \frac{\dz}{n\eps^{2}}\sigma_{p}(\widehat{f}(\X_{t}))^{2}\right)^{n}}{\left(1 + \frac{\dz}{n_{t}\eps^{2}}\sigma_{p}(\widehat{f}(\X_{t}))^{2}\right)^{n_{t}}}} + \sum_{\substack{j = 1 \\ j \neq t}}^{k}\sum_{p = 1}^{\dSj}\log{\frac{\left(1 + \frac{\dz}{n\eps^{2}}\sigma_{p}(f_{\star}(\X_{j}))^{2}\right)^{n}}{\left(1 + \frac{\dz}{n_{j}\eps^{2}}\sigma_{p}(f_{\star}(\X_{j}))^{2}\right)^{n_{j}}}}.
    \end{align*}
    This is strictly lower bounded by
    \begin{align*}
        &\sum_{p = 1}^{\dSt}\log{\frac{\left(1 + \frac{\dz}{n\eps^{2}}\sigma_{p}(f_{\star}(\X_{t}))^{2}\right)^{n}}{\left(1 + \frac{\dz}{n_{t}\eps^{2}}\sigma_{p}(f_{\star}(\X_{t}))^{2}\right)^{n_{t}}}} + \sum_{\substack{j = 1 \\ j \neq t}}^{k}\sum_{p = 1}^{\dSj}\log{\frac{\left(1 + \frac{\dz}{n\eps^{2}}\sigma_{p}(f_{\star}(\X_{j}))^{2}\right)^{n}}{\left(1 + \frac{\dz}{n_{j}\eps^{2}}\sigma_{p}(f_{\star}(\X_{j}))^{2}\right)^{n_{j}}}}\\
        &= 2n\cdot \Delta R(f_{\star}(\X) \mid \bPi).
    \end{align*}
    Thus \(f_{\star}\) is not a maximizer of \(f \mapsto \Delta R(f(\X) \mid \bPi)\), which is a contradiction. Hence, \eqref{fmax} follows.
    
    To finish, we may now solve the problem in terms of the singular values of \(f(\X_{t})\). Indeed, from the above optimization problem and the definition of \(\F\), the singular values \(\sigma_{p}(f_{\star}(\X_{t}))\)  are the solutions of the scalar optimization problem
    \begin{align*}
        \max_{\sigma_{1}, \dots, \sigma_{\dSt} \in \R} \quad & \sum_{p = 1}^{\dSt} \log{\frac{\left(1 + \frac{\dz}{n\eps^{2}}\sigma_{p}^{2}\right)^{n}}{\left(1 + \frac{\dz}{n_{t}\eps^{2}}\sigma_{p}^{2}\right)^{n_{t}}}} \\
        \text{s.t.} \quad & \sigma_{1} \geq \cdots \geq \sigma_{\dSt} \geq 0, \qquad
         \sum_{p = 1}^{\dSt}\sigma_{p}^{2} = n_{j}.
    \end{align*}
    Given the assumption that \(\eps\) is small enough, Lemma A.7 of \cite{OriginalMCR2} says that the solutions to this optimization problem either fulfill \(\sigma_{1} = \cdots = \sigma_{\dSt} = \frac{n_{j}}{\dSt}\) or \(\sigma_{1} = \cdots = \sigma_{\dSt - 1} \in (\frac{n_{t}}{\dSt}, \frac{n_{t}}{\dSt - 1})\) and \(\sigma_{\dSt} > 0\) as desired, where if \(\dSt = 1\) then \(\frac{n_{t}}{\dSt - 1}\) is interpreted as \(+\infty\). 
    This also confirms that \(\dim{f_{\star}(\S_{t})} = \rank{f_{\star}(\X_{t})} = \dSt\).
\end{proof}

We now characterize the maximizers of the consistency function.

\begin{lemma}\label{lem:min_delta_R}
    Suppose the CTRL-MSP assumptions hold, and let \(\F, \G\) be defined as in the CTRL-MSP game. Let \(f_{\star} \in \F\). Then for any 
    \(g_{\star} \in \argmin_{g \in \G}\sum_{j = 1}^{k}\Delta R(f_{\star}(\X_{j}), (f_{\star} \circ g_{\star} \circ f_{\star})(\X_{j}))\), we have
    property 2c from \Cref{thm:ctrl_msp}.
\end{lemma}

\begin{proof}
    From \Cref{lem:min_delta_R_distance}, we have \(\Delta R(\Z_{1}, \Z_{2}) \geq 0\) with equality implying that \(\Col{\Z_{1}} = \Col{\Z_{2}}\). Picking \(g \in \G\) to be the pseudoinverse \(f_{\star}^{+}\) of \(f_{\star}\) shows that \(f_{\star} = f_{\star} \circ f_{\star}^{+} \circ f_{\star} = f_{\star} \circ g_{\star} \circ f_{\star}\), implying that \(\sum_{j = 1}^{k}\Delta R(f_{\star}(\X_{j}), (f_{\star} \circ g \circ f_{\star})(\X_{j})) = 0\). Thus \(\argmin_{g \in \G}\sum_{j = 1}^{k}\Delta R(f_{\star}(\X_{j}), (f_{\star} \circ g \circ f_{\star})(\X_{j}))\) is exactly the set of \(g_{\star} \in \G\) such that \(\sum_{j = 1}^{k}\Delta R(f_{\star}(\X_{j}), (f_{\star} \circ g_{\star} \circ f_{\star})(\X_{j})) = 0\). Indeed, this means that, for each \(j\), we have \(\Col{f_{\star}(\X_{j})} = \Col{(f_{\star} \circ g_{\star} \circ f_{\star})(\X_{j})}\). The former is exactly \(f_{\star}(\S_{j})\), and the latter term is \((f_{\star} \circ g_{\star} \circ f_{\star})(\S_{j})\), so equality holds and the lemma is proved.
\end{proof}

With \Cref{lem:ctrl_msp_instance_ctrl_sg,lem:max_delta_R,lem:min_delta_R}, we are now able to cleanly prove \Cref{thm:ctrl_msp}.

\begin{proof}[Proof of \Cref{thm:ctrl_msp}]
    By \Cref{lem:ctrl_msp_instance_ctrl_sg} and \Cref{thm:ctrl_sg}, we have that, if the CTRL-MSP assumptions hold, a Stackelberg equilibrium to the CTRL-MSP game exists, and that any Stackelberg equilibrium \((f_{\star}, g_{\star})\) enjoys:
    \begin{align*}
        f_{\star} 
        &\in \argmax_{f \in \F}\Delta R(f(\X) \mid \bPi), \\ 
        g_{\star} 
        &\in \argmin_{g \in \G}\sum_{j = 1}^{k}\Delta R(f_{\star}(\X_{j}), (f_{\star} \circ g \circ f_{\star})(\X_{j})).
    \end{align*}
    By \Cref{lem:max_delta_R,lem:min_delta_R}, the conclusions all follow.
\end{proof}

%% file: appendix_ctrl_ssp.tex
\section{Single-Subspace Pursuit via the CTRL Framework} \label{sec:ctrl_ssp}

\subsection{Theory of CTRL-SSP} \label{sub:ctrl_ssp_theory}

We may specialize the CTRL-SG framework to obtain a method for \textit{single subspace pursuit}, called CTRL-SSP, which has similar properties to PCA. 
Suppose in our notation that \(k = 1\), so that \(\X_{1} = \X\). 
We drop the subscript on the only linear subspace \(\S_{1}\) which supports the data distribution, to obtain the single data subspace \(\S\), which we say has  dimension \(d_{\S}\).

Our formulation of PCA in \Cref{sec:pca_gan_games} as finding the best subspace of dimension \(\dz\) for the data shows that, if \(\dz \geq d_{\S}\) then \(\S_{\mathrm{PCA}} \supseteq \S\). Thus \(g_{\mathrm{PCA}} \circ f_{\mathrm{PCA}}\) is the identity on \(\S\), so the semi-orthogonal linear maps \(f_{\mathrm{PCA}}\) and \(g_{\mathrm{PCA}}\) must be isometries on \(\S\) and \(f_{\mathrm{PCA}}(\S)\) respectively. This isometric property is the essential property of the PCA encoder we choose to achieve in CTRL-SSP; our encoder, more than just being injective on \(\S\), must preserve all distances and thus all the structure of the data distribution on \(\S\).

To present the CTRL-SSP game, recall the definition of \(\O{\cdot, \cdot}\) from \Cref{sec:pca_gan_games}.
\begin{definition}[CTRL-SSP Game]\label{def:ctrl_ssp}
    The CTRL-SSP game is a two-player sequential game between: 
    \begin{enumerate}
        \item The encoder, moving first, choosing functions \(f\) in the function class \(\F \doteq \O{\R^{\dx}, \R^{\dz}}\)
        and having utility function 
        \begin{equation*}
            u_{\enc}(f, g) \doteq R(f(\X)) + \Delta R(f(\X), (f \circ g \circ f)(\X)).
        \end{equation*}
        \item The decoder, moving second, choosing functions \(g\) in the function class \(\G \doteq \O{\R^{\dz}, \R^{\dx}}\), and having utility function 
        \begin{equation}
            u_{\dec}(f, g) \doteq -\Delta R(f(\X), (f \circ g \circ f)(\X)).
        \end{equation}
    \end{enumerate}
\end{definition}

As in the previous contexts of CTRL-MSP and CTRL-SG, we outline some assumptions for the data that are required for our main results to hold.
\begin{assumption}[Assumptions in CTRL-SSP Games]\label{a:ctrl_ssp}
    \phantom{}
	\begin{enumerate}
		\item (Informative data.) \(\Col{\X} = \S\).
		\item (Large enough representation space.) \(d_{\S} \leq \min\{\dx, \dz\}\).
	\end{enumerate}
\end{assumption}

We may now explicitly characterize the Stackelberg equilibria of the CTRL-SSP game.

\begin{theorem}[Stackelberg Equilibria of CTRL-SSP Games]\label{thm:ctrl_ssp}
    If \Cref{a:ctrl_ssp} holds, then the CTRL-SSP game has the following properties:
	\begin{enumerate}
		\item A Stackelberg equilibrium \((f_{\star}, g_{\star})\) exists.
		\item Any Stackelberg equilibrium \((f_{\star}, g_{\star})\) enjoys the following properties:
		\begin{enumerate}
			\item (Injective encoder.) \(f_{\star}(\S)\) is a linear subspace of dimension \(d_{\S}\), and \(f_{\star}\) is an \(\ell^{2}\)-isometry on \(\S\).
			\item (Consistent encoding and decoding.) \(f_{\star}(\S) = (f_{\star} \circ g_{\star} \circ f_{\star})(\S)\).
		\end{enumerate}
	\end{enumerate}
\end{theorem}

The proof of this theorem is split up into several steps, similarly to the CTRL-MSP picture.
\begin{enumerate}
    \item We first show that the CTRL-SSP game is a particular instance of the CTRL-SG game. We show that if the CTRL-SSP assumptions hold then the CTRL-SG assumptions hold.
    \item We then characterize the optima of the expressiveness and compatibility functions of the CTRL-SSP game.
\end{enumerate}

\begin{lemma}\label{lem:ctrl_ssp_instance_ctrl_sg}
    The CTRL-SSP game is an instance of the CTRL-SG game. Moreover, if the CTRL-SSP assumptions hold then the CTRL-SG assumptions hold.
\end{lemma}
\begin{proof}
    One may easily see that the CTRL-SSP game is an instance of the CTRL-SG game, with the following correspondences:
    \begin{itemize}
        \item \(\F \doteq \O{\R^{\dx}, \R^{\dz}}\).
        \item \(\G \doteq \O{\R^{\dz}, \R^{\dx}}\).
        \item \(\Q \colon f \mapsto R(f(\X))\).
        \item \(\C \colon (f, g) \mapsto -\Delta R(f(\X), (f \circ g \circ f)(\X))\).
    \end{itemize}

    We show that the CTRL-SG assumptions hold.
    We first claim that \(\argmax_{f \in \F}\Q(f) = \argmax_{f \in \F}R(f(\X))\) exists. Indeed, \(\F\) is compact and \(\Q\) is continuous in \(f\), so the conclusion follows by the extreme value theorem.
    Further, using the same argument as in the corresponding part of the proof of \cref{lem:ctrl_msp_instance_ctrl_sg}, for every \(f \in \F\), we have that \(\argmax_{g \in \G}\C(f, g) = \argmin_{g \in \G}\Delta R(f(\X), (f \circ g \circ f)(\X))\) is nonempty. 
    We finally claim that the function \(f \mapsto \max_{g \in \G}\C(f, g)\) is constant. Indeed, by the choice of \(g = f^{+}\), which for \(f \in \F = \O{\R^{\dx}, \R^{\dz}}\) is contained in \(\G = \O{\R^{\dz}, \R^{\dx}}\), this function is constantly zero, as desired.
\end{proof}

Again, we can separate the utility and analyze each part of it.

\begin{lemma}\label{lem:max_R_orthogonal}
    Suppose the CTRL-SSP assumptions hold, and let \(\F, \G\) be defined as in the CTRL-SSP game. Then any \(f_{\star} \in \argmax_{f \in \F}R(f(\X))\) 
    enjoys property 2a from \cref{thm:ctrl_ssp}.
\end{lemma}

\begin{proof}
    Since \(f_{\star}\) is a linear map and \(\S\) is a linear subspace, \(f_{\star}(\S)\) is a linear subspace, and furthermore \(\dim{f_{\star}(\S)} \leq d_{\S}\). We now claim that \(f_{\star}\) is an \(\ell^{2}\)-isometry on \(\S\). We show this by calculating an upper bound for \(R(f(\X))\) and show that it is achieved if and only if \(f\) is an \(\ell^{2}\)-isometry on \(\S\).
    
    Indeed, for any \(f \in \F = \O{\R^{\dx}, \R^{\dz}}\), we have that \(f\) has operator norm and Lipschitz constant equal to unity, so \(\norm{f(\x)}_{\ell^{2}} \leq \norm{\x}_{\ell^{2}}\) for any \(\x \in \R^{\dx}\). By the Courant-Fischer min-max theorem for singular values, we have,     for each \(1 \leq p \leq d_{\S}\),
    \begin{align*}
        \sigma_{p}(f(\X))
        &= \sup_{S \in \mathrm{Gr}(p, \R^{\dx})}\inf_{\substack{\bm{u} \in S \\ \norm{\bm{u}}_{\ell^{2}} = 1}}\norm{f(\X)\cdot\bm{u}}_{\ell^{2}} \\
        &\leq \sup_{S \in \mathrm{Gr}(p, \R^{\dx})}\inf_{\substack{\bm{u} \in S \\ \norm{\bm{u}}_{\ell^{2}} = 1}}\norm{\X\cdot\bm{u}}_{\ell^{2}}
        = \sigma_{p}(\X).
    \end{align*}
     Thus
    \begin{align*}
        R(f(\X))
        &= \frac{1}{2}\logdet{\I_{\dz} + \frac{\dz}{n\eps^{2}}f(\X)f(\X)^{\top}} \\
        &= \frac{1}{2}\sum_{p = 1}^{d_{\S}}\log{1 + \frac{\dz}{n\eps^{2}}\sigma_{p}(f(\X))^{2}}
        \leq \frac{1}{2}\sum_{p = 1}^{d_{\S}}\log{1 + \frac{\dz}{n\eps^{2}}\sigma_{p}(\X)^{2}}.
    \end{align*}
    As such, \(f\) is an \(\ell^{2}\) isometry on \(\S\) if and only if \(\sigma_{p}(f(\X)) = \sigma_{p}(\X)\) for all \(1 \leq p \leq d_{\S}\).
    Thus, any \(f_{\star} \in \F\) which fulfills the upper bound for \(R(f(\X))\), i.e., any maximizer for \(R(f(\X))\), is an \(\ell^{2}\) isometry on \(\S\). 
    Therefore, \(\dim{f_{\star}(\S)} = d_{\S}\).
\end{proof}

\begin{lemma}\label{lem:min_delta_R_orthogonal}
    Suppose the CTRL-SSP assumptions hold, and let \(\F, \G\) be defined as in the CTRL-SSP game. Let \(f_{\star} \in \argmax_{f \in \F}R(f(\X))\). Then any \(g_{\star} \in \argmin_{g \in \G}\Delta R(f(\X), (f \circ g \circ f)(\X))\) 
    enjoys property 2b from \cref{thm:ctrl_ssp}.
\end{lemma}

\begin{proof}
The proof is exactly as the proof of \Cref{lem:min_delta_R}, and is thus omitted.
\end{proof}

With \Cref{lem:ctrl_ssp_instance_ctrl_sg,lem:max_R_orthogonal,lem:min_delta_R_orthogonal}, we are now able to prove \Cref{thm:ctrl_ssp}. 

\begin{proof}[Proof of \Cref{thm:ctrl_ssp}]
The proof is exactly as the proof of \Cref{thm:ctrl_msp}, and is thus omitted.
\end{proof}

CTRL-SSP replicates the essential isometry aspect of the PCA solution, but it learns all principal components simultaneously, unlike the common greedy algorithms. Thus, it does not require any model selection beyond a choice of \(\dz\), which can be set to \textit{any} integer greater than \(d_{\S}\) without loss of efficacy.

\subsection{Empirical Evidence for CTRL-SSP}\label{sub:ctrl_ssp_experiments}

We demonstrate empirical convergence of learned \(f\) and \(g\) to Stackelberg equilibria of the CTRL-SSP game which satisfy the conclusions of \Cref{thm:ctrl_ssp} in the case of data lying on a single linear subspace. We empirically verify the above claim: that varying \(\dz\) does not impact efficacy of the algorithm. We then demonstrate CTRL-SSP's robustness to noise.

\subsubsection{Optimization Algorithm}

Like CTRL-MSP and CTRL-SG, CTRL-SSP operates in the framework of sequential games, and thus uses the GDMax algorithm \citep{jin2019local}, modified to save runtime in the same way as in \Cref{sec:experiments}. 

\subsubsection{Ways to Evaluate Results}

In CTRL-SSP, our claims for the Stackelberg equilibria \((f_{\star}, g_{\star})\) are furnished by \Cref{thm:ctrl_ssp} and may be tested empirically.

\begin{enumerate}
    \item We wish to show that the encoder is injective on the data subspace \(\S\), i.e., \(f_{\star}(\S)\) is a linear subspace of dimension \(d_{\S}\), and that \(f_{\star}\) is an isometry on \(\S\). To show this, we plot the singular value distribution of \(f_{\star}(\X)\), and show that each \(\sigma_{p}(f_{\star}(\X))\) is close to \(\sigma_{p}(\X)\). As per the previous proofs, this is enough to demonstrate both aspects of the claim.
    \item We wish to show that the encoder and decoder form a self-consistent autoencoding, i.e., that \(f_{\star}(\S) = (f_{\star} \circ g_{\star} \circ f_{\star})(\S)\). Using the same reasoning as in \Cref{sec:experiments}, we show that the distribution of residuals
    \begin{equation*}
        \texttt{resid}(f_{\star}(\x^{i})) \doteq \norm{f_{\star}(\x^{i}) - \operatorname{proj}_{\Col{(f_{\star} \circ g_{\star} \circ f_{\star})(\X)}}(f_{\star}(\x^{i}))}_{\ell^{2}}
    \end{equation*}
    for all \(\x^{i}\) is concentrated near zero.
\end{enumerate}

\subsubsection{CTRL-SSP with Data on Single Subspace}

To generate data which lies on a subspace, the process is simple. We generate a single random matrix with orthonormal columns \(\U \in \R^{\dx \times d_{\S}}\) using the QR decomposition. We then generate random matrices \(\bm{\Xi} \in \R^{d_{\S} \times n}\) and \(\bm{\Phi} \in \R^{\dx \times n_{j}}\) whose entries are i.i.d.~standard normal random variables. We obtain \(\X\) using the formula \(\X \doteq \U\bm{\Xi} + \sqrt{\frac{\nu}{\dx}}\bm{\Phi}\).

For experiments, we set baseline values of \(n = 500\), \(\dx = 50\), \(\dz = 40\), \(d_{\S} = 10\), \(\eps^{2} = 1\), and \(\nu = 0\).

Regarding optimization strategy, we set the number of decoder steps per encoder to be \(L = 10^{3}\). We set the learning rate of \(f\) to be \(10^{-2}\) and the learning rate of \(g\) to be \(10^{-3}\). We use a Riemannian SGD optimizer \citep{kochurov2020geoopt} to optimize over the sets of semi-orthogonal linear maps (i.e., the Stiefel manifold) for \(f\) and \(g\). The data are randomly partitioned into mini-batches of size \(b = 50\) during optimization. We train for \(2\) epochs.

We observe success in the baseline regime. That is, the representation matrix of the subspace data \(f_{\star}(\X)\) has \(d_{\S}\) nonzero singular values \(\sigma_{p}(f_{\star}(\X))\) which are identical to the subspace data matrix \(\X\)'s singular values \(\sigma_{p}(\X)\). This shows that the learned \(f_{\star}\) is injective on \(\S\) and, even stronger, that \(f_{\star}\) is an isometry on \(\S\). Furthermore, we observe that the residuals \(\texttt{resid}(f_{\star}(\x^{i}))\) are concentrated near \(0\) (\Cref{fig:ctrl_ssp_baseline}).

\begin{figure}
    \centering
    \begin{subfigure}[b]{0.24\textwidth}
        \centering 
        \includegraphics[width=\textwidth]{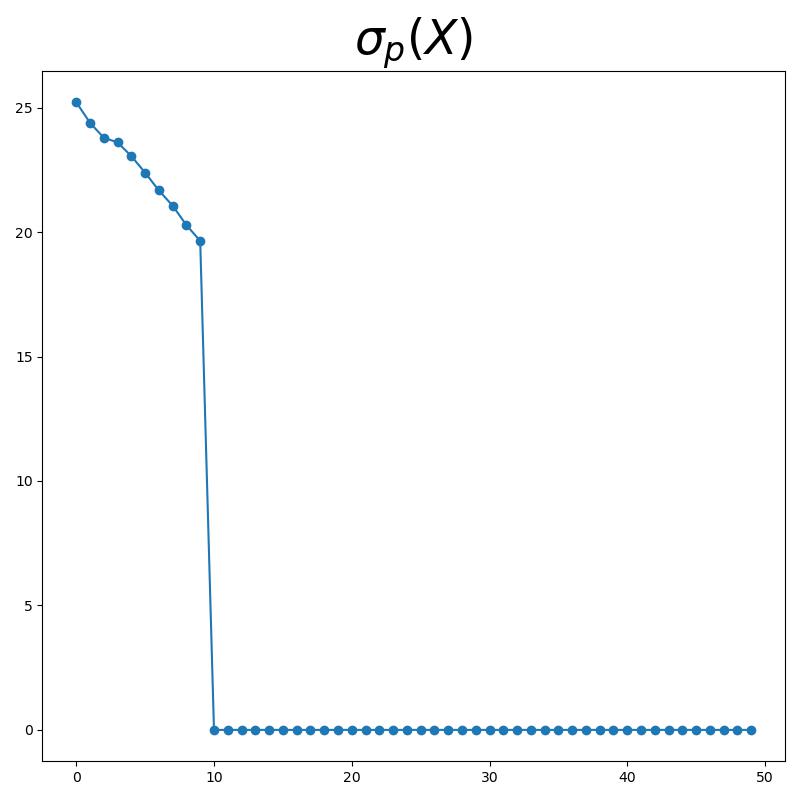}
        \caption{}
    \end{subfigure}
    \begin{subfigure}[b]{0.24\textwidth}
        \centering 
        \includegraphics[width=\textwidth]{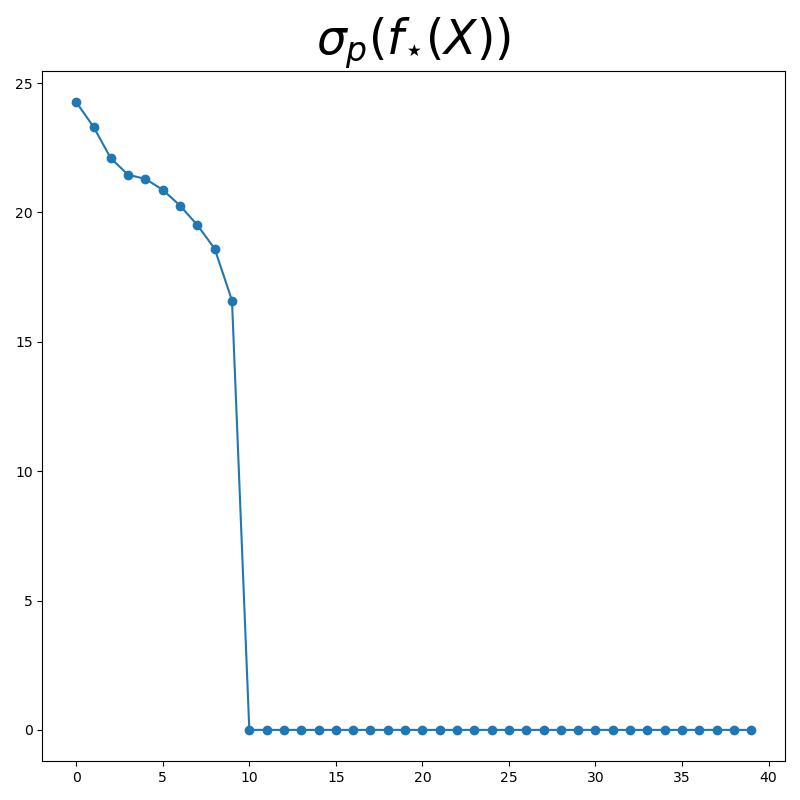}
        \caption{}
    \end{subfigure}
    \begin{subfigure}[b]{0.24\textwidth}
        \centering 
        \includegraphics[width=\textwidth]{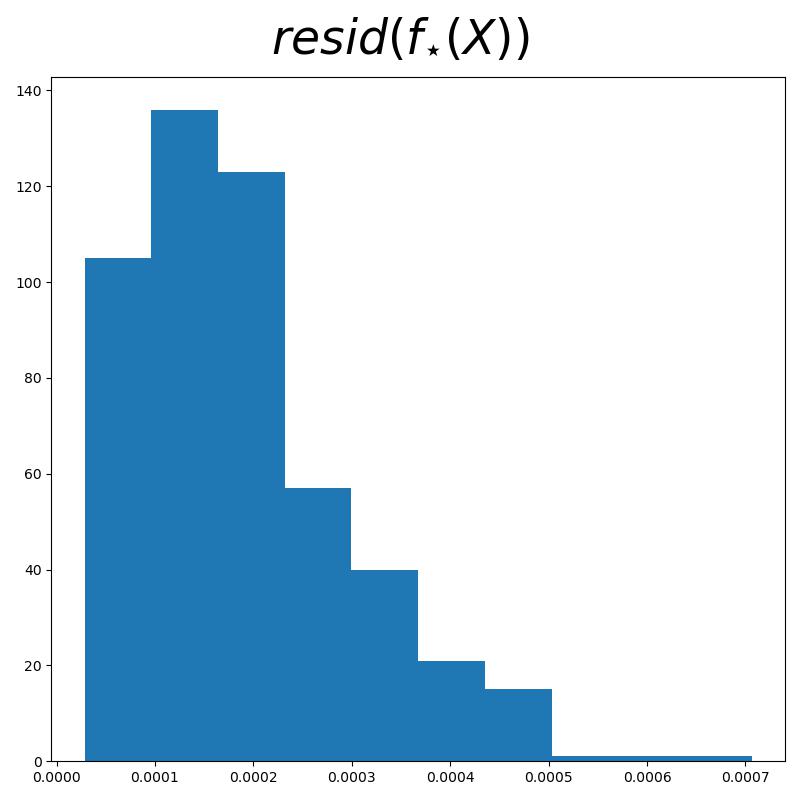}
        \caption{}
    \end{subfigure}
    \caption{Baseline performance of CTRL-SSP. (a) Spectra of the data, i.e., distribution of \(\sigma_{p}(\X)\). (b) Spectra of the representations, i.e., distribution of \(\sigma_{p}(f_{\star}(\X))\). (c) Histogram of the residuals \(\texttt{resid}(f_{\star}(\x^{i}))\) for all \(\x^{i}\) in \(\X\). Note that \(d_{\S} = 10\) singular values of \(f_{\star}(\X)\) are nonzero and large, and the projection residuals are close to \(0\).}
    \label{fig:ctrl_ssp_baseline}
\end{figure}

We now prove our previous heuristic claim: that CTRL-SSP works with any latent dimension \(\dz\) with \(\dz \geq d_{\S}\). Specifically, with \(\dx = 40\) and \(d_{\S} = 10\), we compare the cases of \(\dz = 10\), \(\dz = 20\), \(\dz = 40\), and \(\dz = 50\). In all cases, we observe success of CTRL-SSP to learn equilibria which have properties that correspond to our theoretical guarantees, with slightly increasing efficiency as \(\dz\) increases (\Cref{fig:ctrl_ssp_varying_dz}).

\begin{figure}
    \centering
    \includegraphics[width=0.24\textwidth]{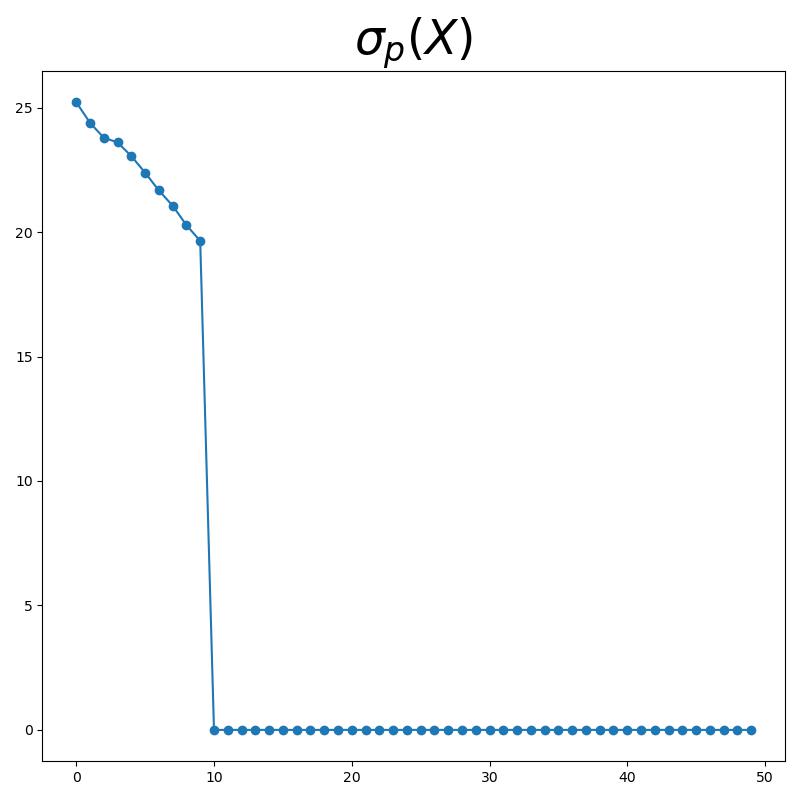}
    
    \includegraphics[width=0.24\textwidth]{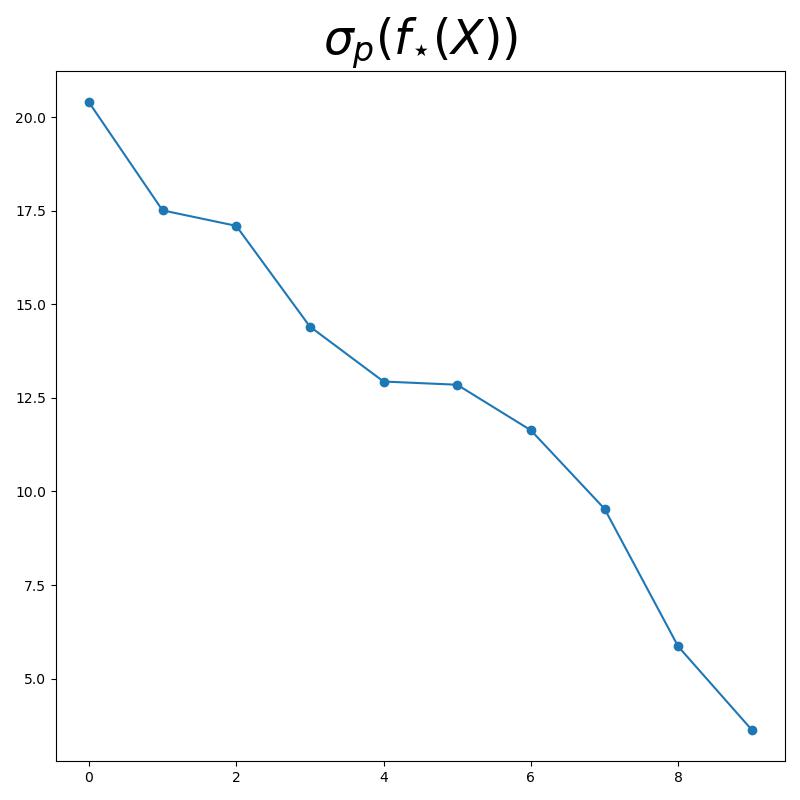}
    \includegraphics[width=0.24\textwidth]{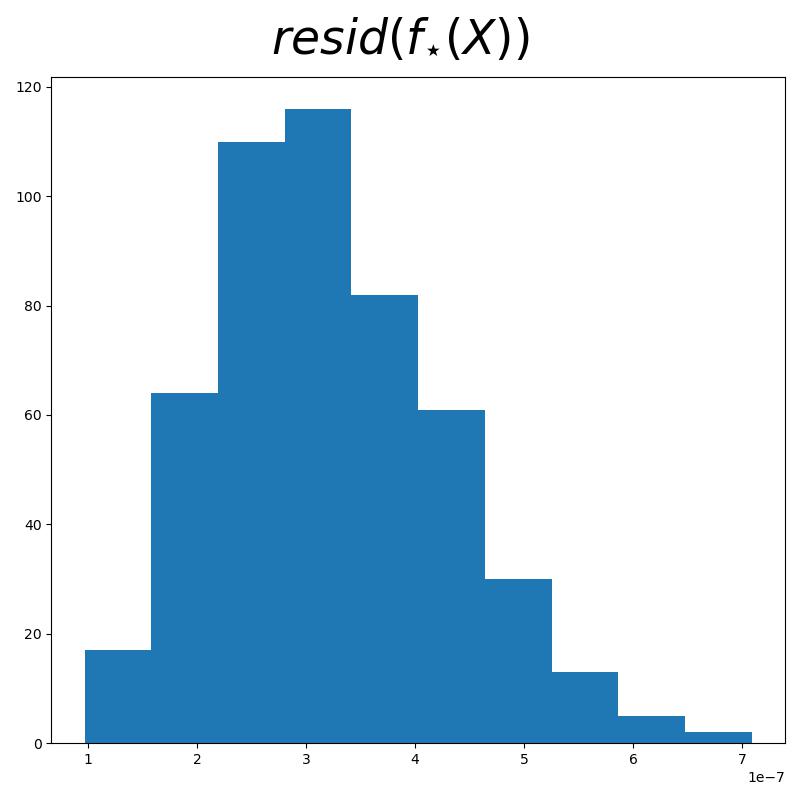}
    
    \includegraphics[width=0.24\textwidth]{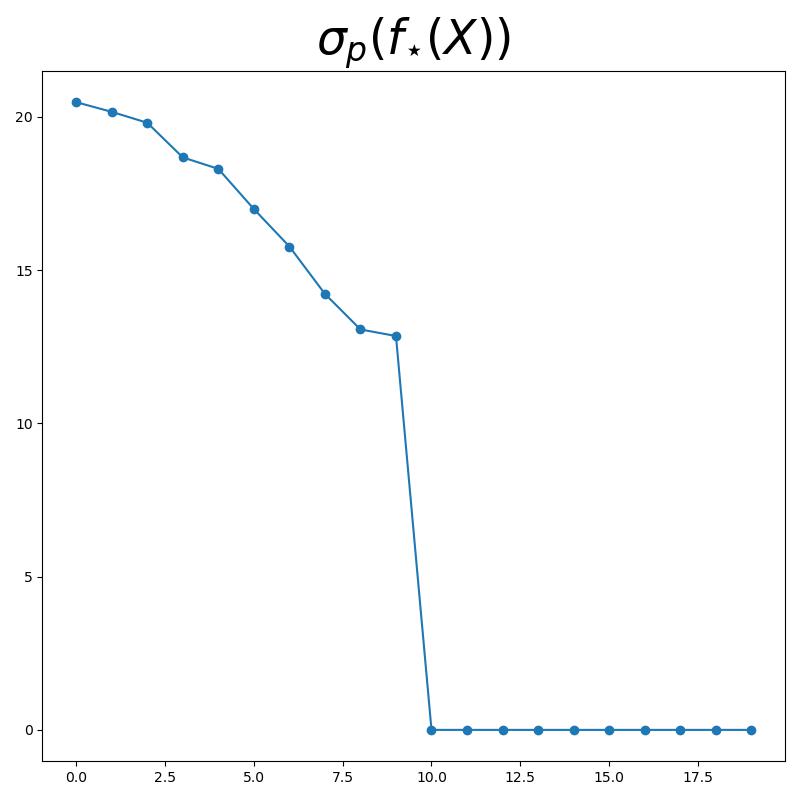}
    \includegraphics[width=0.24\textwidth]{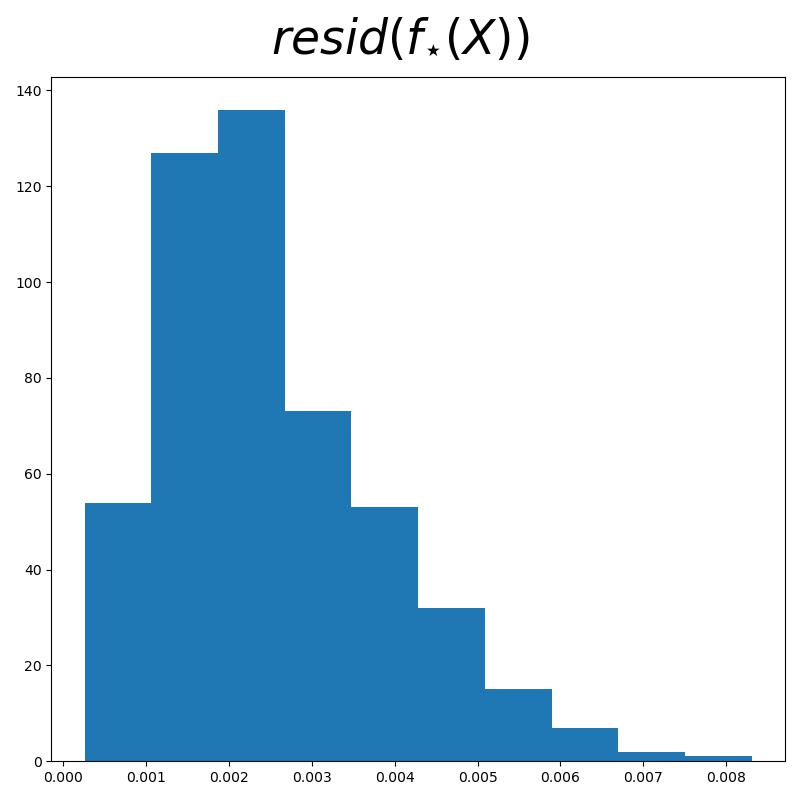}
    
    \includegraphics[width=0.24\textwidth]{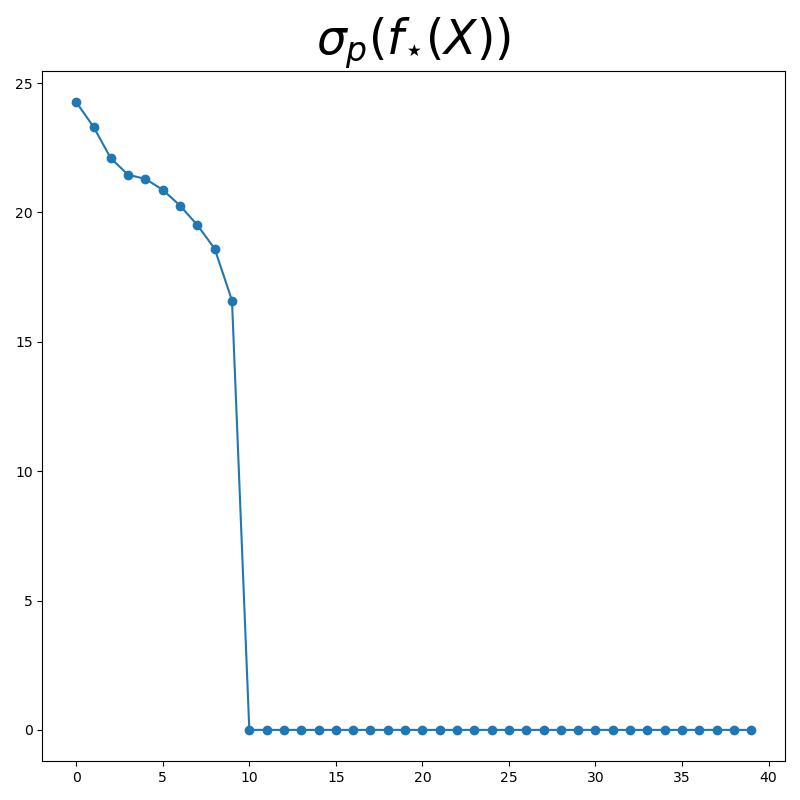}
    \includegraphics[width=0.24\textwidth]{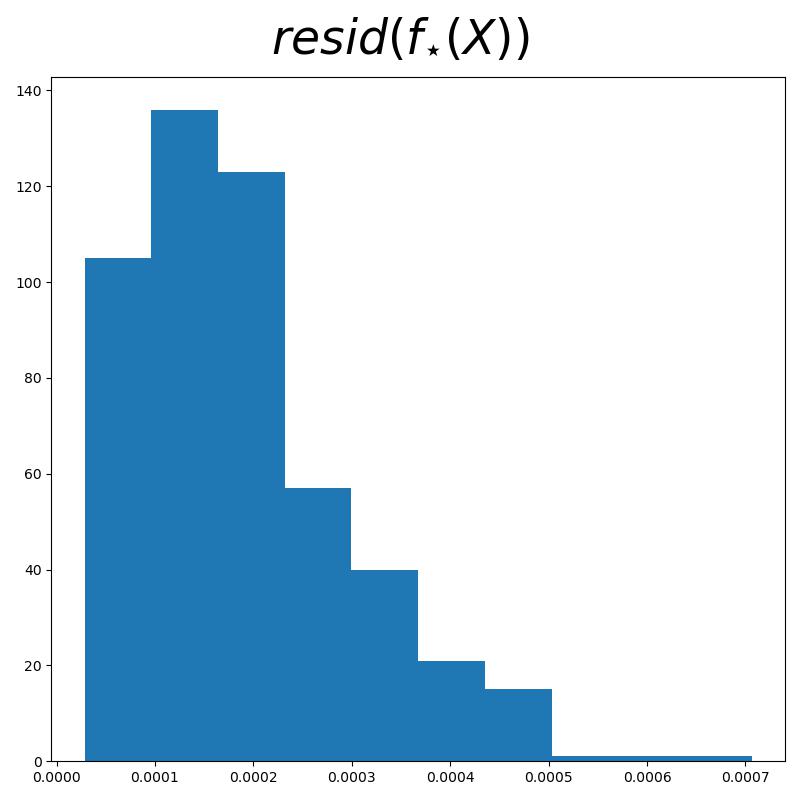}
    
    \includegraphics[width=0.24\textwidth]{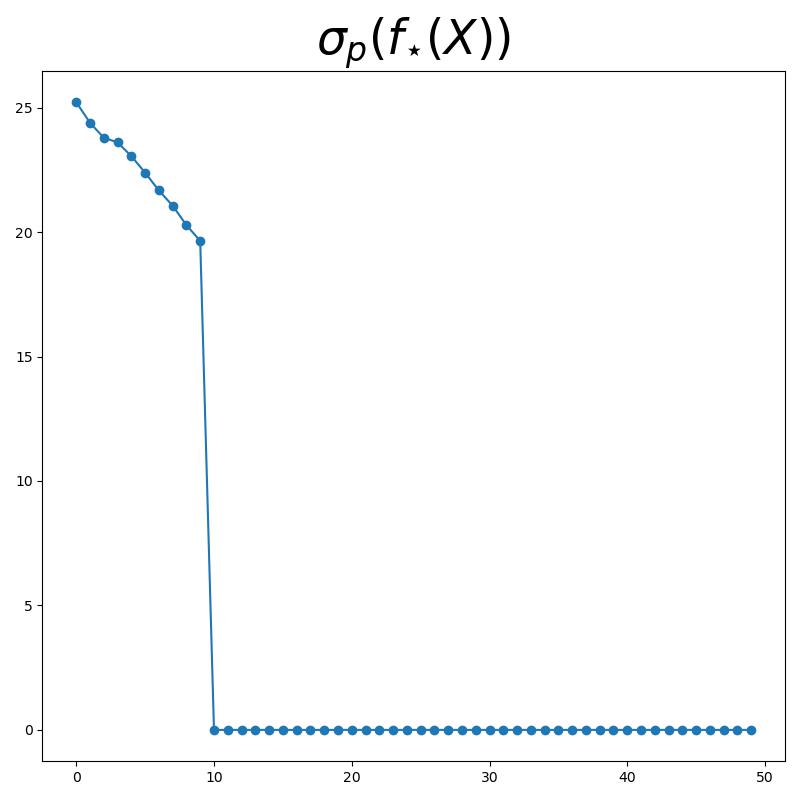}
    \includegraphics[width=0.24\textwidth]{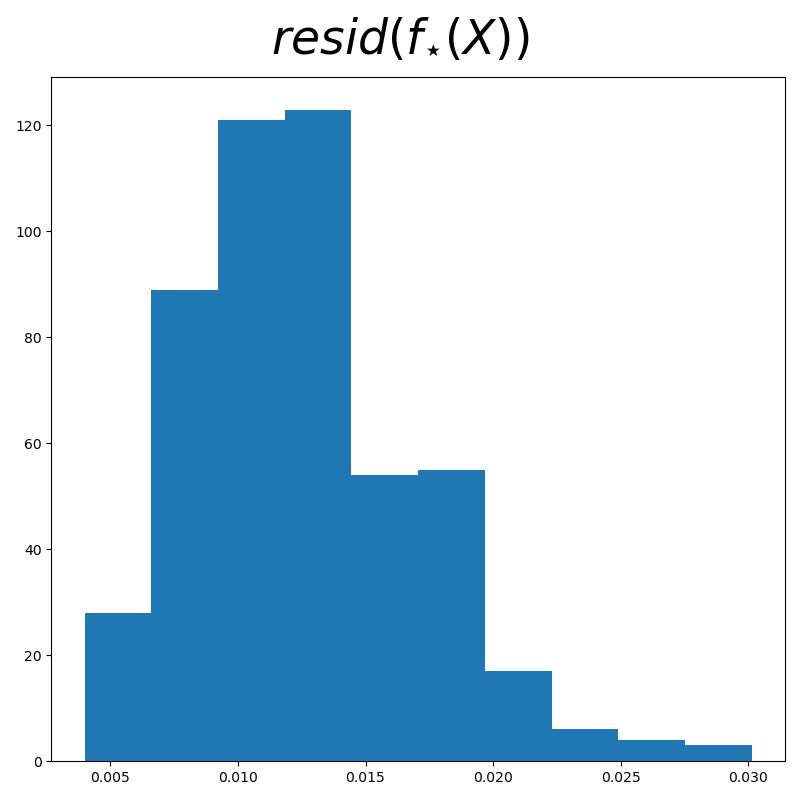}
    \caption{Performance of CTRL-SSP with varying \(\dz\). First row: spectra of the data matrix, i.e., distribution of \(\sigma_{p}(\X)\). Second row: spectra of the representation matrix \(f_{\star}(\X)\), i.e., distribution of \(\sigma_{p}(f_{\star}(\X))\), and histogram of the residuals \(\texttt{resid}(f_{\star}(\x^{i}))\) for all \(\x^{i}\) in \(\X\), with \(\dz = 10\). Third row: the same with \(\dz = 20\). Fourth row: the same with \(\dz = 40\). Fifth row: the same with \(\dz = 60\). Note that in all cases \(d_{\S} = 10\) singular values of \(f_{\star}(\X)\) are nonzero and large, and the projection residuals are close to \(0\).}
    \label{fig:ctrl_ssp_varying_dz}
\end{figure}

We now test the impact of noise on CTRL-SSP. Specifically, we attempt trials with \(\nu = 10^{-6}, \nu = 10^{-4}, \nu = 10^{-2}\), and \(\nu = 1\). This adds off-subspace noise to the data, which no longer satisfies the assumption that the data are distributed on a linear subspace. However, in all but the last setting, CTRL-SSP clearly succeeds to learn equilibria which have properties that correspond to our theoretical guarantees (\Cref{fig:ctrl_ssp_varying_nu}).

\begin{figure}
    \centering
    \includegraphics[width=0.24\textwidth]{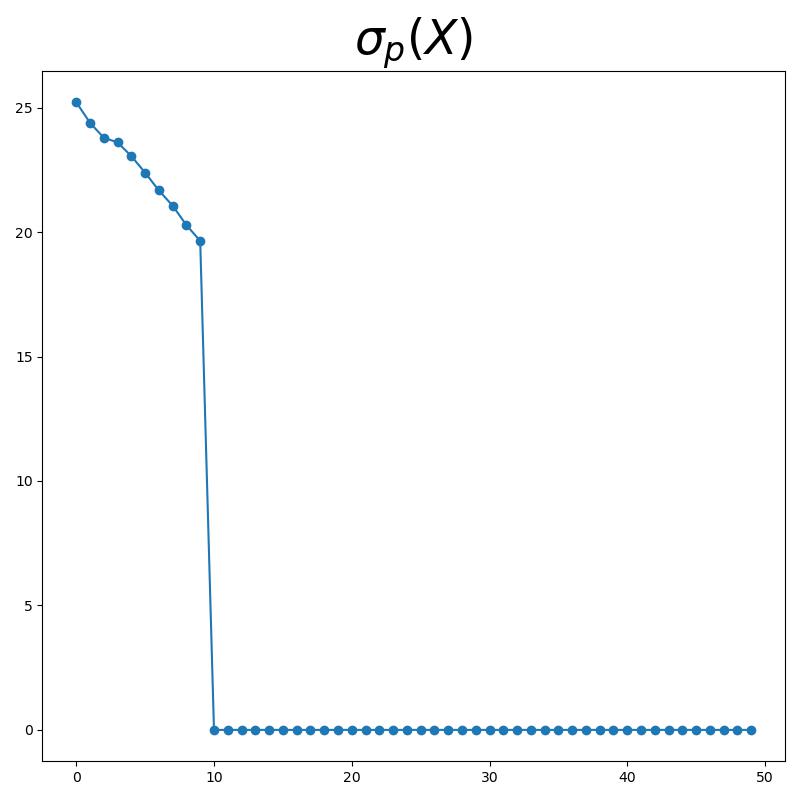}
    \includegraphics[width=0.24\textwidth]{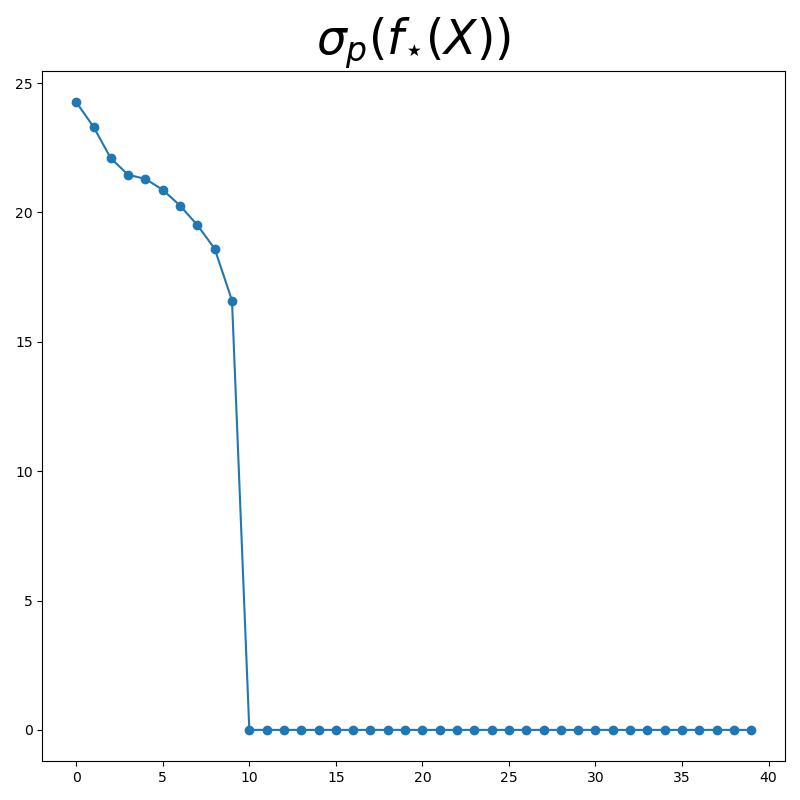}
    \includegraphics[width=0.24\textwidth]{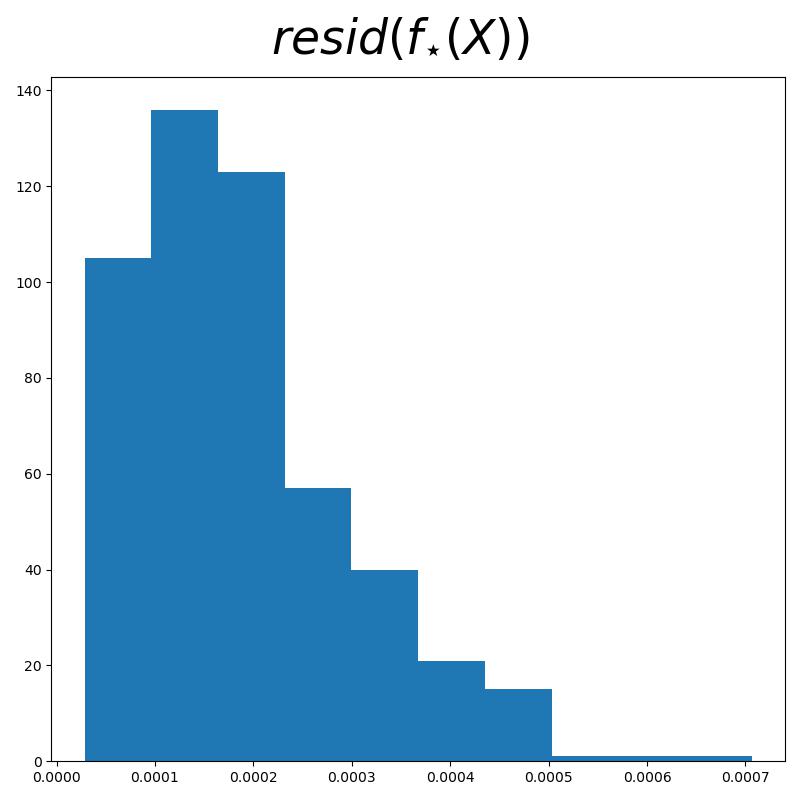}
    
    \includegraphics[width=0.24\textwidth]{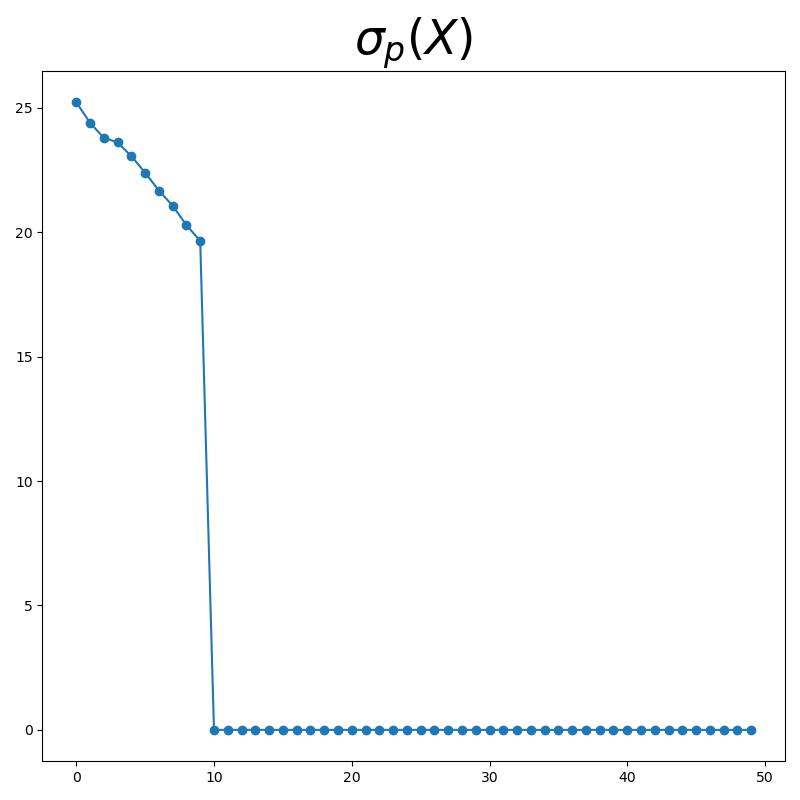}
    \includegraphics[width=0.24\textwidth]{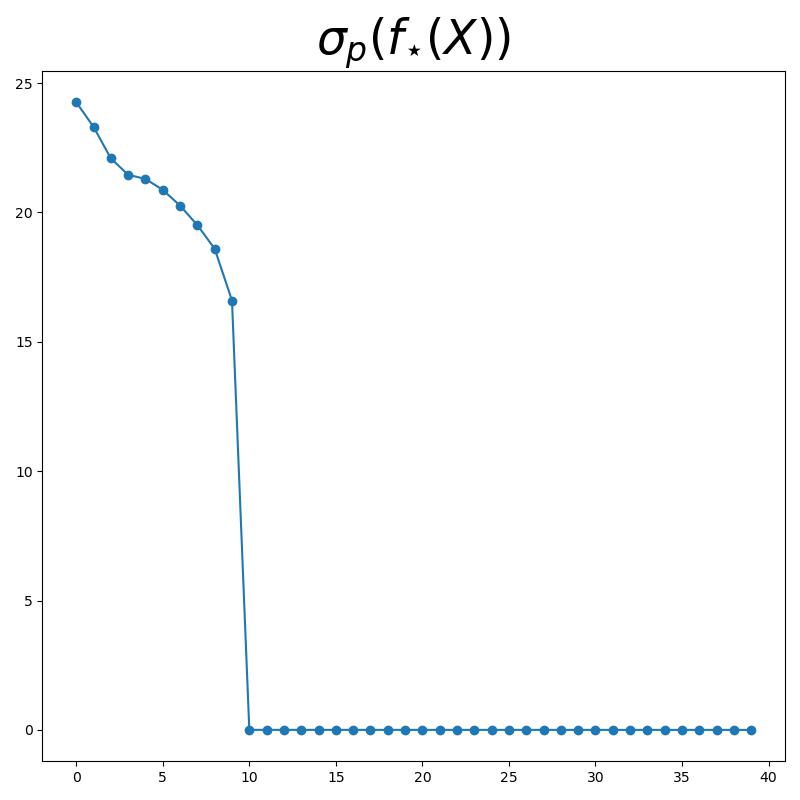}
    \includegraphics[width=0.24\textwidth]{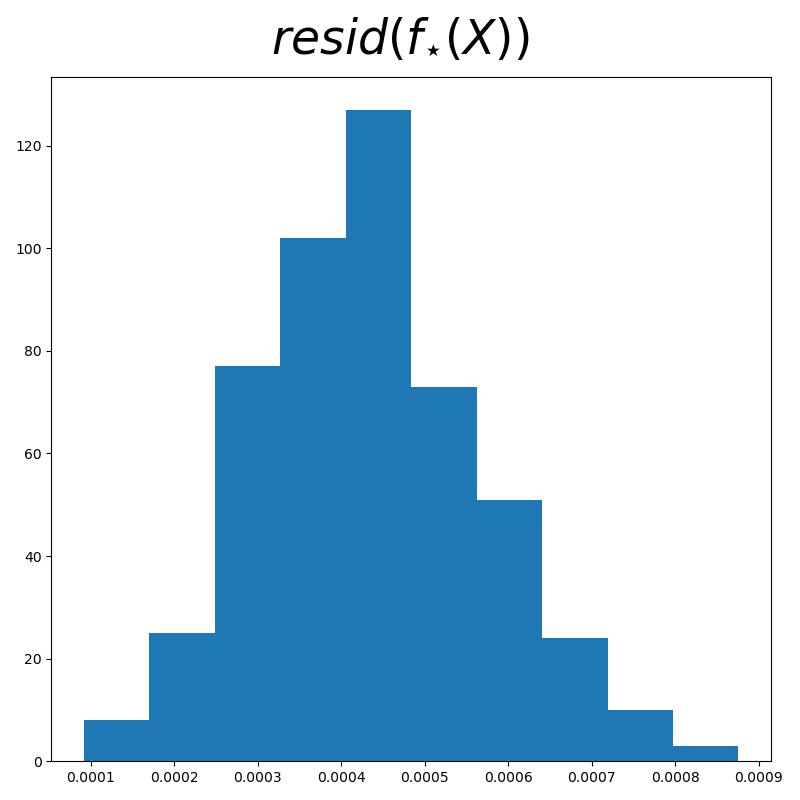}
    
    \includegraphics[width=0.24\textwidth]{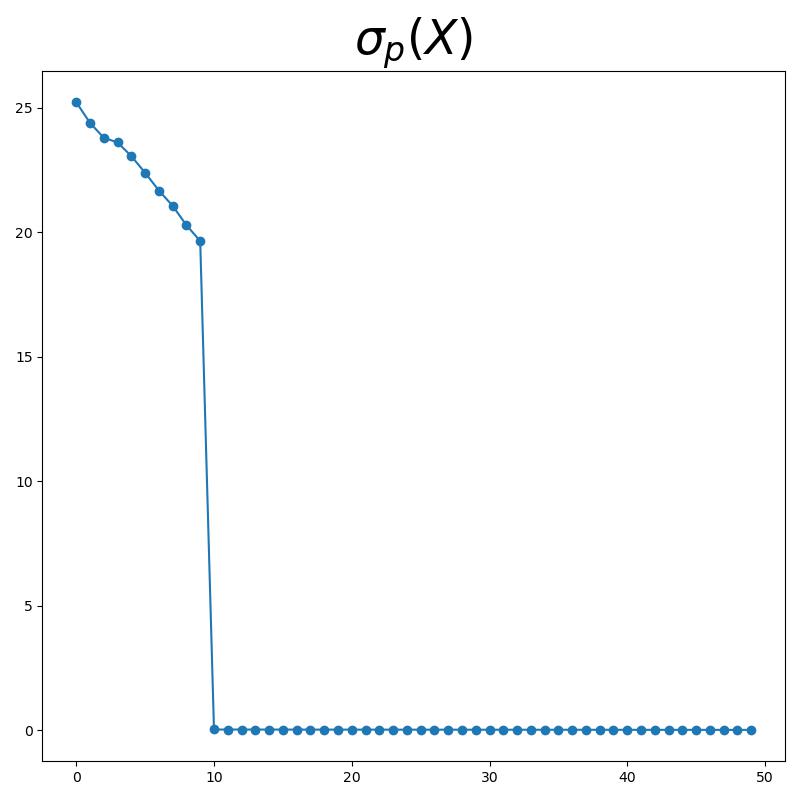}
    \includegraphics[width=0.24\textwidth]{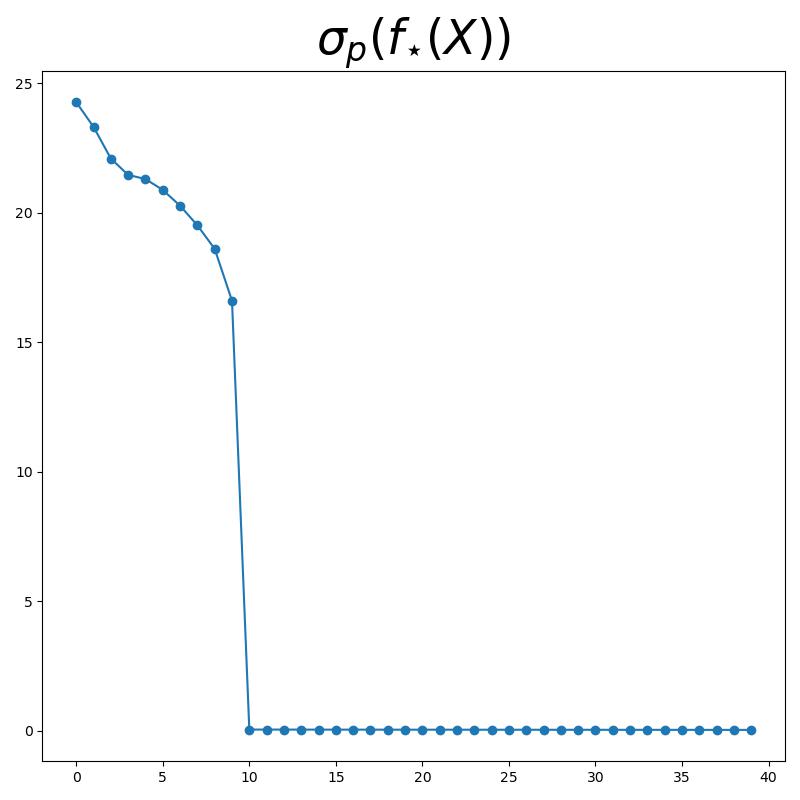}
    \includegraphics[width=0.24\textwidth]{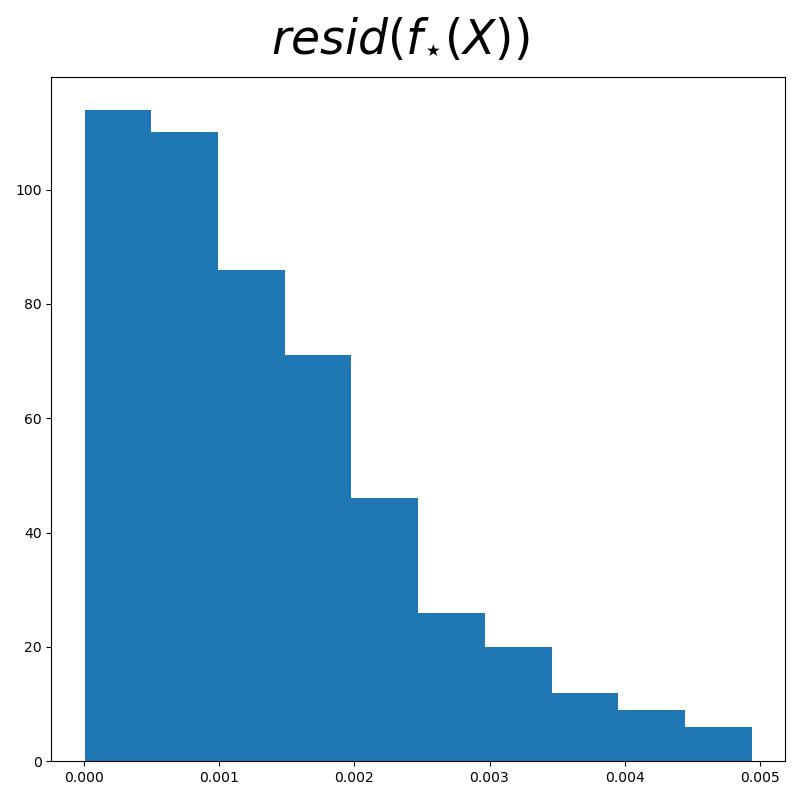}
    
    \includegraphics[width=0.24\textwidth]{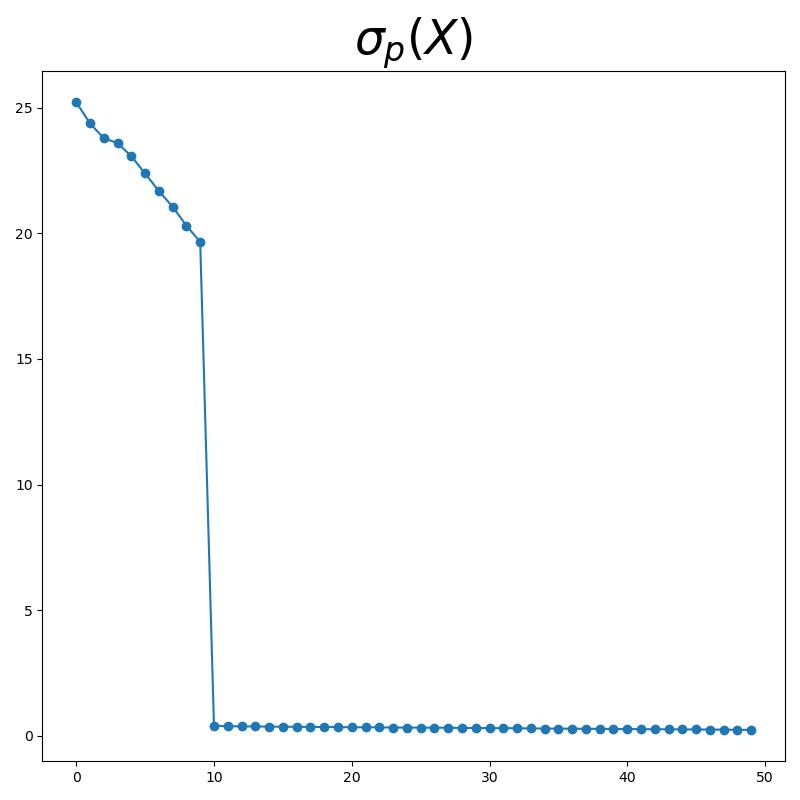}
    \includegraphics[width=0.24\textwidth]{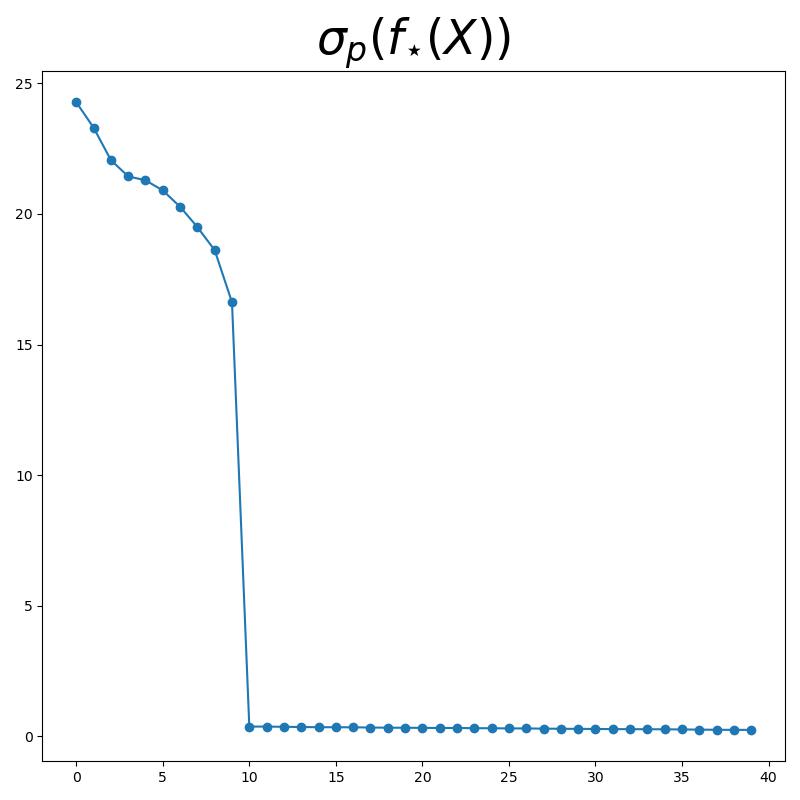}
    \includegraphics[width=0.24\textwidth]{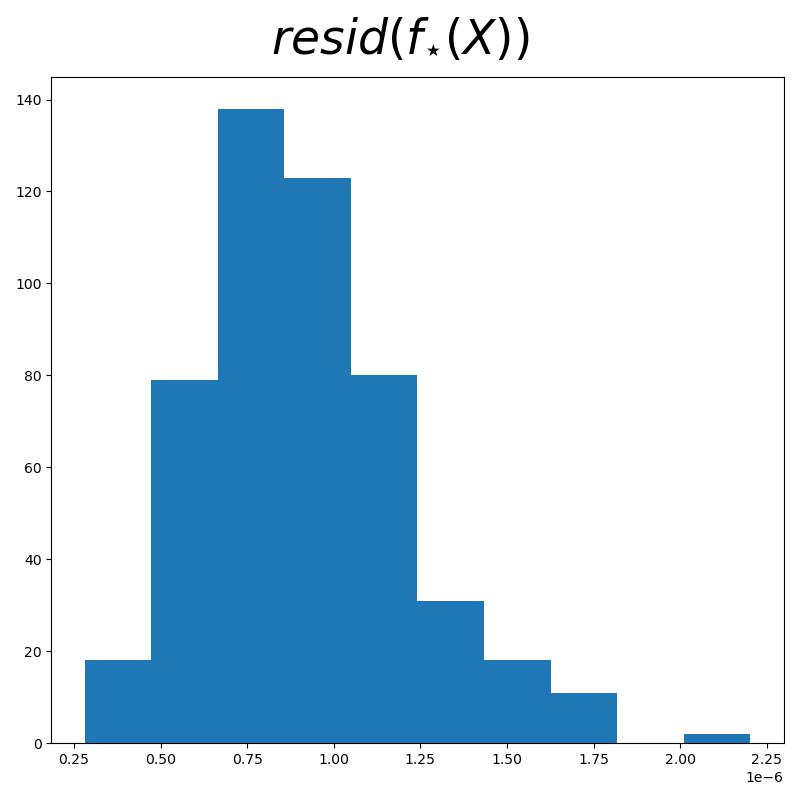}
    
    \includegraphics[width=0.24\textwidth]{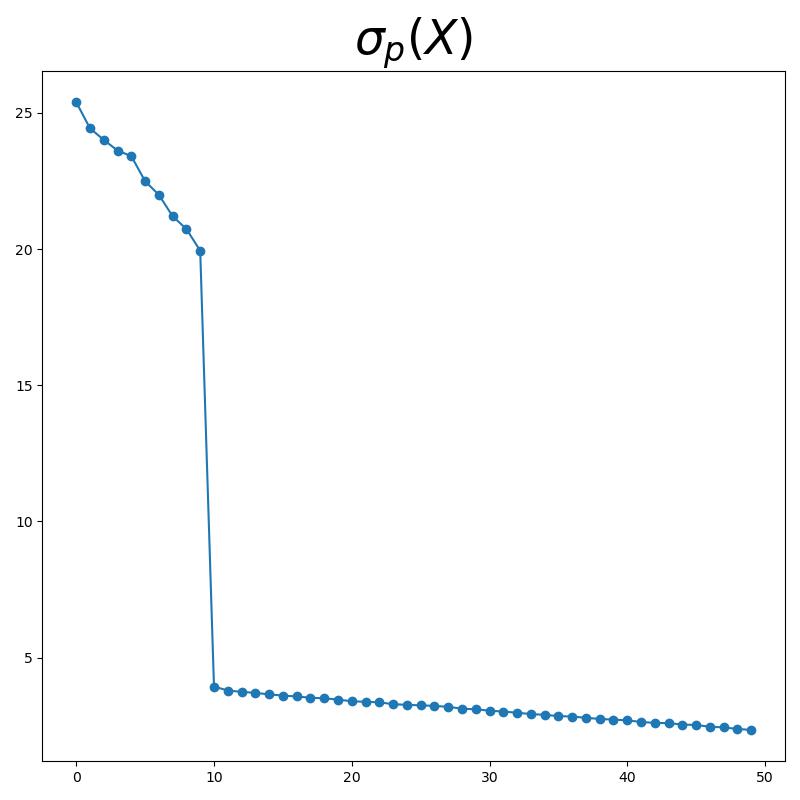}
    \includegraphics[width=0.24\textwidth]{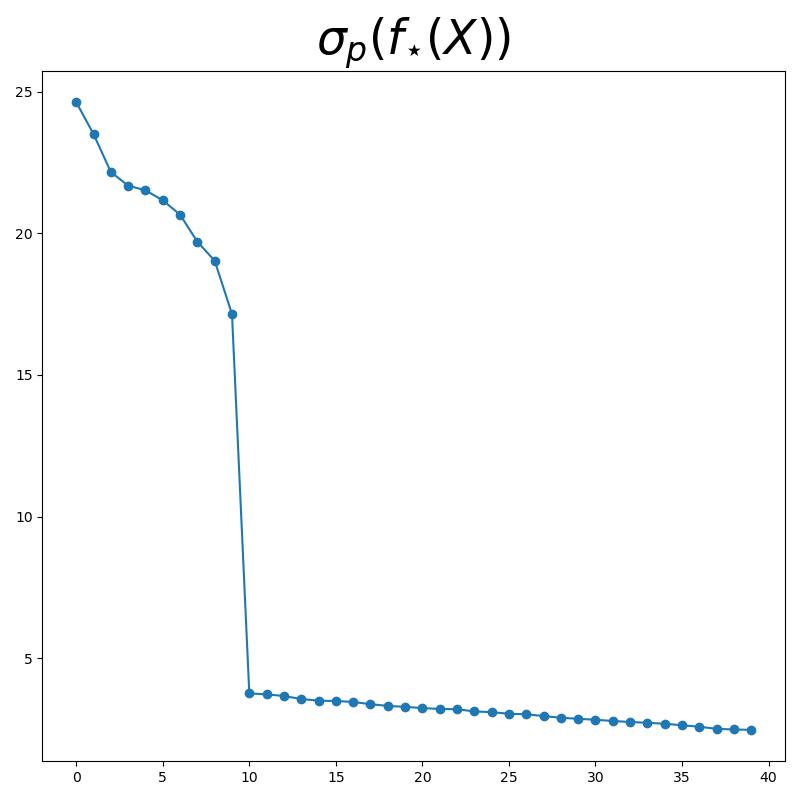}
    \includegraphics[width=0.24\textwidth]{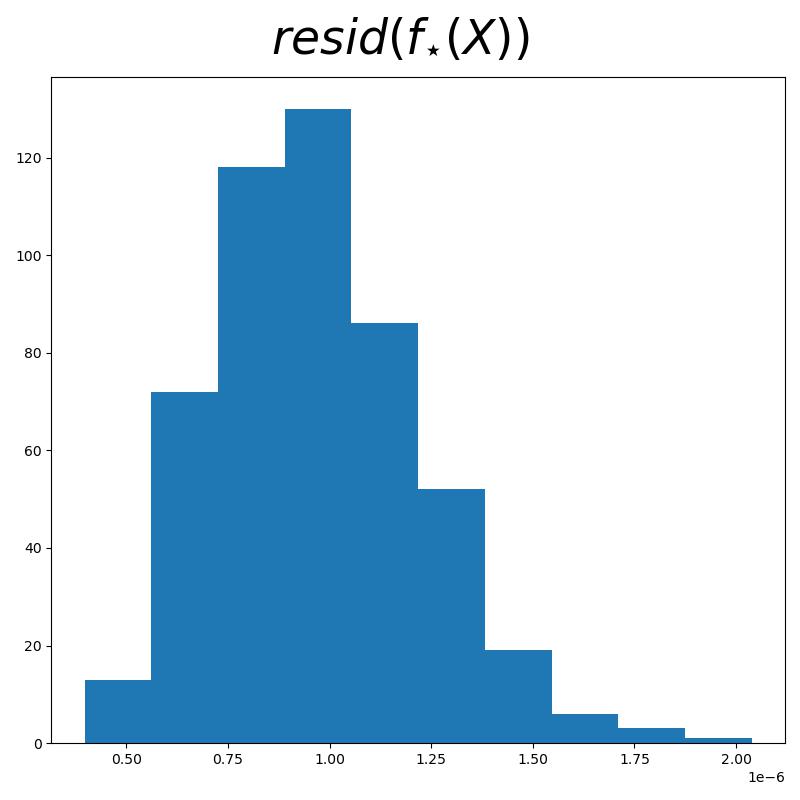}
    
    \caption{Performance of CTRL-SSP with varying \(\nu\). First row: spectra of the data matrix, i..e, distribution of \(\sigma_{p}(\X)\), spectra of the representation matrix \(f_{\star}(\X)\), i.e., distribution of \(\sigma_{p}(f_{\star}(\X))\), and histogram of the residuals \(\texttt{resid}(f_{\star}(\x^{i}))\) for all \(\x^{i}\) in \(\X\), with \(\nu = 0\). Second row: the same with \(\nu = 10^{-6}\). Third row: the same with \(\nu = 10^{-4}\). Fourth row: the same with \(\nu = 10^{-2}\). Fifth row: the same with \(\nu = 1\).}
    \label{fig:ctrl_ssp_varying_nu}
\end{figure}

%% file: appendix_ctrl_msp_experiments.tex
\section{More Empirical Evaluations}\label{sec:more_experiments}

In this section, we continue the study of the empirical convergence of learned \(f\) and \(g\) to Stackelberg equilbria of the CTRL-MSP game. We perform a more detailed analysis of CTRL-MSP's robustness to noise. Finally, we clarify the differences between CTRL-MSP, our instance of CTRL-SG as formulated in \Cref{sec:experiments}, and popular representation learning algorithms, and compare reconstruction performance.

For clarity and interpretability, we focus on the original problem of data lying on or near multiple linear subspaces. Our data generation process, baseline parameters, and optimization details are the same as in \Cref{sec:experiments}.

\subsection{CTRL-MSP}

We first study the resilience of CTRL-MSP to the choice of the ambient dimension \(\dz\), as that is the main hyperparameter to choose in this model. That is, in the framework of \Cref{sec:experiments} where \(d_{1} = 3\), \(d_{2} = 4\), and \(d_{3} = 5\), we vary \(\dz\) across the following levels: \(\dz = 20\), \(\dz = 30\), \(\dz = 40\), and \(\dz = 60\). In all regimes, we see that CTRL-MSP behaves equivalently and achieves success (\Cref{fig:ctrl_msp_varying_dz}).

\begin{figure}
    \centering
    \includegraphics[width=0.24\textwidth]{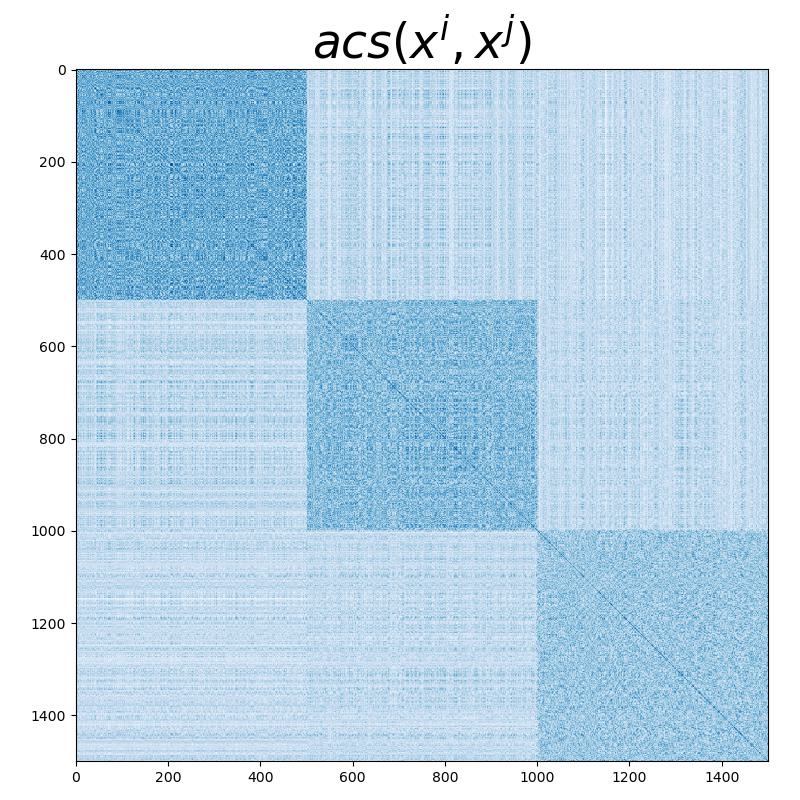}
    
    \includegraphics[width=0.24\textwidth]{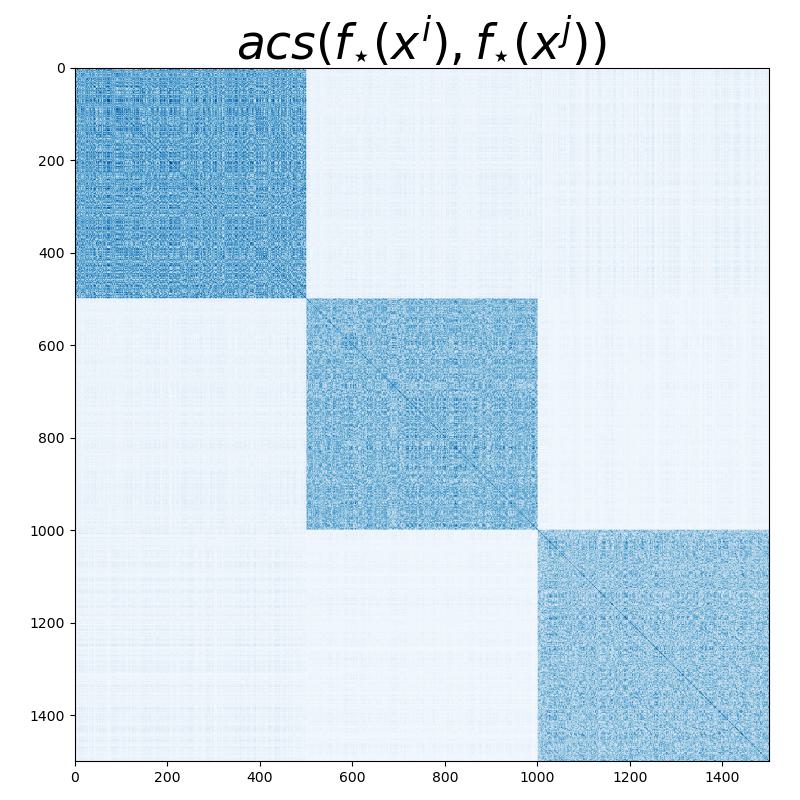}
    \includegraphics[width=0.24\textwidth]{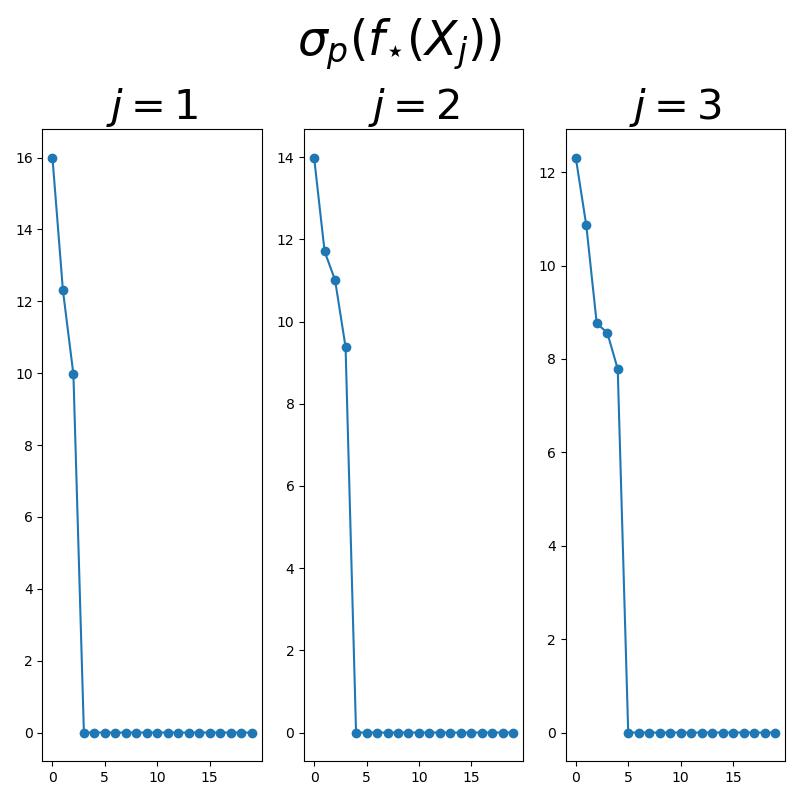}
    \includegraphics[width=0.24\textwidth]{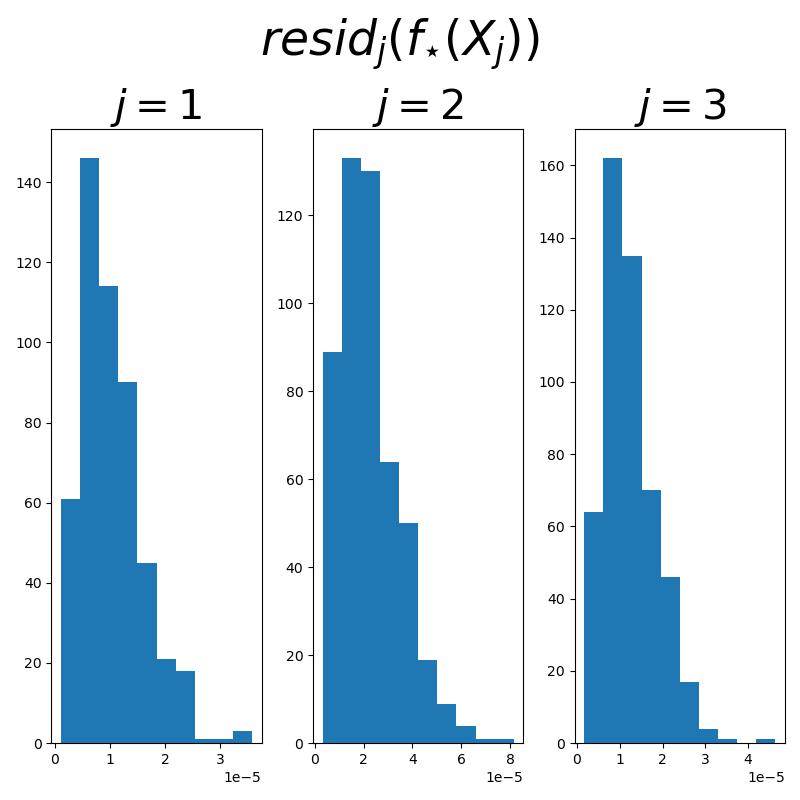}
    
    \includegraphics[width=0.24\textwidth]{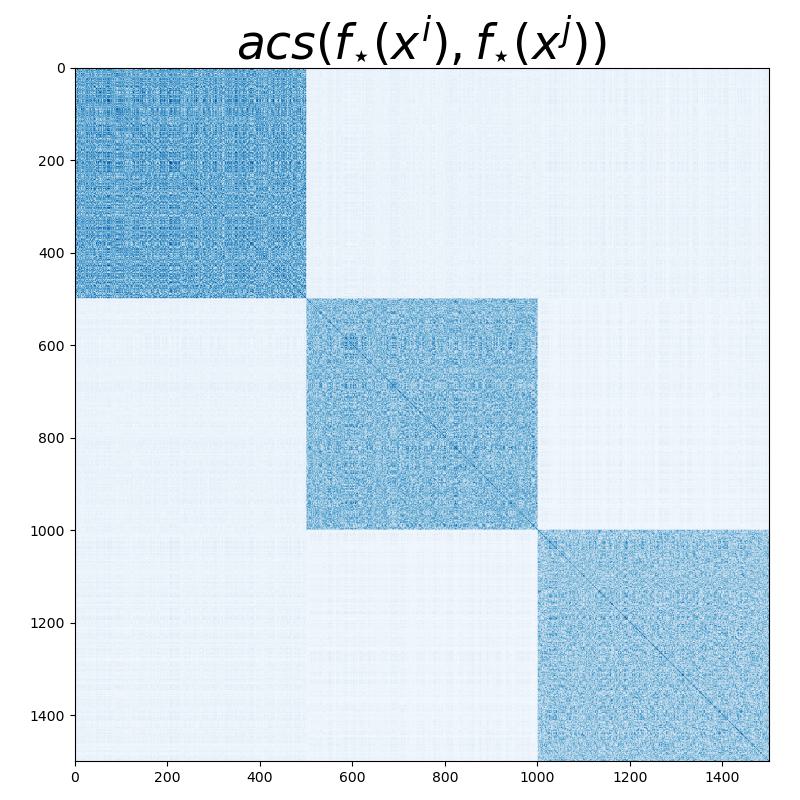}
    \includegraphics[width=0.24\textwidth]{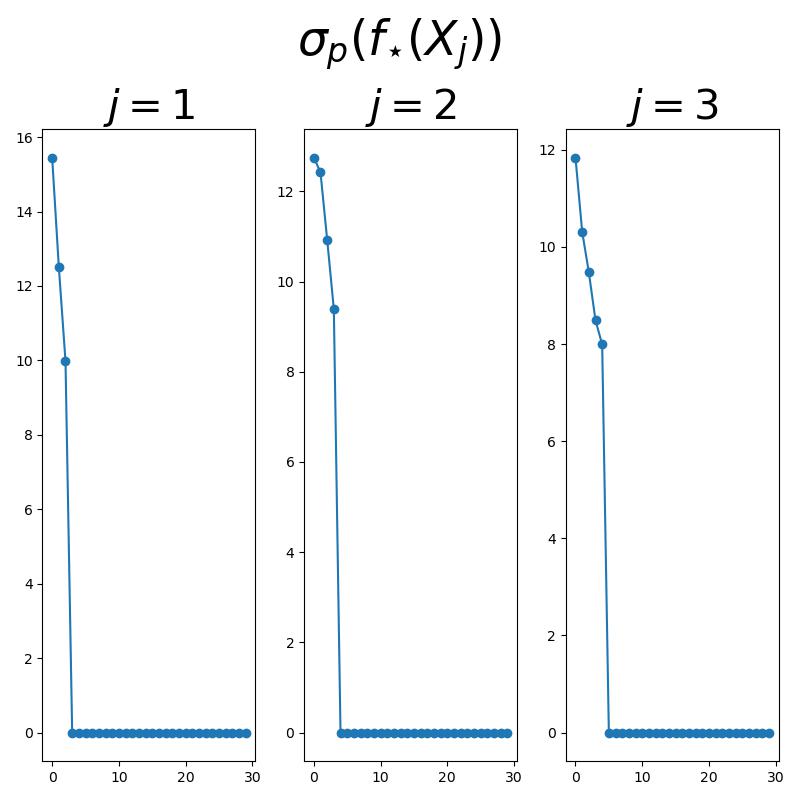}
    \includegraphics[width=0.24\textwidth]{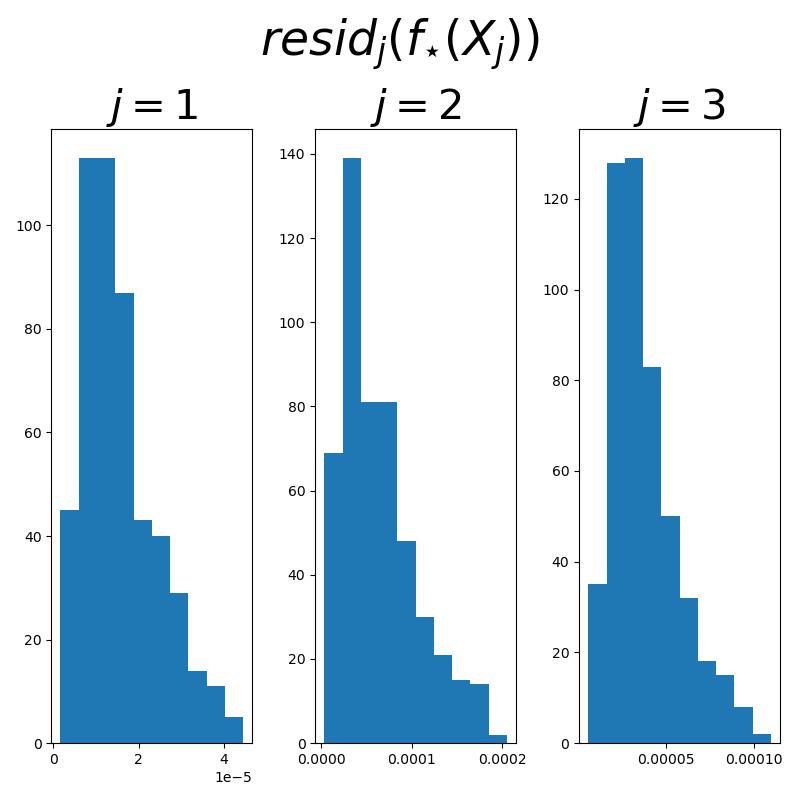}
    
    \includegraphics[width=0.24\textwidth]{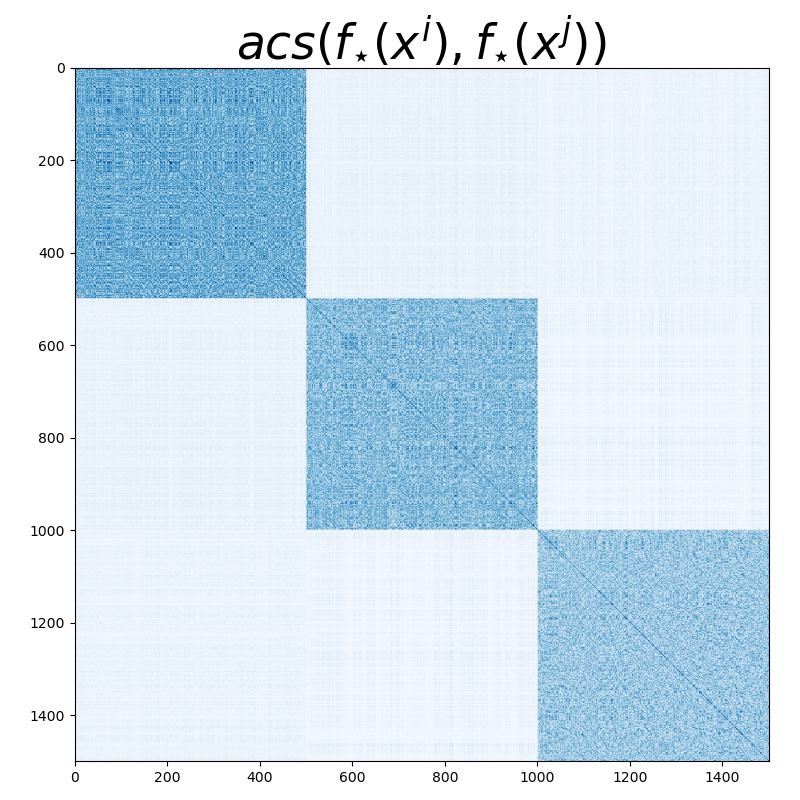}
    \includegraphics[width=0.24\textwidth]{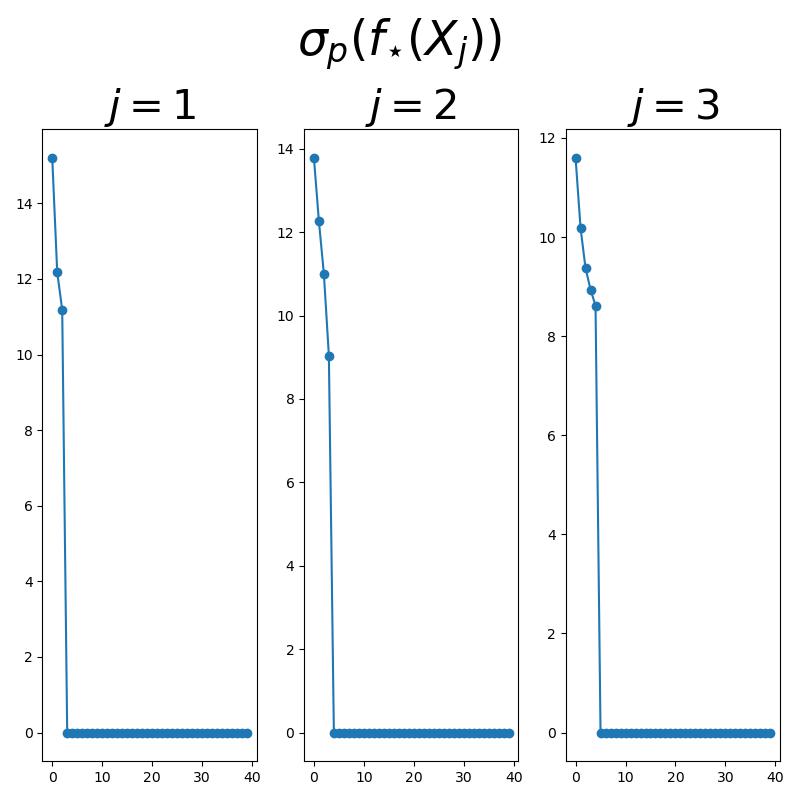}
    \includegraphics[width=0.24\textwidth]{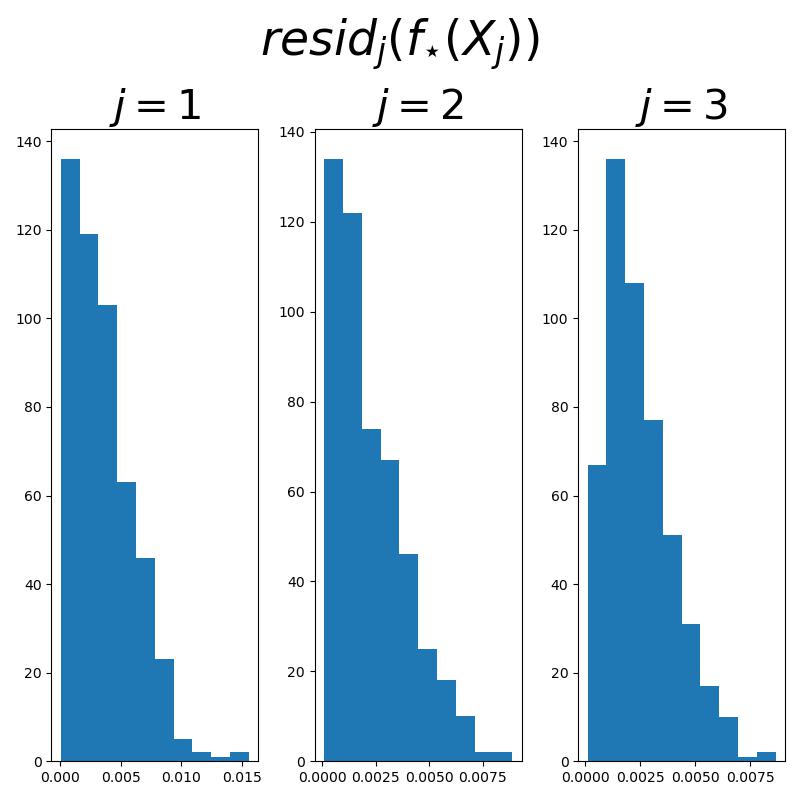}
    
    \includegraphics[width=0.24\textwidth]{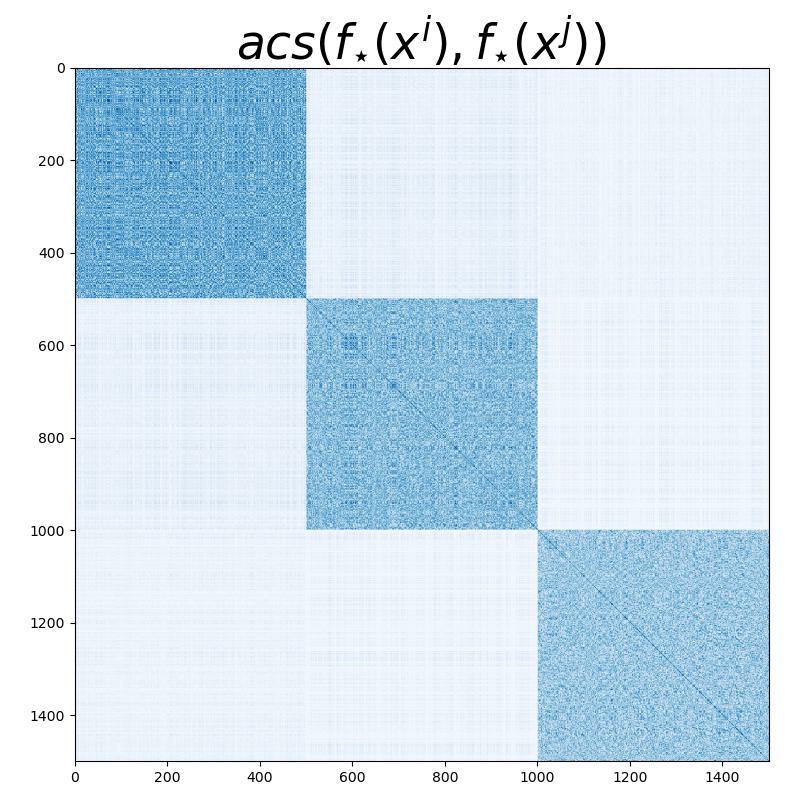}
    \includegraphics[width=0.24\textwidth]{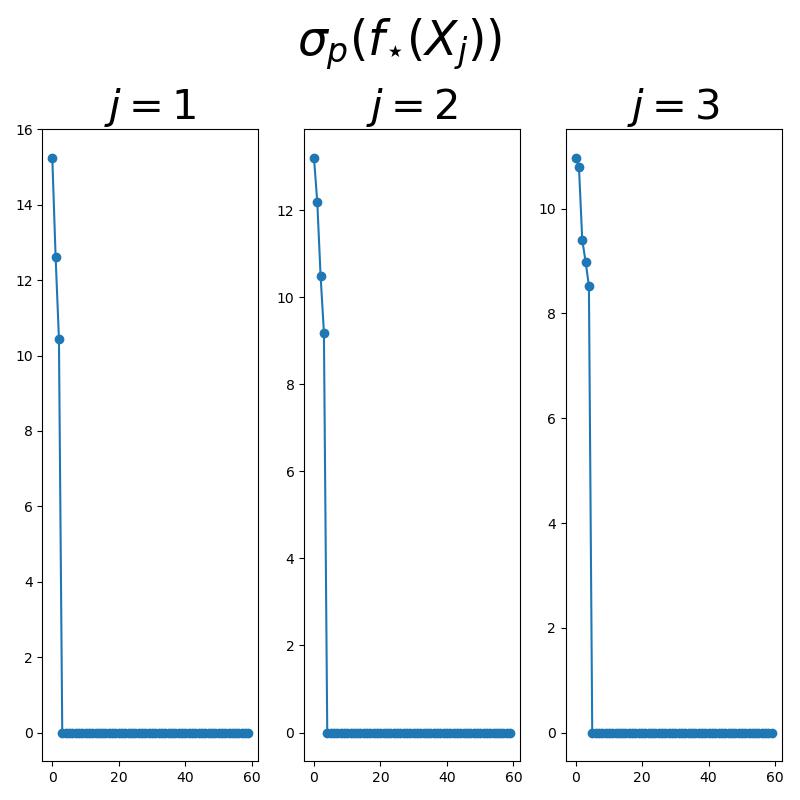}
    \includegraphics[width=0.24\textwidth]{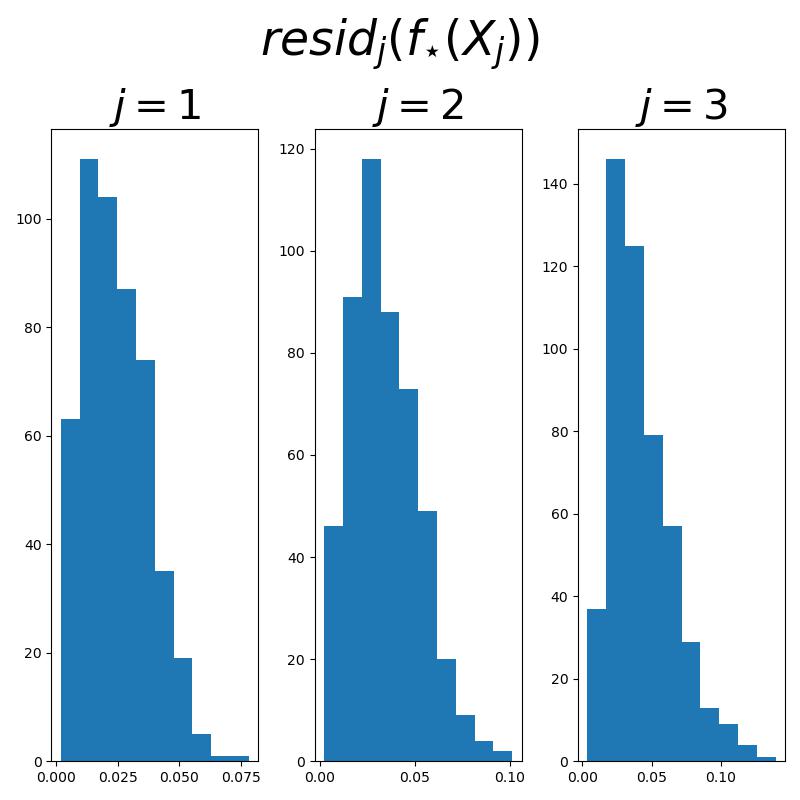}
    
    \caption{Performance of CTRL-MSP with varying \(\dz\). First row: heatmap of absolute cosine similarities of the original data \(\texttt{acs}(\x^{p}, \x^{q})\). Second row: heatmap of absolute cosine similarities of the learned representations \(\texttt{acs}(f_{\star}(\x^{p}), f_{\star}(\x^{q}))\), spectra of the representation matrices \(f_{\star}(\X_{j})\) for every \(j\), histogram of the residuals \(\texttt{resid}_{j}(f_{\star}(\x^{i}))\) for all \(\x^{i}\) in \(\X_{j}\) for every \(j\), with \(\dz = 20\). Third row: the same with \(\dz = 30\). Fourth row: the same with \(\dz = 40\). Fifth row: the same with \(\dz = 60\).}
    \label{fig:ctrl_msp_varying_dz}
\end{figure}

We now study the resilience of CTRL-MSP to noise in the data. In the framework of \Cref{sec:experiments}, we vary \(\nu\) across the following levels: \(\nu = 10^{-6}\), \(\nu = 10^{-4}\), \(\nu = 10^{-2}\), and \(\nu = 1\). In all but the last regime, we see that CTRL-MSP is resilient to noise and achieves success (\Cref{fig:ctrl_msp_varying_nu}). Interestingly, between the second-to-last and last noise levels, the learned model quality degenerates rapidly; this alludes to a phase transition behavior, which could be an interesting phenomenon to analyze in future work.

\begin{figure}
    \centering
    \includegraphics[width=0.24\textwidth]{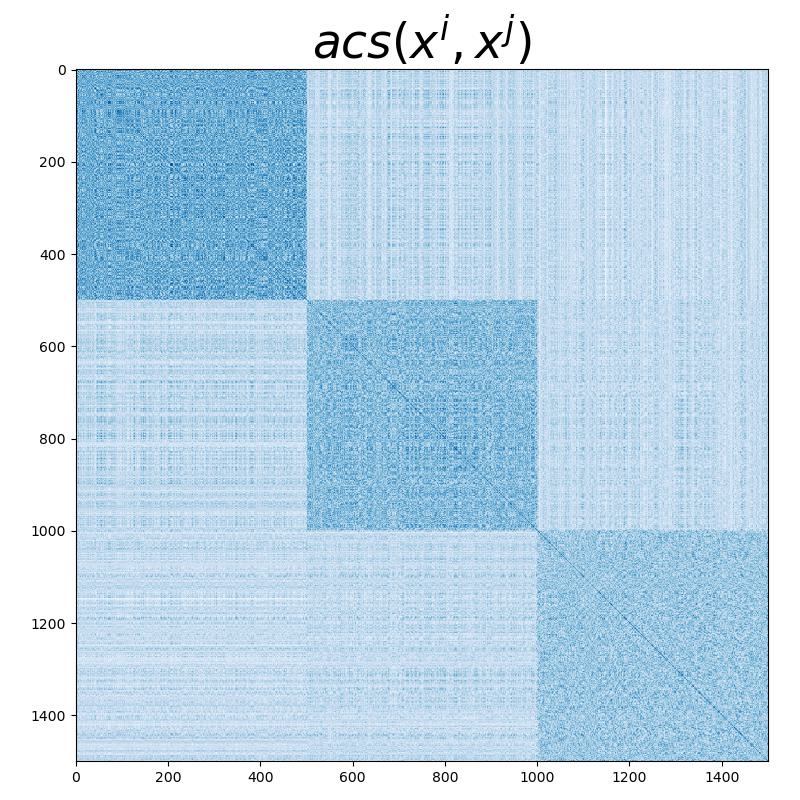}
    \includegraphics[width=0.24\textwidth]{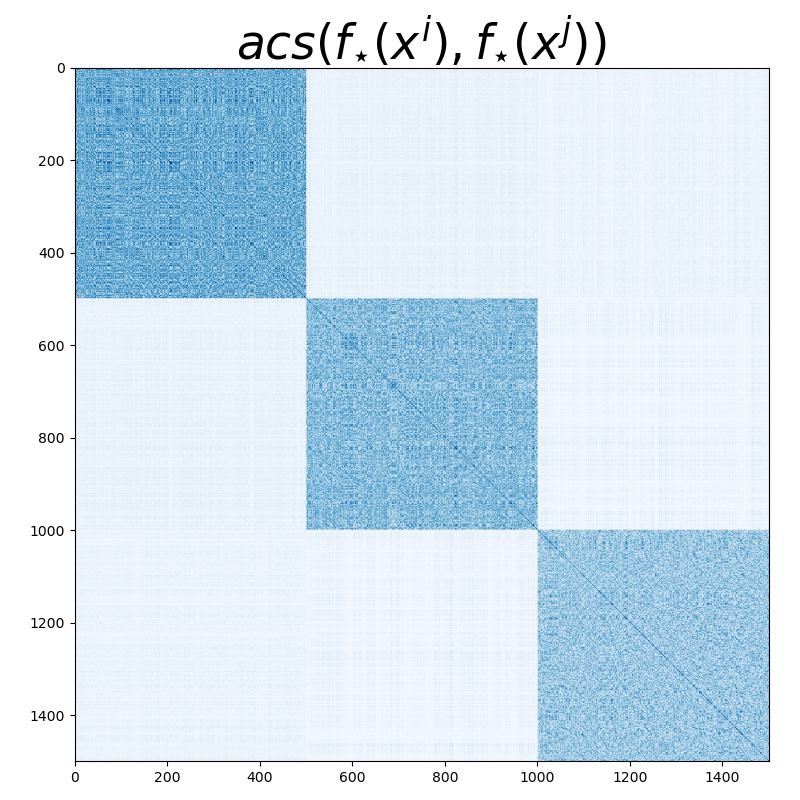}
    \includegraphics[width=0.24\textwidth]{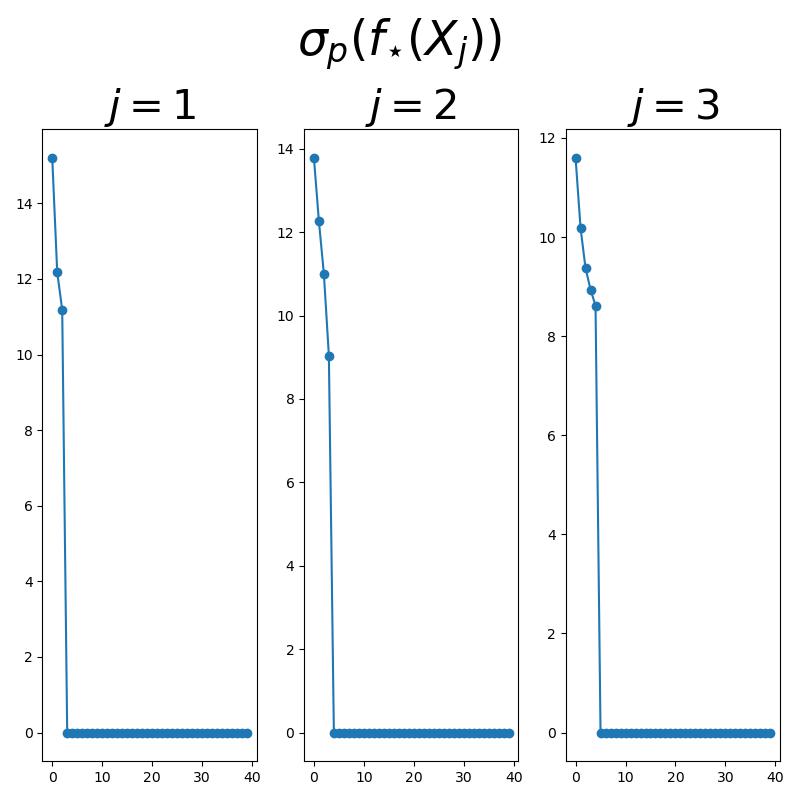}
    \includegraphics[width=0.24\textwidth]{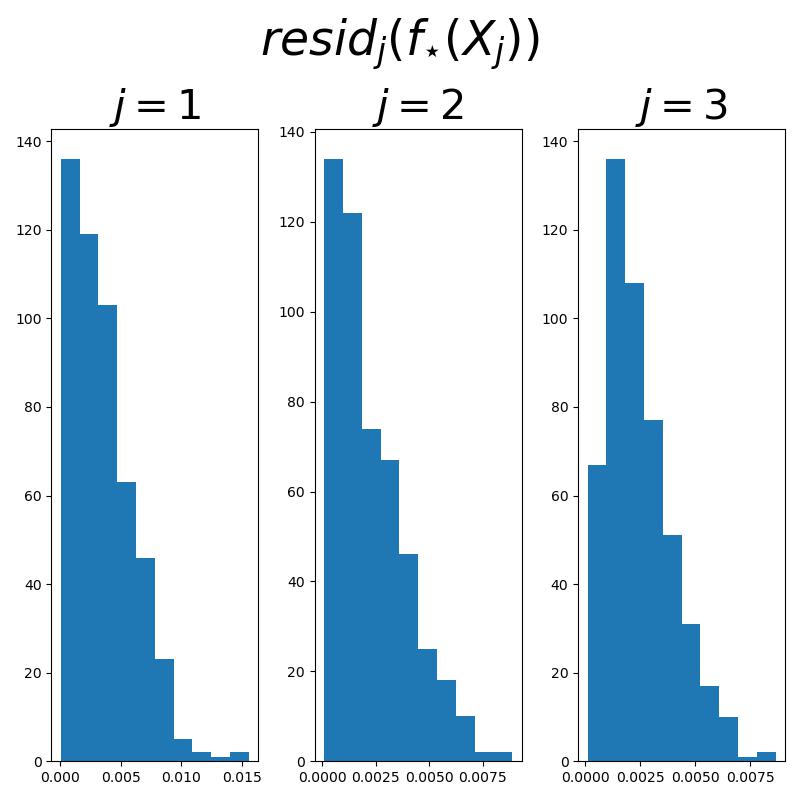}
    
    \includegraphics[width=0.24\textwidth]{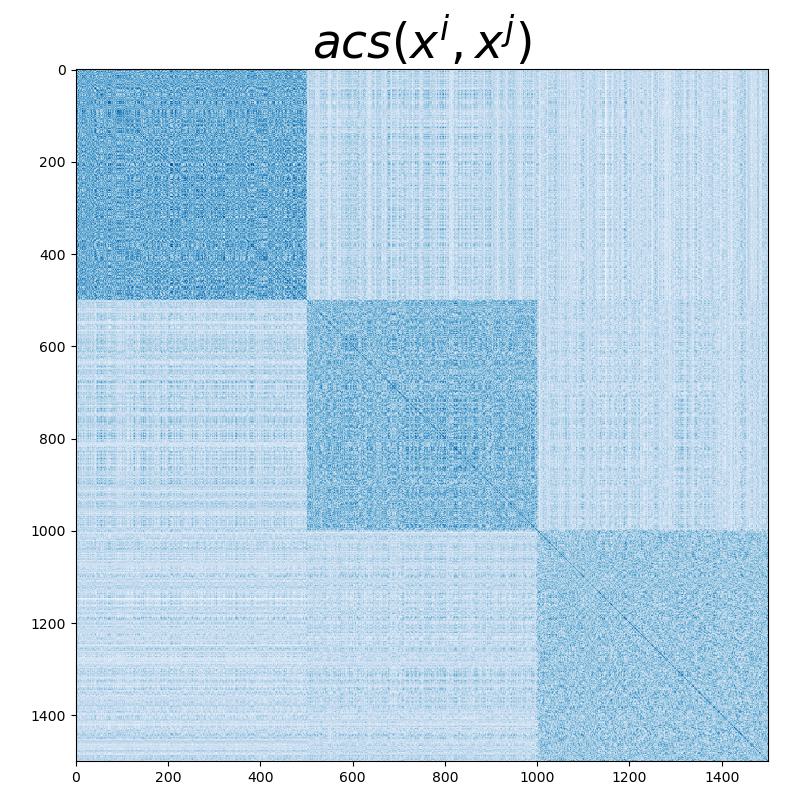}
    \includegraphics[width=0.24\textwidth]{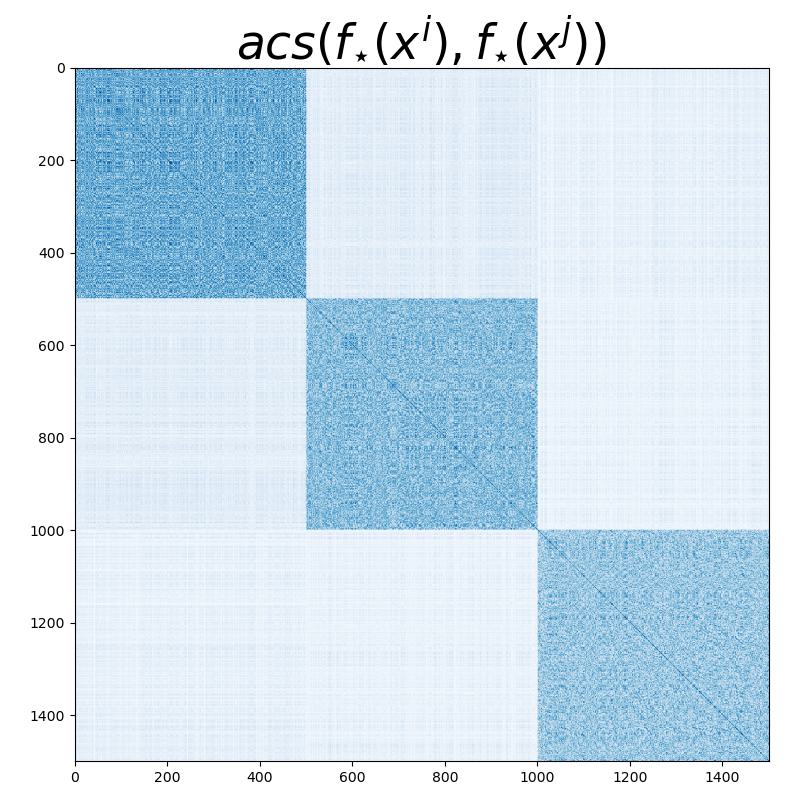}
    \includegraphics[width=0.24\textwidth]{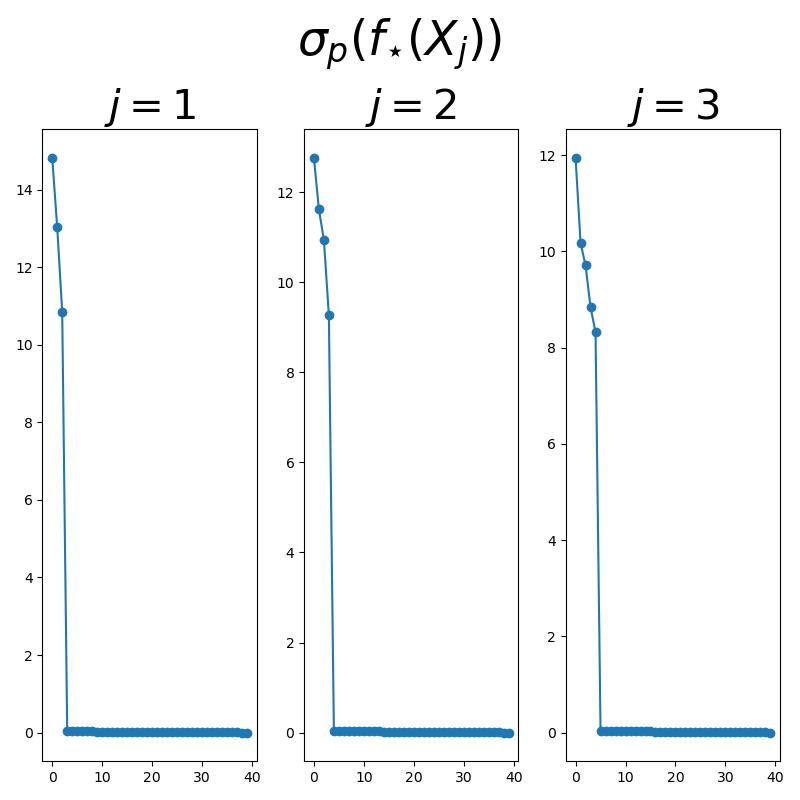}
    \includegraphics[width=0.24\textwidth]{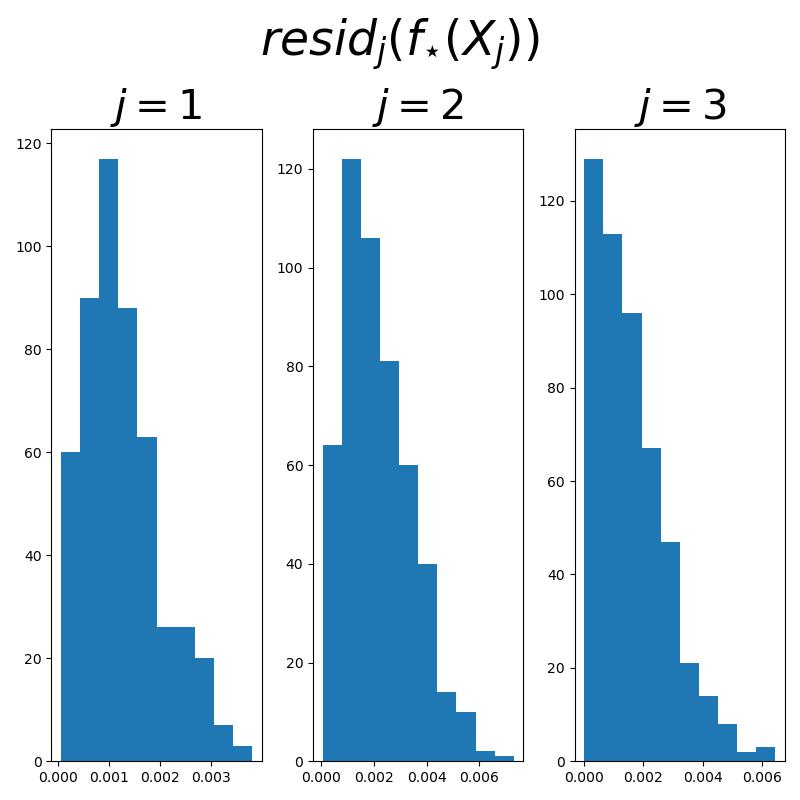}
    
    \includegraphics[width=0.24\textwidth]{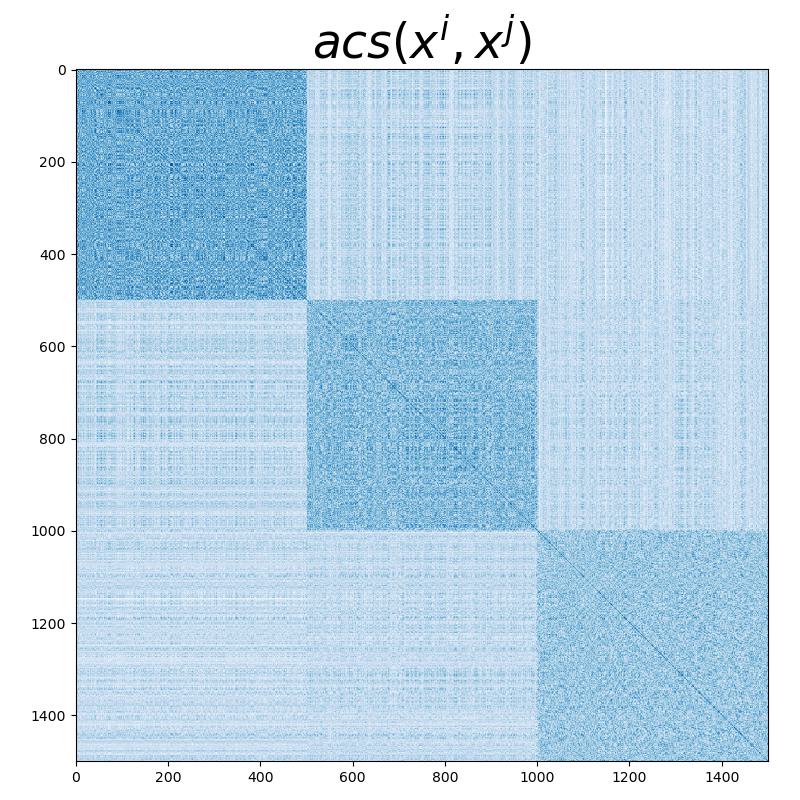}
    \includegraphics[width=0.24\textwidth]{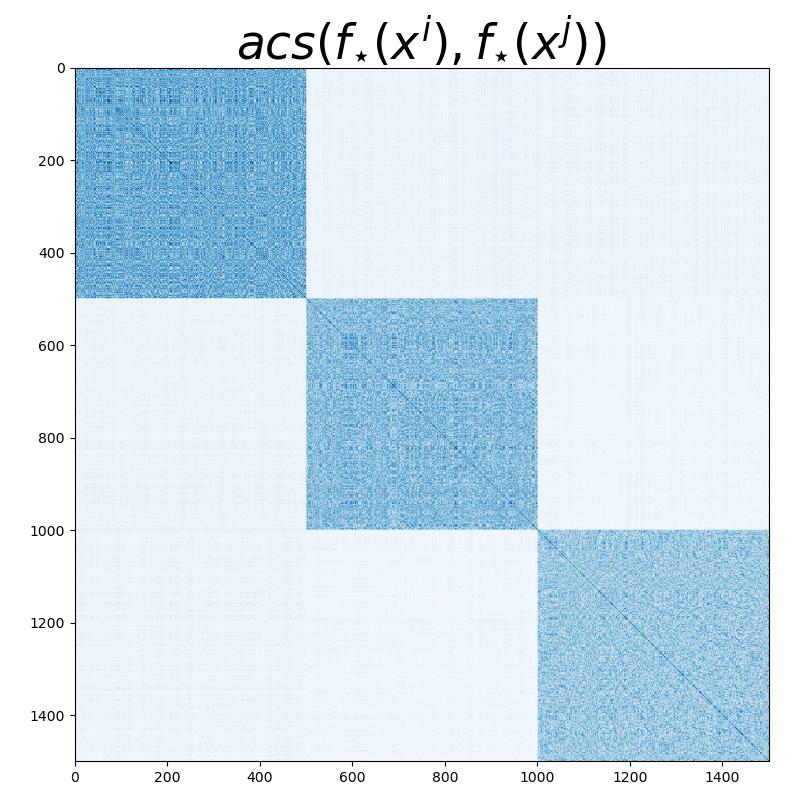}
    \includegraphics[width=0.24\textwidth]{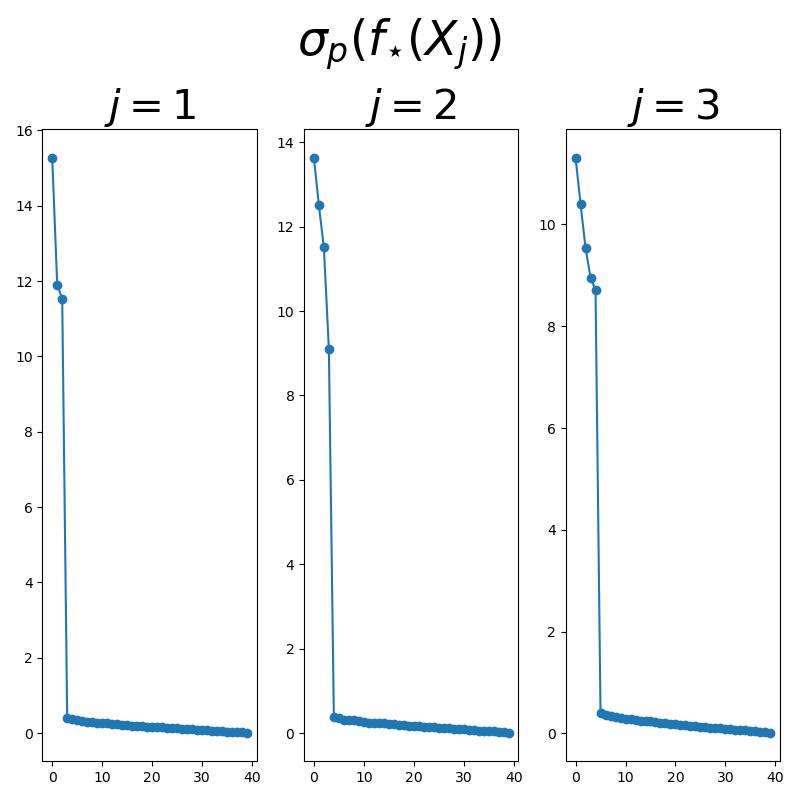}
    \includegraphics[width=0.24\textwidth]{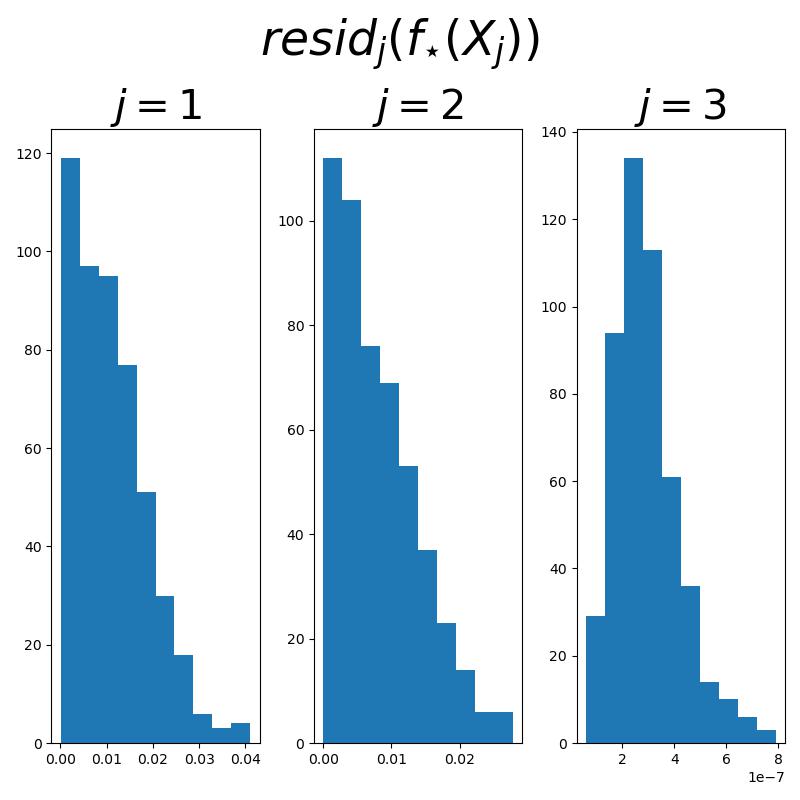}
    
    \includegraphics[width=0.24\textwidth]{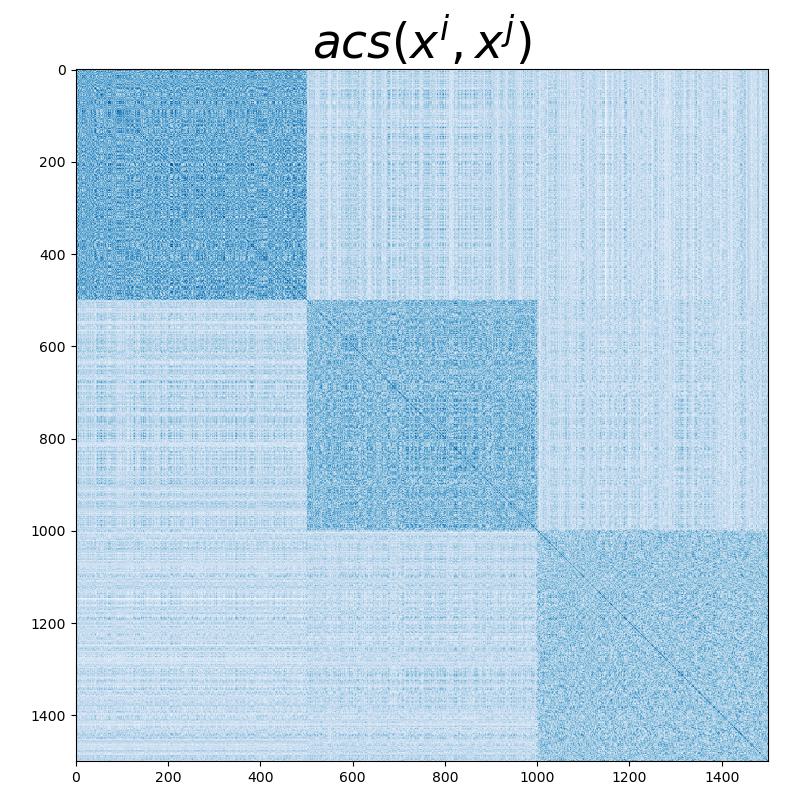}
    \includegraphics[width=0.24\textwidth]{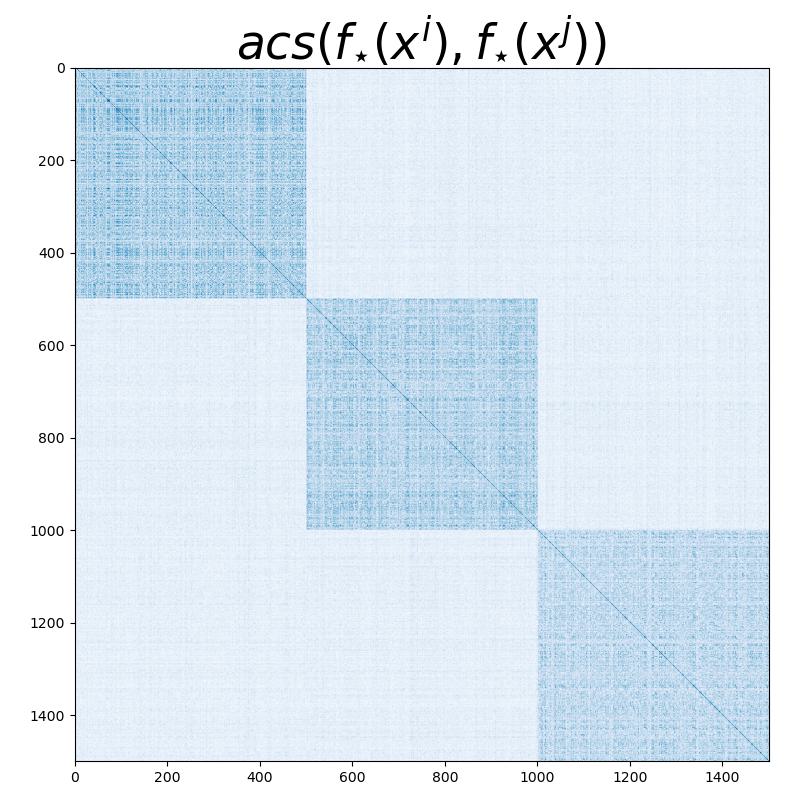}
    \includegraphics[width=0.24\textwidth]{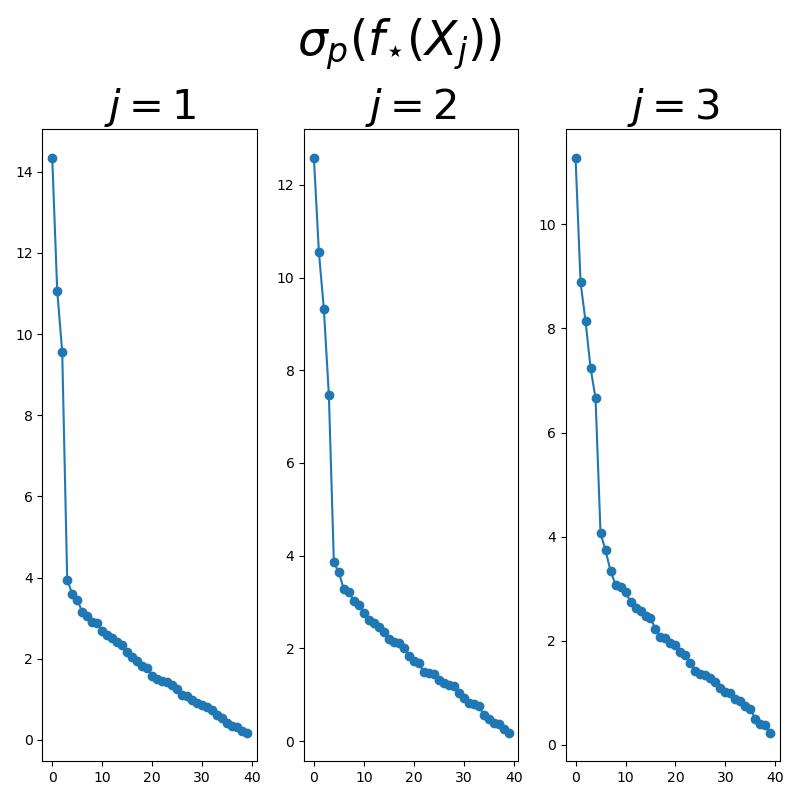}
    \includegraphics[width=0.24\textwidth]{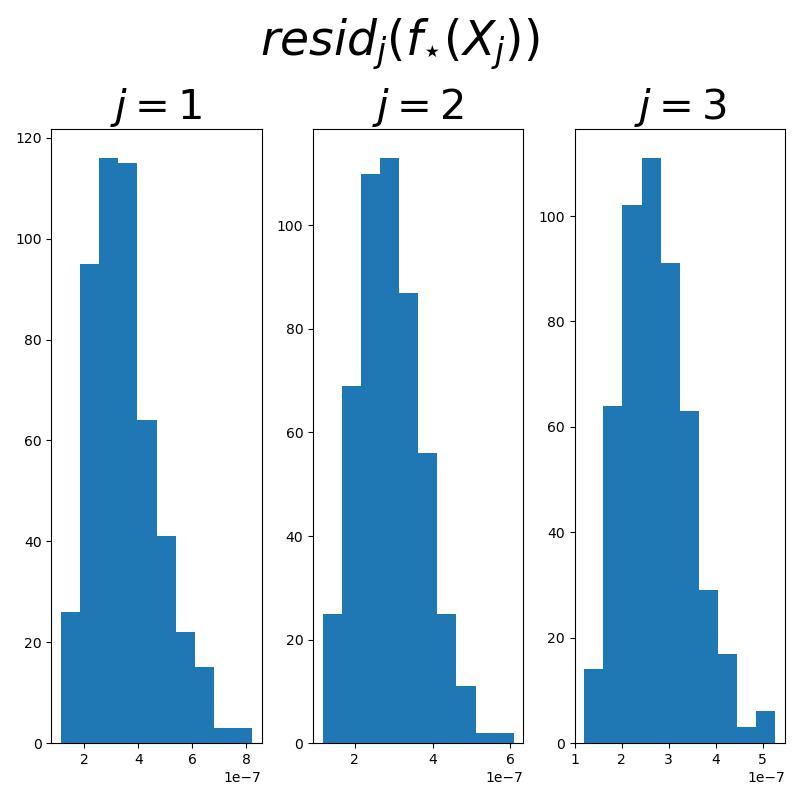}
    
    \includegraphics[width=0.24\textwidth]{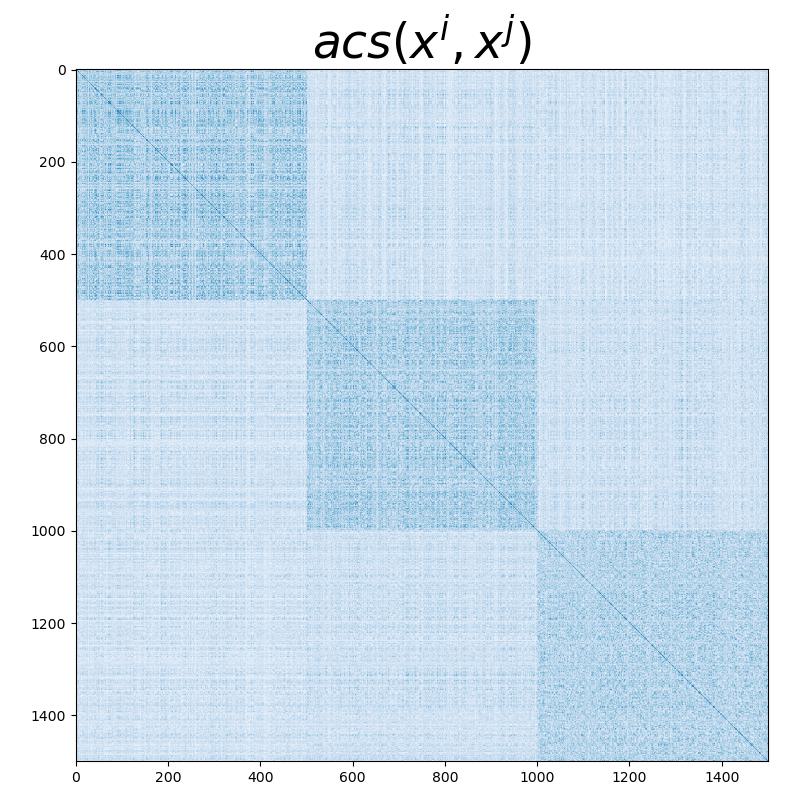}
    \includegraphics[width=0.24\textwidth]{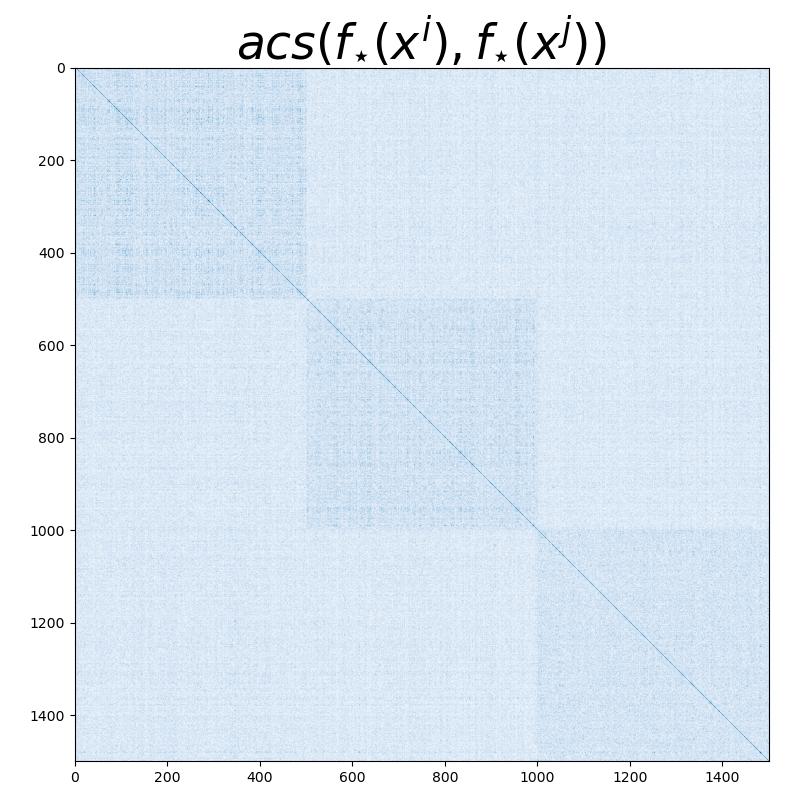}
    \includegraphics[width=0.24\textwidth]{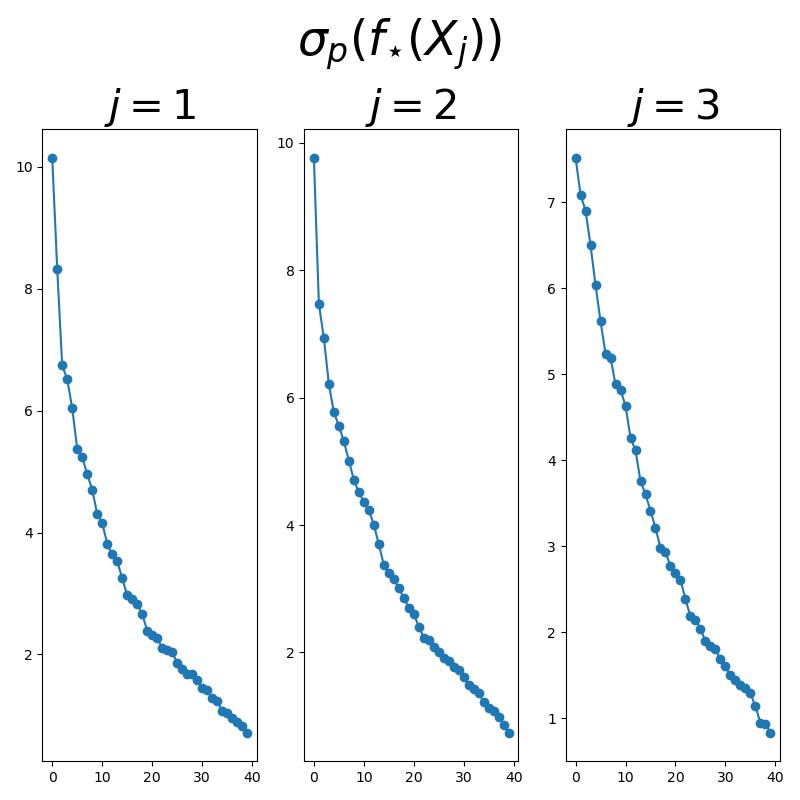}
    \includegraphics[width=0.24\textwidth]{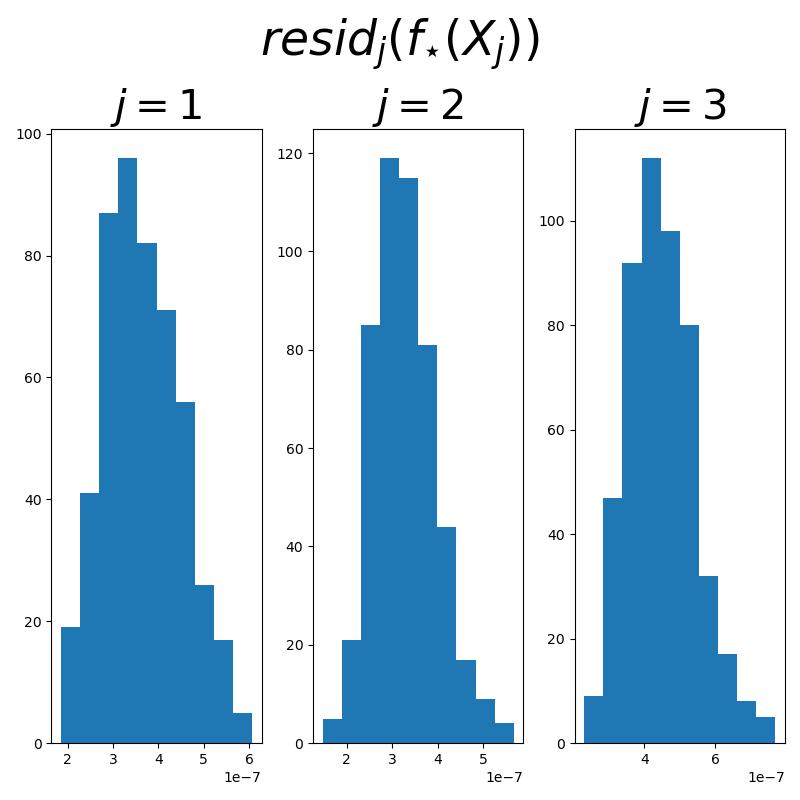}
    
    \caption{Performance of CTRL-MSP with varying \(\nu\). First row: heatmap of absolute cosine similarities of the original data \(\texttt{acs}(\x^{p}, \x^{q})\), heatmap of absolute cosine similarities of the learned representations \(\texttt{acs}(f_{\star}(\x^{p}), f_{\star}(\x^{q}))\), spectra of the representation matrices \(f_{\star}(\X_{j})\) for every \(j\), histogram of the residuals \(\texttt{resid}_{j}(f_{\star}(\x^{i}))\) for all \(\x^{i}\) in \(\X_{j}\) for every \(j\), with \(\nu = 0\). Second row: the same with \(\nu = 10^{-6}\). Third row: the same with \(\nu = 10^{-4}\). Fourth row: the same with \(\nu = 10^{-2}\). Fifth row: the same with \(\nu = 1\).}
    \label{fig:ctrl_msp_varying_nu}
\end{figure}

\subsection{Comparisons}

Finally, we discuss the differences between CTRL-MSP, our CTRL-SG instance, and popular representation learning algorithms. The biggest difference is that \textit{no} mainstream representation learning algorithm learns, or even claims to learn, \textit{explicit and interpretable} representations for structured high-dimensional data, much less having \textit{provable guarantees} for achieving this goal. Thus, it is impossible to build a truly fair comparison of CTRL-MSP and our CTRL-SG instance versus popular representation learning algorithms.

Despite this, we still show that CTRL-MSP and our CTRL-SG instance capture low-dimensional non-noisy structure of the data in a way that popular algorithms fail to do. We do this by looking at the correlation structure and per-class spectra of the original data \(\X\) as well as the autoencoded data \((g_{\star} \circ f_{\star})(\X)\).

More specifically, we use the following popular and well-studied supervised architectures for a comparison:
\begin{itemize}
    \item Conditional VAE \citep{sohn2015learning};
    \item Conditional GAN \citep{mirza2014conditional};
    \item InfoGAN \citep{chen2016infogan}.
\end{itemize}
Each neural network in each architecture, including in our CTRL-SG instance, is a multi-layer perceptron with \(2\) layers and ReLU activation; we use latent/noise dimension equal to \(d_{\mathrm{noise}} = d_{\mathrm{latent}} = 40\). InfoGAN also uses an extra ``code'' variable, whose dimension we set to \(d_{\mathrm{code}} = 10\). Our implementations for comparison are adapted from the popular GitHub repositories \href{https://github.com/eriklindernoren/PyTorch-GAN}{Pytorch-GAN} and \href{https://github.com/AntixK/PyTorch-VAE}{Pytorch-VAE}.

Despite not directly optimizing for data-space reconstruction, we show that CTRL-MSP and our CTRL-SG instance learn the statistics of the original data quite well, in contrast with several popular representation learning models. In fact, one sees that the autoencodings provided by CTRL-MSP and our CTRL-SG instance preserve the correlation structures between each data subspace, the dimension of each subspace \(\S_{j}\), and the spectra of the data matrices. On the other hand, the autoencoding provided by CVAE, and the generated data provided by the GAN frameworks, do not preserve the dimension of the subspace or the correlation structures of the original data, and in fact the autoencodings almost collapse the multi-dimensional subspace data into a one-dimensional subspace where correlations are uniform. Thus, one sees that CTRL-MSP and our CTRL-SG instance successfully learn the low-dimensional structure where other models fail (\Cref{fig:ctrl_comparison}). This confirms lines of theoretical work showing that, for example, VAEs are not able to learn low-dimensional structure under certain conditions on the encoder \citep{koehler2021variational}.

\begin{figure}
    \centering
    \begin{subfigure}[b]{\textwidth}
    \centering
    \includegraphics[width=0.20\textwidth]{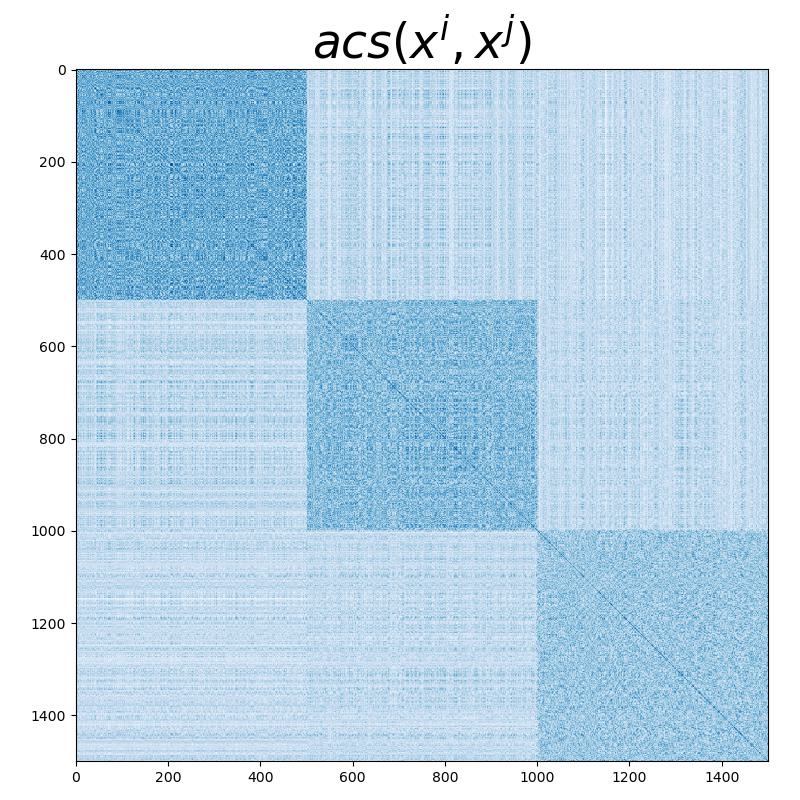}
    \includegraphics[width=0.20\textwidth]{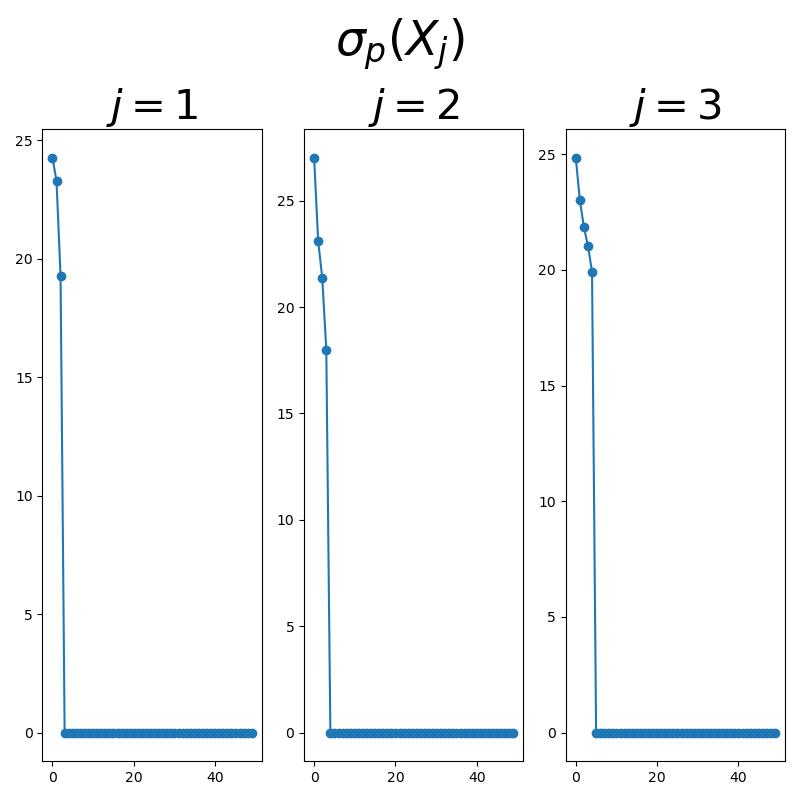}
    \caption{} 
    \end{subfigure}
    
    \begin{subfigure}[b]{0.4\textwidth}
    \centering
    \includegraphics[width=0.44\textwidth]{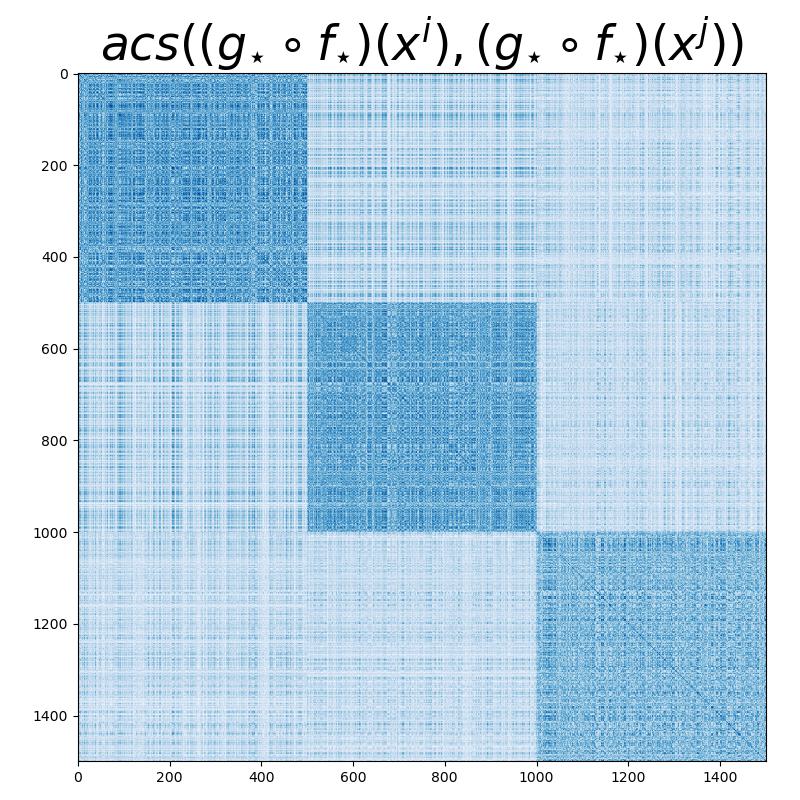}
    \includegraphics[width=0.44\textwidth]{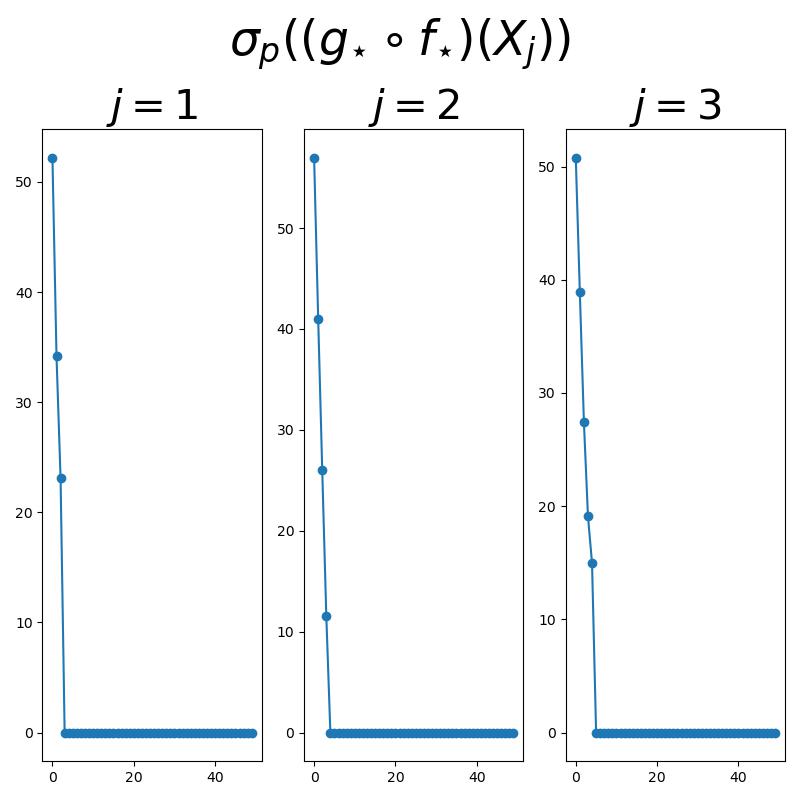}
    \caption{}
    \end{subfigure}
    \begin{subfigure}[b]{0.4\textwidth}
    \centering
    \includegraphics[width=0.44\textwidth]{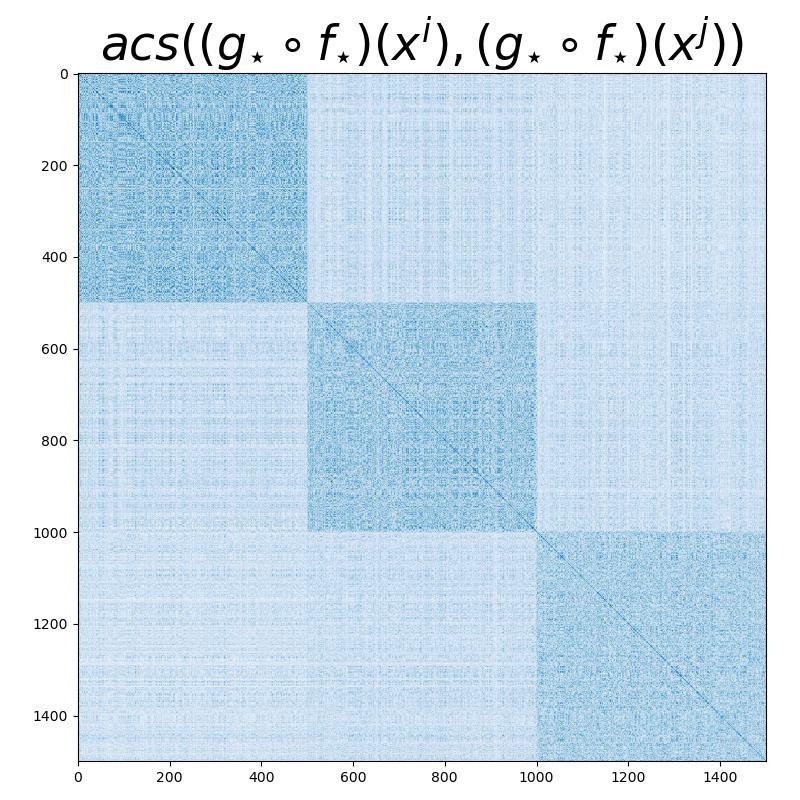}
    \includegraphics[width=0.44\textwidth]{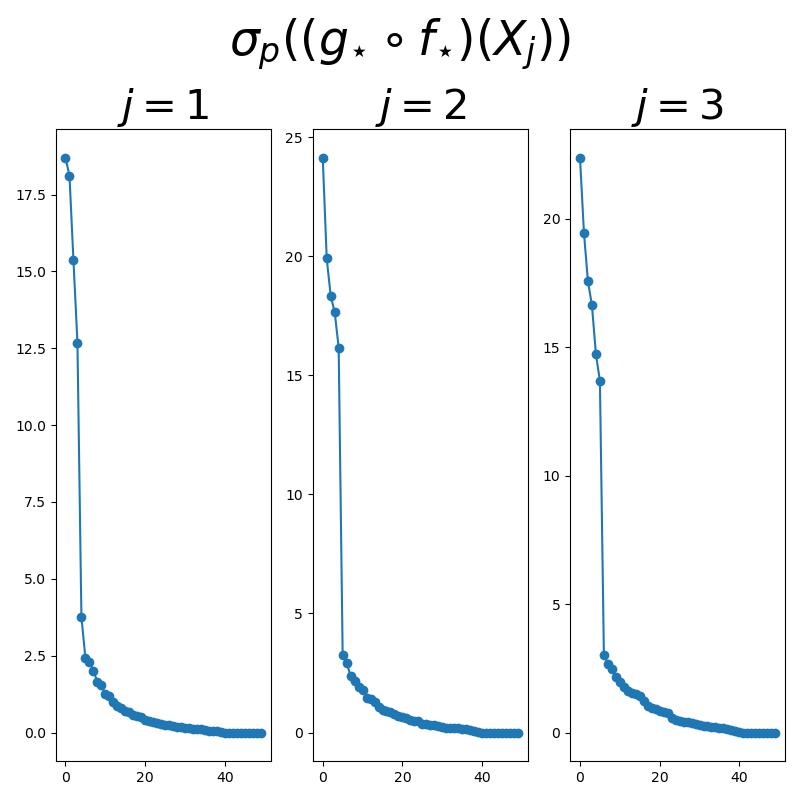}
    \caption{}
    \end{subfigure}
    
    \begin{subfigure}[b]{0.45\textwidth}
    \centering
    \includegraphics[width=0.44\textwidth]{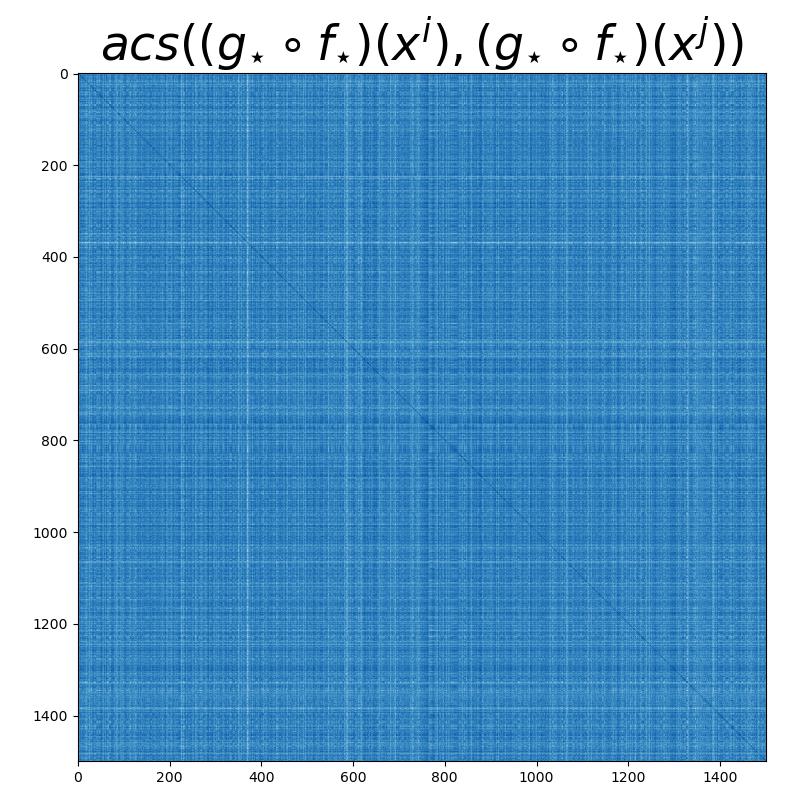}
    \includegraphics[width=0.44\textwidth]{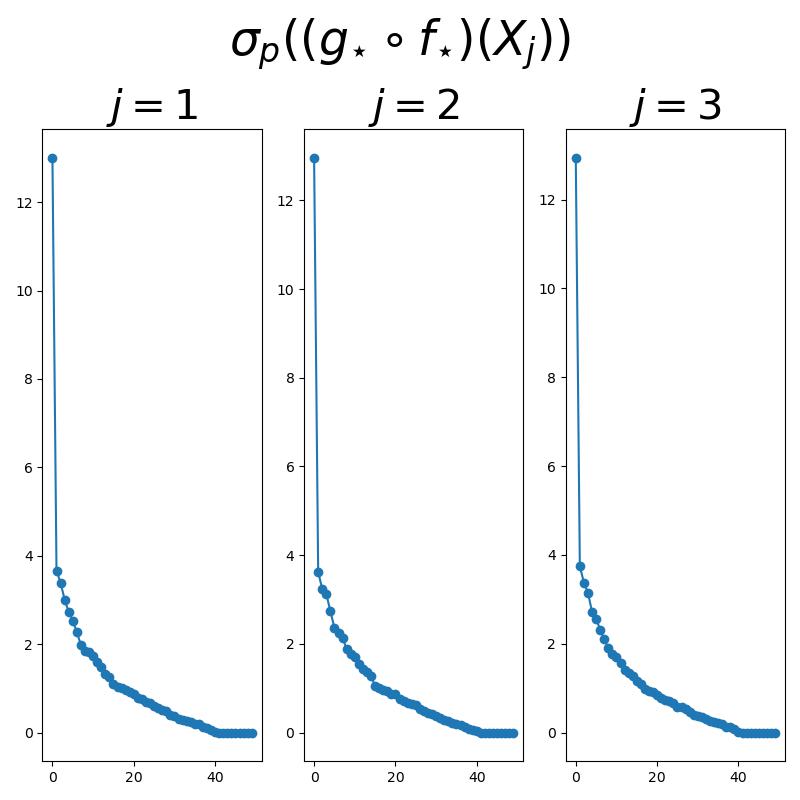}
    \caption{}
    \end{subfigure}
    
    \begin{subfigure}[b]{0.45\textwidth}
    \centering
    \includegraphics[width=0.44\textwidth]{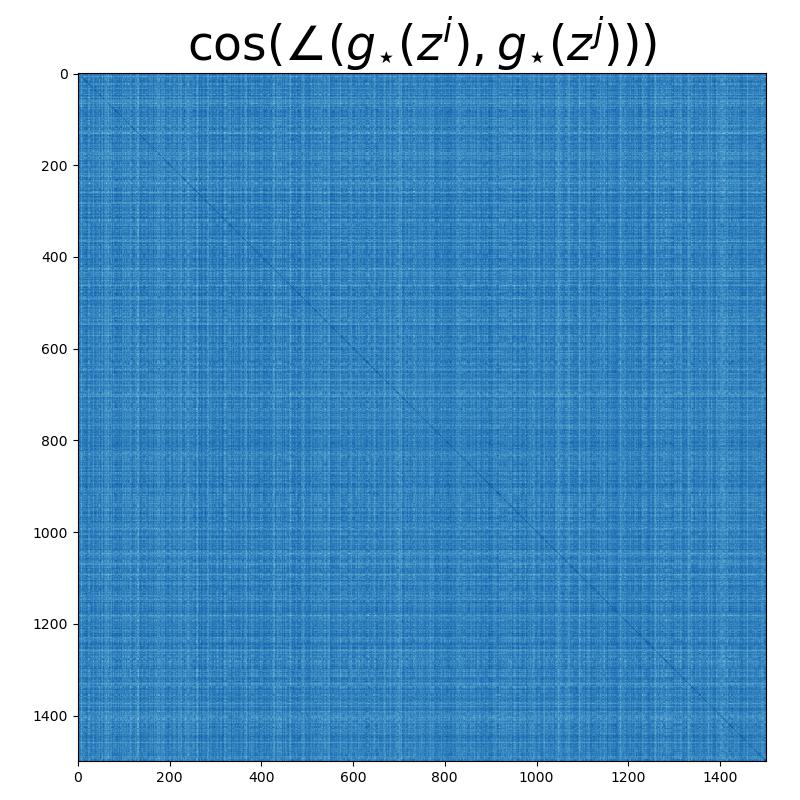}
    \includegraphics[width=0.44\textwidth]{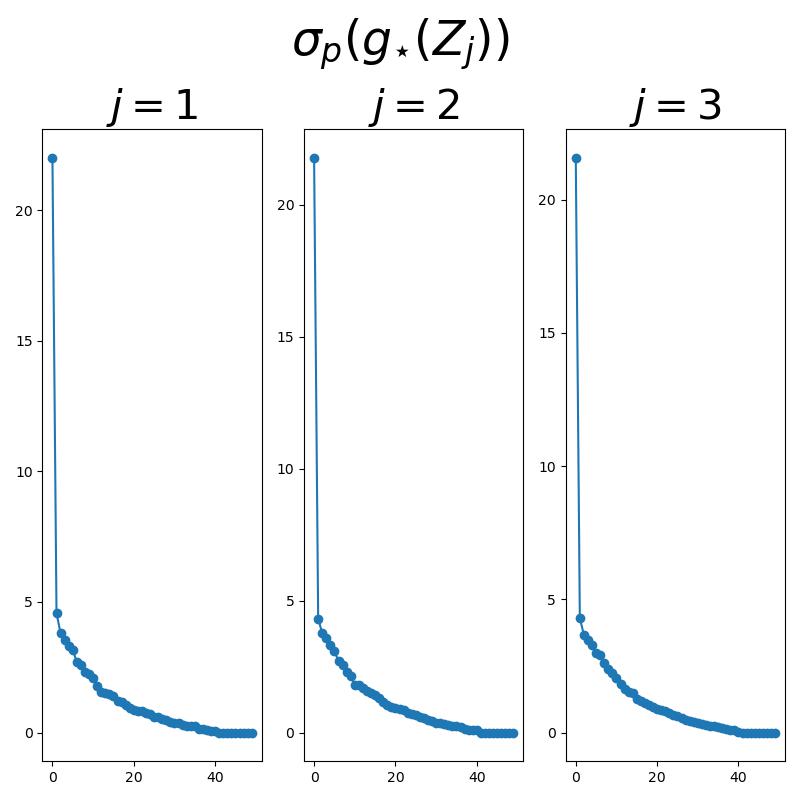}
    \caption{}
    \end{subfigure}
    \begin{subfigure}[b]{0.45\textwidth}
    \centering
    \includegraphics[width=0.44\textwidth]{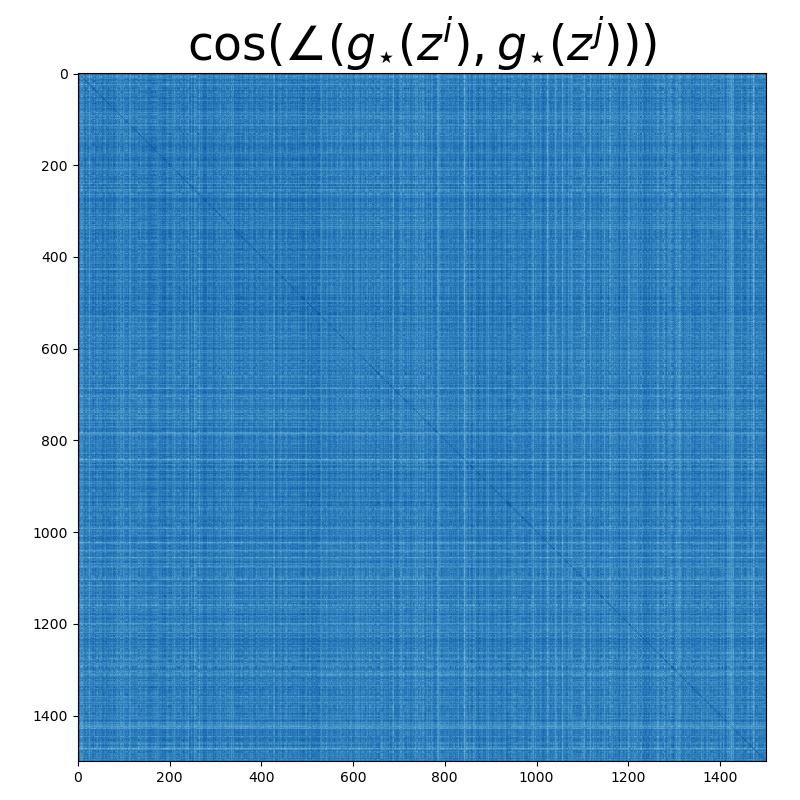}
    \includegraphics[width=0.44\textwidth]{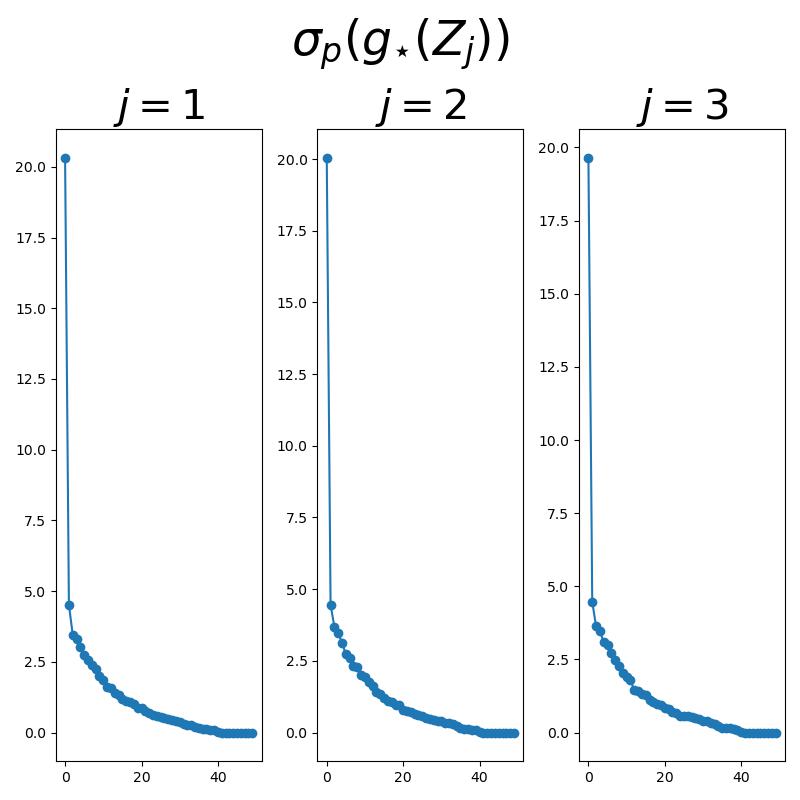}
    \caption{}
    \end{subfigure}

    \caption{Performance of various representation learning algorithms on data lying on linear subspaces. (a): Heatmap of absolute cosine similarities of the original data \(\texttt{acs}(\x^{p}, \x^{q})\), and distribution of singular values \(\sigma_{p}(\X_{j})\) for each class. (b): Heatmap of absolute cosine similarities of the autoencoded data \(\texttt{acs}((g_{\star} \circ f_{\star})(\x^{p}), (g_{\star} \circ f_{\star})(\x^{q}))\) and spectra of the autoencoded matrices \((g_{\star} \circ f_{\star})(\X_{j})\) for every \(j\), where the autoencoding is from CTRL-MSP. (c): The same, where the autoencoding is from our instance of CTRL-SG. (d): The same, where the autoencoding is from CVAE. (e): Heatmap of absolute cosine similarities of the generated data \(\texttt{acs}(g_{\star}(\z^{p}), g_{\star} (\z^{q}))\) and spectra of the generated data matrices \(g_{\star}(\Z_{j})\) for every \(j\), where the generation is from CGAN. (f): The same, where the generation is from InfoGAN.}
    \label{fig:ctrl_comparison}
\end{figure}

%% file: main.bbl
\begin{thebibliography}{43}
\providecommand{\natexlab}[1]{#1}
\providecommand{\url}[1]{\texttt{#1}}
\expandafter\ifx\csname urlstyle\endcsname\relax
  \providecommand{\doi}[1]{doi: #1}\else
  \providecommand{\doi}{doi: \begingroup \urlstyle{rm}\Url}\fi

\bibitem[Arjovsky et~al.(2017)Arjovsky, Chintala, and
  Bottou]{arjovsky2017wasserstein}
Martin Arjovsky, Soumith Chintala, and Léon Bottou.
\newblock {Wasserstein GAN}, 2017.
\newblock URL \url{https://arxiv.org/abs/1701.07875}.

\bibitem[Başar and Olsder(1998)]{basar1998dynamic}
Tamer Başar and Geert~Jan Olsder.
\newblock \emph{{Dynamic Noncooperative Game Theory, 2nd Edition}}.
\newblock Society for Industrial and Applied Mathematics, 1998.
\newblock \doi{10.1137/1.9781611971132}.
\newblock URL \url{https://epubs.siam.org/doi/abs/10.1137/1.9781611971132}.

\bibitem[Chen et~al.(2016)Chen, Duan, Houthooft, Schulman, Sutskever, and
  Abbeel]{chen2016infogan}
Xi~Chen, Yan Duan, Rein Houthooft, John Schulman, Ilya Sutskever, and Pieter
  Abbeel.
\newblock {Infogan: Interpretable Representation Learning by Information
  Maximizing Generative Adversarial Nets}.
\newblock \emph{Advances in neural information processing systems}, 29, 2016.

\bibitem[Cover(1999)]{cover1999elements}
Thomas~M Cover.
\newblock \emph{{Elements of Information Theory}}.
\newblock John Wiley \& Sons, 1999.

\bibitem[Dai et~al.(2022)Dai, Tong, Li, Wu, Psenka, Chan, Zhai, Yu, Yuan, Shum,
  and Ma]{LDR}
Xili Dai, Shengbang Tong, Mingyang Li, Ziyang Wu, Michael Psenka, Kwan Ho~Ryan
  Chan, Pengyuan Zhai, Yaodong Yu, Xiaojun Yuan, Heung-Yeung Shum, and Yi~Ma.
\newblock {CTRL: Closed-Loop Transcription to an LDR via Minimaxing Rate
  Reduction}.
\newblock \emph{Entropy}, 24\penalty0 (4), 2022.
\newblock ISSN 1099-4300.
\newblock \doi{10.3390/e24040456}.
\newblock URL \url{https://www.mdpi.com/1099-4300/24/4/456}.

\bibitem[Farnia and Ozdaglar(2020)]{ozdaglar2020nash}
Farzan Farnia and Asuman Ozdaglar.
\newblock {GANs May Have No Nash Equilibria}, 2020.
\newblock URL \url{https://arxiv.org/abs/2002.09124}.

\bibitem[Fefferman et~al.(2013)Fefferman, Mitter, and
  Narayanan]{fefferman2013testing}
Charles Fefferman, Sanjoy Mitter, and Hariharan Narayanan.
\newblock {Testing the Manifold Hypothesis}, 2013.
\newblock URL \url{https://arxiv.org/abs/1310.0425}.

\bibitem[Feizi et~al.(2017)Feizi, Farnia, Ginart, and
  Tse]{tse2017understanding}
Soheil Feizi, Farzan Farnia, Tony Ginart, and David Tse.
\newblock {Understanding GANs: the LQG Setting}, 2017.
\newblock URL \url{https://arxiv.org/abs/1710.10793}.

\bibitem[Fiez et~al.(2019)Fiez, Chasnov, and Ratliff]{fiez2019convergence}
Tanner Fiez, Benjamin Chasnov, and Lillian~J. Ratliff.
\newblock {Convergence of Learning Dynamics in Stackelberg Games}, 2019.
\newblock URL \url{https://arxiv.org/abs/1906.01217}.

\bibitem[Gemp et~al.(2020)Gemp, McWilliams, Vernade, and
  Graepel]{gemp2020eigengame}
Ian Gemp, Brian McWilliams, Claire Vernade, and Thore Graepel.
\newblock {EigenGame: PCA as a Nash Equilibrium}, 2020.
\newblock URL \url{https://arxiv.org/abs/2010.00554}.

\bibitem[Goodfellow et~al.(2014)Goodfellow, Pouget-Abadie, Mirza, Xu,
  Warde-Farley, Ozair, Courville, and Bengio]{goodfellow2014generative}
Ian~J. Goodfellow, Jean Pouget-Abadie, Mehdi Mirza, Bing Xu, David
  Warde-Farley, Sherjil Ozair, Aaron Courville, and Yoshua Bengio.
\newblock {Generative Adversarial Networks}, 2014.
\newblock URL \url{https://arxiv.org/abs/1406.2661}.

\bibitem[Hastie and Stuetzle(1989)]{hastie1989principal}
Trevor Hastie and Werner Stuetzle.
\newblock {Principal Curves}.
\newblock \emph{Journal of the American Statistical Association}, 84\penalty0
  (406):\penalty0 502--516, 1989.

\bibitem[Hotelling(1933)]{hotelling1933analysis}
Harold Hotelling.
\newblock {Analysis of a Complex of Statistical Variables into Principal
  Components}.
\newblock \emph{Journal of educational psychology}, 24\penalty0 (6):\penalty0
  417, 1933.

\bibitem[Huang et~al.(2018)Huang, Yu, and Wang]{huang2018introduction}
He~Huang, Philip~S. Yu, and Changhu Wang.
\newblock {An Introduction to Image Synthesis with Generative Adversarial
  Nets}, 2018.
\newblock URL \url{https://arxiv.org/abs/1803.04469}.

\bibitem[Jin et~al.(2019)Jin, Netrapalli, and Jordan]{jin2019local}
Chi Jin, Praneeth Netrapalli, and Michael~I. Jordan.
\newblock {What is Local Optimality in Nonconvex-Nonconcave Minimax
  Optimization?}, 2019.
\newblock URL \url{https://arxiv.org/abs/1902.00618}.

\bibitem[Jolliffe(2002)]{Jolliffe2002}
Ian~T Jolliffe.
\newblock \emph{{Principal {C}omponent {A}nalysis}}.
\newblock Springer-Verlag, 2nd edition, 2002.

\bibitem[Kallenberg(2021)]{kallenberg2021foundations}
Olav Kallenberg.
\newblock \emph{{Foundations of Modern Probability}}.
\newblock Probability Theory and Stochastic Modelling. Springer Cham, 2021.
\newblock ISBN 9783030618704.

\bibitem[Karras et~al.(2018)Karras, Laine, and Aila]{karras2018style}
Tero Karras, Samuli Laine, and Timo Aila.
\newblock {A Style-Based Generator Architecture for Generative Adversarial
  Networks}, 2018.
\newblock URL \url{https://arxiv.org/abs/1812.04948}.

\bibitem[Kingma and Ba(2014)]{kingma2014adam}
Diederik~P. Kingma and Jimmy Ba.
\newblock {Adam: A Method for Stochastic Optimization}, 2014.
\newblock URL \url{https://arxiv.org/abs/1412.6980}.

\bibitem[Kingma and Welling(2013)]{kingma2013autoencoding}
Diederik~P Kingma and Max Welling.
\newblock {Auto-Encoding Variational Bayes}, 2013.
\newblock URL \url{https://arxiv.org/abs/1312.6114}.

\bibitem[Kochurov et~al.(2020)Kochurov, Karimov, and
  Kozlukov]{kochurov2020geoopt}
Max Kochurov, Rasul Karimov, and Serge Kozlukov.
\newblock {Geoopt: Riemannian Optimization in PyTorch}, 2020.
\newblock URL \url{https://arxiv.org/abs/2005.02819}.

\bibitem[Koehler et~al.(2021)Koehler, Mehta, Risteski, and
  Zhou]{koehler2021variational}
Frederic Koehler, Viraj Mehta, Andrej Risteski, and Chenghui Zhou.
\newblock {Variational Autoencoders in the Presence of Low-dimensional Data:
  Landscape and Implicit Bias}.
\newblock \emph{arXiv preprint arXiv:2112.06868}, 2021.

\bibitem[Kramer(1991)]{kramer1991nonlinear}
Mark~A Kramer.
\newblock {Nonlinear principal component analysis using autoassociative neural
  networks}.
\newblock \emph{AIChE journal}, 37\penalty0 (2):\penalty0 233--243, 1991.

\bibitem[Lau et~al.(2020)Lau, Qu, Kuo, Zhou, Zhang, and Wright]{Lau2020Short}
Yenson Lau, Qing Qu, Han-Wen Kuo, Pengcheng Zhou, Yuqian Zhang, and John
  Wright.
\newblock {Short and Sparse Deconvolution --- A Geometric Approach}.
\newblock In \emph{International Conference on Learning Representations}, 2020.
\newblock URL \url{https://openreview.net/forum?id=Byg5ZANtvH}.

\bibitem[Li and Bresler(2018)]{li2018global}
Yanjun Li and Yoram Bresler.
\newblock {Global Geometry of Multichannel Sparse Blind Deconvolution on the
  Sphere}.
\newblock In \emph{Advances in Neural Information Processing Systems}, pages
  1132--1143, 2018.

\bibitem[Ma et~al.(2007)Ma, Derksen, Hong, and Wright]{OriginalRateReduction}
Yi~Ma, Harm Derksen, Wei Hong, and John Wright.
\newblock {Segmentation of Multivariate Mixed Data via Lossy Data Coding and
  Compression}.
\newblock \emph{IEEE Transactions on Pattern Analysis and Machine
  Intelligence}, 29\penalty0 (9):\penalty0 1546--1562, 2007.
\newblock \doi{10.1109/TPAMI.2007.1085}.

\bibitem[Mino and Spanakis(2018)]{mino2018logan}
Ajkel Mino and Gerasimos Spanakis.
\newblock {LoGAN: Generating Logos with a Generative Adversarial Neural Network
  Conditioned on color}, 2018.
\newblock URL \url{https://arxiv.org/abs/1810.10395}.

\bibitem[Mirza and Osindero(2014)]{mirza2014conditional}
Mehdi Mirza and Simon Osindero.
\newblock {Conditional Generative Adversarial Nets}.
\newblock \emph{arXiv preprint arXiv:1411.1784}, 2014.

\bibitem[Prasad et~al.(2020)Prasad, Das, and Bhowmick]{prasad2020variational}
Vignesh Prasad, Dipanjan Das, and Brojeshwar Bhowmick.
\newblock {Variational Clustering: Leveraging Variational Autoencoders for
  Image Clustering}.
\newblock \emph{arXiv}, 2020.
\newblock \doi{10.48550/ARXIV.2005.04613}.
\newblock URL \url{https://arxiv.org/abs/2005.04613}.

\bibitem[Qu et~al.(2019)Qu, Zhai, Li, Zhang, and Zhu]{qu2019geometric}
Qing Qu, Yuexiang Zhai, Xiao Li, Yuqian Zhang, and Zhihui Zhu.
\newblock {Geometric Analysis of Nonconvex Optimization Landscapes for
  Overcomplete Learning}.
\newblock In \emph{International Conference on Learning Representations}, 2019.

\bibitem[Shen et~al.(2020)Shen, Xue, Zhang, Letaief, and Lau]{shen2020complete}
Yifei Shen, Ye~Xue, Jun Zhang, Khaled~B Letaief, and Vincent Lau.
\newblock {Complete Dictionary Learning via $\ell_p$-norm Maximization}.
\newblock \emph{arXiv preprint arXiv:2002.10043}, 2020.

\bibitem[Sohn et~al.(2015)Sohn, Lee, and Yan]{sohn2015learning}
Kihyuk Sohn, Honglak Lee, and Xinchen Yan.
\newblock {Learning Structured Output Representation using Deep Conditional
  Generative Models}.
\newblock In C.~Cortes, N.~Lawrence, D.~Lee, M.~Sugiyama, and R.~Garnett,
  editors, \emph{Advances in Neural Information Processing Systems}, volume~28.
  Curran Associates, Inc., 2015.
\newblock URL
  \url{https://proceedings.neurips.cc/paper/2015/file/8d55a249e6baa5c06772297520da2051-Paper.pdf}.

\bibitem[Tipping and Bishop(1999)]{tipping1999probabilistic}
Michael~E Tipping and Christopher~M Bishop.
\newblock {Probabilistic Principal Component Analysis}.
\newblock \emph{Journal of the Royal Statistical Society: Series B (Statistical
  Methodology)}, 61\penalty0 (3):\penalty0 611--622, 1999.

\bibitem[Tong et~al.(2022)Tong, Dai, Wu, Li, Yi, and Ma]{tong2022incremental}
Shengbang Tong, Xili Dai, Ziyang Wu, Mingyang Li, Brent Yi, and Yi~Ma.
\newblock {Incremental Learning of Structured Memory via Closed-Loop
  Transcription}, 2022.
\newblock URL \url{https://arxiv.org/abs/2202.05411}.

\bibitem[Van Der~Maaten et~al.(2009)Van Der~Maaten, Postma, and Van~den
  Herik]{van2009dimensionality}
Laurens Van Der~Maaten, Eric Postma, and Jaap Van~den Herik.
\newblock {Dimensionality Reduction: a Comparative}.
\newblock \emph{J Mach Learn Res}, 10\penalty0 (66-71):\penalty0 13, 2009.

\bibitem[Vidal et~al.(2012)Vidal, Ma, and Sastry]{vidal2012generalized}
Rene Vidal, Yi~Ma, and Shankar Sastry.
\newblock {Generalized Principal Component Analysis (GPCA)}.
\newblock \emph{arXiv}, 2012.
\newblock \doi{10.48550/ARXIV.1202.4002}.
\newblock URL \url{https://arxiv.org/abs/1202.4002}.

\bibitem[Vidal et~al.(2016)Vidal, Ma, and Sastry]{Vidal:Springer16}
Ren\'{e} Vidal, Yi~Ma, and Shankar Sastry.
\newblock \emph{{Generalized Principal Component Analysis}}.
\newblock Springer Verlag, 2016.

\bibitem[Wright and Ma(2022)]{Wright-Ma-2022}
John Wright and Yi~Ma.
\newblock \emph{{High-Dimensional Data Analysis with Low-Dimensional Models:
  Principles, Computation, and Applications}}.
\newblock Cambridge University Press, 2022.

\bibitem[Yu et~al.(2020)Yu, Chan, You, Song, and Ma]{OriginalMCR2}
Yaodong Yu, Kwan Ho~Ryan Chan, Chong You, Chaobing Song, and Yi~Ma.
\newblock {Learning Diverse and Discriminative Representations via the
  Principle of Maximal Coding Rate Reduction}, 2020.
\newblock URL \url{https://arxiv.org/abs/2006.08558}.

\bibitem[Zhai et~al.(2019)Zhai, Mehta, Zhou, and Ma]{zhai2019understanding}
Yuexiang Zhai, Hermish Mehta, Zhengyuan Zhou, and Yi~Ma.
\newblock {Understanding $\ell_{4}$-based Dictionary Learning: Interpretation,
  Stability, and Robustness}.
\newblock In \emph{International Conference on Learning Representations}, 2019.

\bibitem[Zhai et~al.(2020)Zhai, Yang, Liao, Wright, and Ma]{zhai2020complete}
Yuexiang Zhai, Zitong Yang, Zhenyu Liao, John Wright, and Yi~Ma.
\newblock {Complete Dictionary Learning via $\ell_{4}$-Norm Maximization over
  the Orthogonal Group.}
\newblock \emph{J. Mach. Learn. Res.}, 21\penalty0 (165):\penalty0 1--68, 2020.

\bibitem[Zhang et~al.(2019)Zhang, Kuo, and Wright]{zhang2019structured}
Yuqian Zhang, Han-Wen Kuo, and John Wright.
\newblock {Structured Local Optima in Sparse Blind Deconvolution}.
\newblock \emph{IEEE Transactions on Information Theory}, 66\penalty0
  (1):\penalty0 419--452, 2019.

\bibitem[Zhu et~al.(2020)Zhu, Jiao, and Tse]{zhu2020deconstructing}
Banghua Zhu, Jiantao Jiao, and David Tse.
\newblock {Deconstructing generative adversarial networks}.
\newblock \emph{IEEE Transactions on Information Theory}, 66\penalty0
  (11):\penalty0 7155--7179, 2020.

\end{thebibliography}
